\documentclass[letterpaper,journal]{IEEEtran}

\usepackage{amsmath,amsfonts,amssymb}
\usepackage{algorithmic}
\usepackage{algorithm}
\usepackage{array}
\usepackage{tabularx}
\usepackage[caption=false,font=normalsize,labelfont=sf,textfont=sf]{subfig}
\usepackage{textcomp}
\usepackage{stfloats}
\usepackage{url}
\usepackage{verbatim}

\usepackage{graphicx}
\usepackage{adjustbox}
\usepackage[nocompress]{cite}
\usepackage{booktabs}
\usepackage{multirow}
\usepackage[table]{xcolor}
\usepackage{placeins}
\usepackage{float}
\usepackage{needspace}
\usepackage{pifont}
\usepackage[colorlinks=true, linkcolor=blue, citecolor=blue]{hyperref}
\newcommand{\hsaFigRef}[2][]{\hyperref[#2]{Fig.~\ref*{#2}#1}}
\newcommand{\hsaFigureRef}[2][]{\hsaFigRef[#1]{#2}}
\newcommand{\hsaTableRef}[1]{\hyperref[#1]{Table~\ref*{#1}}}
\newcommand{\hsaAppendixRef}[1]{\hyperref[#1]{Appendix~\ref*{#1}}}
\makeatletter
\newcommand{\hsaEqNumber}[1]{\textup{\tagform@{\ref*{#1}}}}
\makeatother
\newcommand{\hsaEqRef}[1]{\hyperref[#1]{Eq.~\hsaEqNumber{#1}}}
\newcommand{\hsaEquationRef}[1]{\hsaEqRef{#1}}
\newcommand{\hsaDefinitionRef}[1]{\hyperref[#1]{Definition~\ref*{#1}}}
\newcommand{\hsaAssumptionRef}[1]{\hyperref[#1]{Assumption~\ref*{#1}}}
\newcommand{\hsaTheoremRef}[1]{\hyperref[#1]{Theorem~\ref*{#1}}}
\newcommand{\hsaCorollaryRef}[1]{\hyperref[#1]{Corollary~\ref*{#1}}}
\newcommand{\hsaPropositionRef}[1]{\hyperref[#1]{Proposition~\ref*{#1}}}
\newcommand{\hsaAlgorithmRef}[1]{\hyperref[#1]{Algorithm~\ref*{#1}}}
\newcommand{\hsaSectionRef}[1]{\hyperref[#1]{Section~\ref*{#1}}}

\newtheorem{definition}{Definition}
\newtheorem{assumption}{Assumption}
\newtheorem{theorem}{Theorem}
\newtheorem{corollary}{Corollary}
\newtheorem{proposition}{Proposition}
\makeatletter
\def\@begintheorem#1#2{\@IEEEtmpitemindent\itemindent\relax\topsep 0pt\rmfamily\trivlist%
    \item[]{\itshape\indent #1\ #2:}\itemindent\@IEEEtmpitemindent\relax}
\def\@opargbegintheorem#1#2#3{\@IEEEtmpitemindent\itemindent\relax\topsep 0pt\rmfamily\trivlist%
    \item[]{\itshape\indent #1\ #2\ (#3):}\itemindent\@IEEEtmpitemindent\relax}
\makeatother

\newcolumntype{C}[1]{>{\centering\arraybackslash}p{#1}}
\newcolumntype{M}[1]{>{\centering\arraybackslash}m{#1}}
\newcolumntype{Y}{>{\centering\arraybackslash}X}
\newcolumntype{Z}{>{\raggedleft\arraybackslash}X}
\definecolor{hsamain}{RGB}{30,105,160}
\definecolor{hsarow}{RGB}{232,244,251}
\definecolor{hsaloss}{RGB}{178,72,64}
\newcommand{\hsacell}[1]{\cellcolor{hsarow}#1}
\newcommand{\gaincell}[2]{\begin{tabular}[c]{@{}c@{}}#1\\[-0.15ex]{\scriptsize\color{hsamain}(#2)}\end{tabular}}
\newcommand{\gainmaincell}[2]{\begin{tabular}[c]{@{}c@{}}#1\\[-0.15ex]{\fontsize{6}{6.5}\selectfont\color{hsamain}(#2)}\end{tabular}}

\newcommand{\gaininline}[2]{\mbox{#1\,{\scriptsize\color{hsamain}(#2)}}}
\newcommand{\lossinline}[2]{\mbox{#1\,{\scriptsize\color{hsaloss}(#2)}}}
\newcommand{\tworowhead}[1]{\multirow{2}{*}[-1.7ex]{#1}}
\newcommand{\gainrowcell}[1]{\begin{tabular}[c]{@{}c@{}}#1\end{tabular}}
\newcommand{\hsaTakeaway}[1]{%
  \par\nobreak
  \ifdim\prevdepth>-1000pt \vskip-\prevdepth \fi
  \vskip8pt\relax\nointerlineskip
  \begingroup
  \setlength{\fboxsep}{4pt}%
  \vbox{\hbox{\colorbox{gray!12}{\parbox{\dimexpr\linewidth-2\fboxsep\relax}{\small\textcolor{black}{\textbf{Takeaway.} #1}}}}\kern0pt}%
  \endgroup
  \vskip6pt\relax}
\newcommand{\icmlvenue}[1]{\emph{Proc. #1 Int. Conf. Mach. Learn.}}
\newcommand{\pmlrseries}{ser. \emph{Proc. Mach. Learn. Res.}}
\newcommand{\neuripsvenue}{\emph{Adv. Neural Inf. Process. Syst.}}
\newcommand{\cvprvenue}{\emph{Proc. IEEE/CVF Conf. Comput. Vis. Pattern Recognit.}}
\newcommand{\iclrvenue}{\emph{Int. Conf. Learn. Represent.}}
\hypersetup{
  pdftitle={Hub-Spectral Activation of Latent Multimodal Knowledge},
  pdfauthor={Ying Guo, Haidong Chen, Linrui Xu, Xiaohao Liu, Chuancheng Shi, Canran Xiao, Dan Zhang, Fei Shen, Li Shen, Tat-Seng Chua},
  pdfsubject={Hub-spectral activation in frozen multimodal representations}
}

\begin{document}

\title{Hub-Spectral Activation of Latent Multimodal Knowledge}

\author{
Ying~Guo$^{*}$,
Haidong~Chen$^{*}$,
    Linrui~Xu,
    Xiaohao~Liu,
    Chuancheng~Shi,
    Canran~Xiao, \\
    Dan Zhang,
    Fei Shen,~\IEEEmembership{Senior~Member,~IEEE},
    Li Shen,~\IEEEmembership{Senior~Member,~IEEE},
    and Tat-Seng Chua

\thanks{This work was supported in part by the Key Supported Program of Joint
Fund of the National Natural Science Foundation of China (Grant No. U23B2029)
and the Beijing Natural Science Foundation (Grant No. 4262004), and in part by the National Research Foundation, Singapore
under its National Large Language Models Funding Initiative
(AISG Award No: AISG-NMLP-2024-002).
Any opinions, findings and conclusions or recommendations expressed in this
material are those of the author(s) and do not reflect the views of
National Research Foundation, Singapore.
$^{*}$ denotes equal contribution.
(Corresponding author: Fei Shen.)}

\thanks{Ying Guo and Haidong Chen are with the Beijing Key Laboratory of Key Technologies for AI+ Domain Applications, North China University of Technology, Beijing 100144, China (e-mail: guoying@ncut.edu.cn; chenhaidong123@mail.ncut.edu.cn).}

\thanks{Linrui Xu is with the School of Geosciences and Info-Physics, Central South University, Changsha 410083, China (e-mail: xulinrui@csu.edu.cn).}

\thanks{Chuancheng Shi is with the School of
Computer Science, The University of Sydney, Australia
(e-mail: cshi0459@uni.sydney.edu.au).}

\thanks{Canran Xiao and Li Shen are with the School of Cyber Science and Technology, Shenzhen Campus of Sun Yat-sen University, China
(e-mail: xiaocr3@mail.sysu.edu.cn; mathshenli@gmail.com).}

\thanks{Xiaohao Liu, Dan Zhang, Fei Shen and Tat-Seng Chua are with the NExT++ Research
Centre, National University of Singapore, Singapore
(e-mail: xiaohao.liu@u.nus.edu, zhangdan25@nus.edu.sg, shenfei29@nus.edu.sg; dcscts@nus.edu.sg).}

}

\markboth{Hub-Spectral Activation of Latent Multimodal Knowledge}%
{Guo et al.: Hub-Spectral Activation of Latent Multimodal Knowledge}

\maketitle

\begin{abstract}
\color{black}
Multimodal representation learning seeks shared representations for cross-modal retrieval and knowledge transfer. 
Hub-based binding reduces pairwise supervision costs, but separate hub connections cannot guarantee reliable alignment between modalities without direct joint training. 
We introduce Hub-Spectral Activation (HSA), a closed-form method for recovering and activating the hub-readable component of latent multimodal knowledge in frozen representations. 
We formalize this knowledge as source-induced cross-modal dependence and characterize the component determined by the second-order statistics of two trained hub edges. 
Under a second-order source model, we establish conditions for exact recovery of the complete source-induced relation and bound the dimension of its hub-readable component by the hub covariance rank. 
HSA composes and standardizes hub-edge statistics, extracts paired spectral directions, and combines reliability-weighted matching evidence with source-gated candidate resolution for bidirectional retrieval and prototype classification. HSA requires no target-pair supervision, gradient optimization, or backbone updates. 
Across 19 retrieval and 11 prototype-classification relations on ImageBind and LanguageBind, HSA raises mean bidirectional Recall@10 from 18.27\% to 31.15\% and mean macro Top-1 accuracy from 29.01\% to 52.43\%, respectively. 
Controlled analyses further identify valid hub-edge correspondence and leading spectral directions as key sources of retrieval gains, demonstrating the utility of latent multimodal knowledge beyond native similarity scores.
Code and models are publicly available at \url{https://github.com/Luo1Yan/HSA}.
\end{abstract}

\begin{IEEEkeywords}
Cross-modal alignment, cross-modal classification, cross-modal retrieval, latent multimodal knowledge, multimodal binding, multimodal representation learning.
\end{IEEEkeywords}

\Needspace{5\baselineskip}
\section{Introduction}

{\color{black}
The same object or event~\cite{baltrusaitis2019multimodal,xu2023transformers,
zhu2024visionx,liu2020ntu120} can be observed through images, language, audio, depth, thermal signals, and motion. Multimodal representation learning organizes these heterogeneous observations so that information acquired in one modality can support matching and transfer in another \cite{lahat2015fusion,lu2023theory,zong2024survey}. Joint objectives learn such relations from paired or co-occurring observations \cite{l3net,cmc,mmv,avid,vatt,audioclip}. Direct target-pair training can provide precise alignment for a chosen relation, as illustrated in \hsaFigRef[(a)]{fig:hsa_overview} \cite{clip,align}. Extending this strategy to every modality pair requires paired data and objectives for a number of relations that grows quadratically with the number of modalities.
\par}

\begin{figure}[t]
\centering
\includegraphics[width=\linewidth]{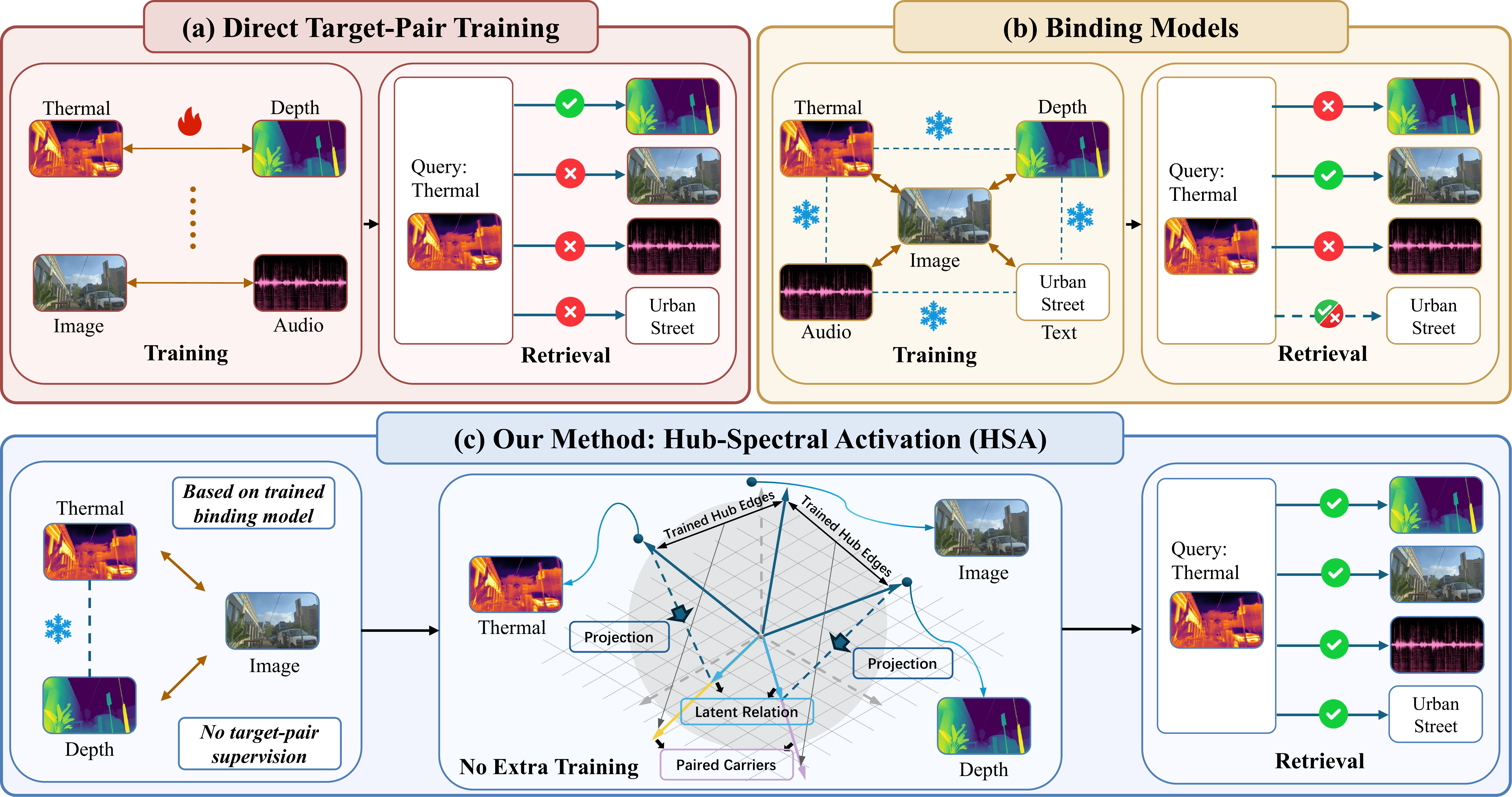}
\caption{\textcolor{black}{\textbf{From Multimodal Training to Hub-Spectral Activation.} (a) Target-pair training learns one relation from paired data. (b) Hub-based binding reduces pairwise supervision but leaves held-out retrieval uneven. (c) HSA reads the relation supported by two trained hub connections and produces retrieval rankings. HSA improves held-out retrieval in closed form, with frozen encoders and without positive target-pair examples or gradient optimization.}}
\label{fig:hsa_overview}
\end{figure}

{\color{black}
Hub-based binding models reduce this supervision burden by aligning each modality with a shared hub. ImageBind uses image to connect text, audio, depth, thermal, and inertial modalities \cite{imagebind}; LanguageBind uses language to connect video, infrared, depth, audio, and image \cite{languagebind}. Text-hub binding has also been extended to medical imaging modalities \cite{m3bind}. We call two non-hub modalities connected only through separately trained hub edges a \emph{held-out pair}. The indirect path can support transfer, yet native retrieval remains highly uneven, as illustrated in \hsaFigRef[(b)]{fig:hsa_overview}. This variability motivates a specific question: after binding-model training is complete, what relation between a held-out pair can be identified and activated using only its two trained hub connections?
\par}

{\color{black}
We answer this question with Hub-Spectral Activation (HSA), shown in \hsaFigRef[(c)]{fig:hsa_overview}. We define \emph{latent multimodal knowledge} as source-induced dependence preserved in frozen representations and identify the \emph{hub-readable component} determined by the second-order statistics of the two trained hub edges. HSA composes their moment systems, standardizes the resulting relation, and identifies paired spectral carriers that couple informative directions across the target spaces. It maps samples into these low-dimensional knowledge coordinates and combines reliability-weighted carrier agreement with source-gated candidate resolution. HSA thereby turns the recovered hub-readable relation into a basis for comparing candidates, supporting both bidirectional retrieval and prototype classification. Our analysis establishes when the hub-readable component equals the complete source-induced relation and further bounds its dimension by the rank of the hub representation. Every fitted quantity is obtained in closed form, without positive target-pair examples, pair identities, gradient optimization, or backbone updates.
\par}

{\color{black}
Existing approaches connect indirect modalities through target-pair correspondences, common anchors, overlapping modalities, unpaired target marginals, learned proxies, or target-encoder adaptation \cite{asif,relrep,cmcr,exmcr,realign,emergentbridge}. Related frozen-encoder methods still optimize lightweight maps from paired target data \cite{structure2025,freezealign2025}. HSA operates in a fully frozen post-training setting, using the two hub-edge datasets to construct and calibrate the held-out relation readout. Across nineteen ImageBind and LanguageBind retrieval relations, it raises mean Recall@10 from 18.27\% to 31.15\%, compared with 31.49\% for fixed-duration supervised target-pair training. Across the eleven prototype-classification relations, HSA estimates the target-modality relation and constructs its readout using the classification source features, raising mean macro Top-1 accuracy from 29.01\% to 52.43\%. Hub-edge shuffling, disjoint-source estimation, and carrier controls jointly show that HSA uses valid sample correspondence within the two hub edges to recover the hub-readable relation, with leading paired spectral carriers providing its main retrieval gain.
\par}

{\color{black}
Before delving into details, we summarize our contributions as follows:
\begin{itemize}
    \item We formalize latent multimodal knowledge as source-induced dependence, identify the component readable through two hub edges, and give its exact-recovery condition. A rank bound further quantifies the capacity available through the hub.
    \item We develop HSA, a closed-form readout that recovers the target-modality relation from two hub edges, extracts paired spectral carriers, and uses the resulting relation for bidirectional retrieval and prototype classification. Relation fitting uses no positive target-pair examples, gradient optimization, or backbone updates.
    \item Across nineteen retrieval relations and eleven classification relations from ImageBind and LanguageBind, HSA improves every reported relation-level point estimate over frozen cosine. Correspondence, disjoint-source, and spectral controls link the retrieval gains to valid hub-edge statistics and the leading paired carriers.
\end{itemize}
}

\section{Related Work}

\subsection{Building Unified Multimodal Spaces}

Unified multimodal spaces are learned with complementary objectives. Canonical-correlation methods learn shared coordinates in classical, multiset, and deep settings \cite{hotelling1936,kettenring1971,andrew2013deepcca,wang2015multiview}. Other approaches use reconstruction \cite{ngiam2011multimodal}, joint generative modeling \cite{srivastava2012multimodal}, or pairwise correspondence and contrastive alignment \cite{l3net,cmc,mmv,avid,vatt,audioclip,clip,align}. Vision-language pretraining provides an extensively studied setting for these
objectives \cite{zhang2024visiontasks}. ViLBERT, LXMERT, and UNITER learn
joint image-text representations through transformer-based interaction and
paired pretraining \cite{vilbert,lxmert,uniter}. FLAVA combines contrastive
alignment with multimodal fusion \cite{flava}; ALBEF aligns image and text
representations before cross-modal attention, while BLIP learns understanding
and generation from filtered and synthetic captions \cite{albef,blip}. Frozen
visual encoders can also support trainable cross-modal connections: LiT tunes
a text encoder against a locked image encoder, and BLIP-2 trains a querying
transformer between frozen image and language models \cite{lit,blip2}. Multimodal
conditioning also supports image-and-pose-guided generation in
IMAGPose~\cite{shen2024imagpose} and garment-conditioned synthesis in
IMAGDressing-v1~\cite{shen2025imagdressing}.

When modalities rarely co-occur, binding methods exploit shared structure across separately collected datasets \cite{unialign2025}. ImageBind uses image as the hub for text, audio, depth, thermal, and IMU, while LanguageBind uses language to connect video, infrared, depth, audio, and image \cite{imagebind,languagebind}. Medical binding similarly connects text, imaging, and physiological signals \cite{m3bind,medbind,probmed}. UniBind learns modality-agnostic alignment centers, and ULIP places point clouds, images, and language in one representation space \cite{unibind,ulip}. PMRL uses a rank-1 Gram objective for anchor-free simultaneous alignment; CalMRL handles missing observations through anchor-shift analysis and representation-level imputation \cite{pmrl,calmrl}. These methods construct or adapt a unified space. Their objectives determine what cross-modal structure is encoded and how it is organized. HSA operates after training and reads the hub-readable component of a held-out relation from two frozen hub edges.

\subsection{Connecting Indirectly Related Modalities}

Methods that connect indirectly related modalities differ in the information available for alignment. ReAlign uses training-free Anchor, Trace, and Centroid Alignment to map text embeddings toward the image distribution from large unpaired image and text marginals \cite{realign}. ASIF forms a coupled dictionary from paired target examples \cite{asif}. C-MCR and Ex-MCR connect independently trained spaces through an overlapping modality \cite{cmcr,exmcr}, whereas FreeBind and OmniBind fuse pretrained expert spaces \cite{freebind,omnibind}. COX couples known and unseen modalities without instance-level pairs by learning variational bottlenecks and emergent correspondences \cite{cox2025}. TextME projects unseen modalities into an LLM embedding space from text descriptions \cite{textme}. Related frozen-encoder methods learn lightweight maps from limited paired supervision \cite{structure2025,freezealign2025}. Together, these approaches use target pairs, overlapping modalities, text, unpaired marginals, or learned correspondences. Their supervision and model-update requirements therefore vary across settings.

EmergentBridge strengthens weak emergent transfer by learning a proxy generator from
an anchor to an already aligned modality \cite{emergentbridge}. It then adapts the target-modality encoder with anchor-paired data and an orthogonal-subspace proxy objective, while keeping the anchor encoder frozen. This procedure uses no direct pairs between the evaluated modalities, but it requires proxy training and gradient-based target-encoder adaptation. HSA keeps the hub and both target encoders frozen. Its fit uses only the existing $A$--$H$ and $H$--$B$ hub-edge datasets and excludes $A$--$B$ identities, target losses, gradient optimization, and target-informed rank selection.

\subsection{Linear Recovery and Representation Geometry}

Analyses of frozen representations examine the structure retained across models and modalities~\cite{klabunde2025similarity,li2015convergent,boix2022gulp,bansal2021stitching,lenc2015equivalence}. Representational similarity analysis, canonical-correlation comparisons, centered-kernel alignment, and statistical tests compare internal organization \cite{rsa,svcca,pwcca,cka,visionlanguagesim2024,williams2021shape,ding2021grounding}. Modality-gap and cross-encoder studies show that semantic structure can coexist with modality-dependent separation \cite{visionlanguagesim2024,modalitygap}. The Platonic Representation Hypothesis and subsequent matching work study compatible relational geometry across independently learned representations \cite{platonic,blindmatch2025}. Alignment also depends on modality similarity and the balance of redundant and unique information \cite{emergencealignment}. Relative representations expose shared organization by expressing samples through similarities to common anchors \cite{relrep}. These studies characterize cross-modal compatibility; HSA constructs a held-out retrieval rule from two frozen hub edges.

Classical linear methods offer transparent tools for cross-space recovery. CCA finds correlated coordinates from paired two-view observations, with multiset extensions for more than two views \cite{hotelling1936,kettenring1971}. Ridge regression stabilizes linear prediction under correlated coordinates \cite{hoerl1970ridge}, and Orthogonal Procrustes estimates an optimal orthogonal map from paired correspondences \cite{schonemann1966}. Latent Space Translation derives a closed-form transformation from parallel semantic anchors that denote the same high-level concept in both spaces \cite{latenttranslation}. We use these constructions as protocol references. Bidirectional ridge~\cite{hoerl1970ridge} and Procrustes~\cite{schonemann1966} map
each target independently into the hub space, while Paired Orthogonal Procrustes
(Paired-OP)~\cite{schonemann1966} solves one orthogonal map per retrieval direction
from paired $A$--$B$ rows.
Building on these classical linear tools, HSA studies which component of latent multimodal knowledge can be identified from the second-order statistics of two trained hub connections. It characterizes the recovery limits of this component and uses it for held-out retrieval and prototype classification, connecting representation analysis with cross-modal transfer without direct target-pair supervision or backbone updates.

\section{Preliminary}
\label{sec:preliminaries}

\noindent\textbf{Held-Out Hub Setting.}
A real-world object, scene, or event may give rise to several modality-specific observations that retain dependence through their common source. We consider two target modalities $A$ and $B$ that are connected during training only through a hub modality $H$. Let $O_i$ denote the common source and let $\mathbf L_i=\ell(O_i)\in\mathbb R^q$ collect the factors shared across its observations. For each $m\in\{A,H,B\}$, $X_i^m$ is a partial observation of $O_i$, and a frozen encoder maps it into the common $d$-dimensional binding space:
\begin{equation}
\begin{gathered}
O_i\sim\mathcal P_O,
\qquad
X_i^m\mid O_i\sim p_m(\cdot\mid O_i),
\qquad
\mathbf L_i=\ell(O_i),\\
\mathbf a_i=\operatorname{nrm}(f_A(X_i^A)),
\qquad
\mathbf h_i=\operatorname{nrm}(f_H(X_i^H)),\\
\mathbf b_i=\operatorname{nrm}(f_B(X_i^B)).
\end{gathered}
\label{eq:hsa_source_model}
\end{equation}
Although fitting receives no direct $A$--$B$ pairs, the target representations can retain a source-induced relation through $\mathbf L_i$. We formalize this relation next.
\par

\begin{definition}[Source-induced latent multimodal knowledge]
\label{def:latent_knowledge} 
For square-integrable frozen representations $\mathbf A$ and $\mathbf B$ from a common source, latent multimodal knowledge is the cross-covariance of their source-conditioned means:
\begin{equation}
\mathcal K_{AB}
:=\operatorname{Cov}\!\left(
\mathbb E[\mathbf A\mid\mathbf L],
\mathbb E[\mathbf B\mid\mathbf L]
\right).
\label{eq:hsa_knowledge_relation}
\end{equation}
It collects the dependence carried by shared source factors in the two frozen target spaces.
\end{definition}

\noindent\textbf{Available Information.}
\textcolor{black}{During fitting, HSA receives only two paired hub-edge datasets,}
\begin{equation*}
\mathcal D_{AH}
=\{(\mathbf a_i,\mathbf h_i^{AH})\}_{i=1}^{n_{AH}},
\qquad
\mathcal D_{HB}
=\{(\mathbf h_j^{HB},\mathbf b_j)\}_{j=1}^{n_{HB}}.
\end{equation*}
HSA estimates the two edge-moment systems separately, without requiring cross-edge row correspondence. The edge datasets may therefore use disjoint source instances. When synchronized observations are available, one source row may instead contribute a pair to each dataset, so the two hub representations can coincide. HSA summarizes each hub edge through its moments, including $\Sigma_{AH}$ and $\Sigma_{HB}$, without forming $\Sigma_{AB}$ or an $A$--$B$ training objective. At population level, the available information is
\begin{equation}
\begin{aligned}
\mathfrak S_H=\{&\boldsymbol\mu_A,\boldsymbol\mu_H,\boldsymbol\mu_B,
\Sigma_{AA},\Sigma_{AH},\\
&\Sigma_{HH},\Sigma_{HB},\Sigma_{BB}\}.
\end{aligned}
\label{eq:hsa_information_interface}
\end{equation}
Its empirical counterpart, $\widehat{\mathfrak S}_H$, supplies the second-order statistics for estimating the hub-readable relation. HSA uses the two hub-edge datasets for rank selection and score calibration. Relation fitting excludes the empirical $A$--$B$ cross-covariance, an $A$--$B$ training loss, target labels, and target-informed rank selection. The statistical task is to identify the component of $\mathcal K_{AB}$ fixed by $\mathfrak S_H$ and convert it into a score for matching new sample pairs.

\noindent\textbf{Second-Order Source Model.}
To connect the shared factors, frozen representations, and observed hub edges, we use a second-order source model. The model permits target-specific residual dependence while requiring both hub edges to reflect the shared factors consistently. Intuitively, the shared factors $\mathbf L$ link the two observed hub edges, while $A$ and $B$ may retain residual dependence beyond HSA's stated recovery scope.
\begin{assumption}[\textcolor{black}{Second-order consistency across the two hub edges}]
\label{ass:hsa_source_model}
All variables have finite second moments, $\mathbb E[\mathbf L]=\mathbf0$, and $\operatorname{Cov}(\mathbf L)=I_q$. For each
$(m,\mathbf M)\in\{(A,\mathbf A),(H,\mathbf H),(B,\mathbf B)\}$,
$\mathbb E[\mathbf M\mid\mathbf L]=\boldsymbol\mu_m+G_m\mathbf L$, and
$\boldsymbol\varepsilon_m:=\mathbf M-\mathbb E[\mathbf M\mid\mathbf L]$ satisfies
$\mathbb E[\boldsymbol\varepsilon_m\mid\mathbf L]=\mathbf0$.
The \textcolor{black}{observed hub edges} satisfy
$\operatorname{Cov}(\boldsymbol\varepsilon_A,\boldsymbol\varepsilon_H)=0$ and
$\operatorname{Cov}(\boldsymbol\varepsilon_H,\boldsymbol\varepsilon_B)=0$.
Their samples provide compatible estimates of the population means and the second-order statistics of the hub.
\end{assumption}

\hsaAssumptionRef{ass:hsa_source_model} yields
\begin{equation}
\begin{array}{@{}r@{\;}c@{\;}l@{\qquad}r@{\;}c@{\;}l@{}}
\mathbf A & = & \boldsymbol\mu_A+G_A\mathbf L+\boldsymbol\varepsilon_A,
& \mathbf H & = & \boldsymbol\mu_H+G_H\mathbf L+\boldsymbol\varepsilon_H,\\
\mathbf B & = & \boldsymbol\mu_B+G_B\mathbf L+\boldsymbol\varepsilon_B,
& \mathcal K_{AB} & = & G_AG_B^\top.
\end{array}
\label{eq:hsa_factor_knowledge}
\end{equation}
The population target cross-covariance may also contain target-residual dependence:
$\Sigma_{AB}=\mathcal K_{AB}+\Omega_{AB}$, where
$\Omega_{AB}:=\operatorname{Cov}(\boldsymbol\varepsilon_A,\boldsymbol\varepsilon_B)$ is unrestricted. This separation ties the recovery target to the common source.
The derivation of this decomposition and the exact information interface supplied by the
two hub edges are detailed in the appendix.

\section{Hub-Spectral Activation}
\label{sec:method}

HSA proceeds from identification to task readout, as summarized in \hsaFigRef{fig:hsa_method_overview}. It first isolates the relation determined by the two observed hub edges, identifies paired spectral carriers that preserve this relation, converts their sample-wise evidence into a calibrated score, and applies that score to retrieval rankings or class prototypes. During fitting, HSA estimates the carrier matrices from sample
statistics of the two hub-edge datasets. During task readout,
these matrices remain fixed as each input sample is mapped to
a carrier coordinate vector for query--candidate scoring.
\par

\begin{figure*}[t]
\centering
\includegraphics[width=\textwidth]{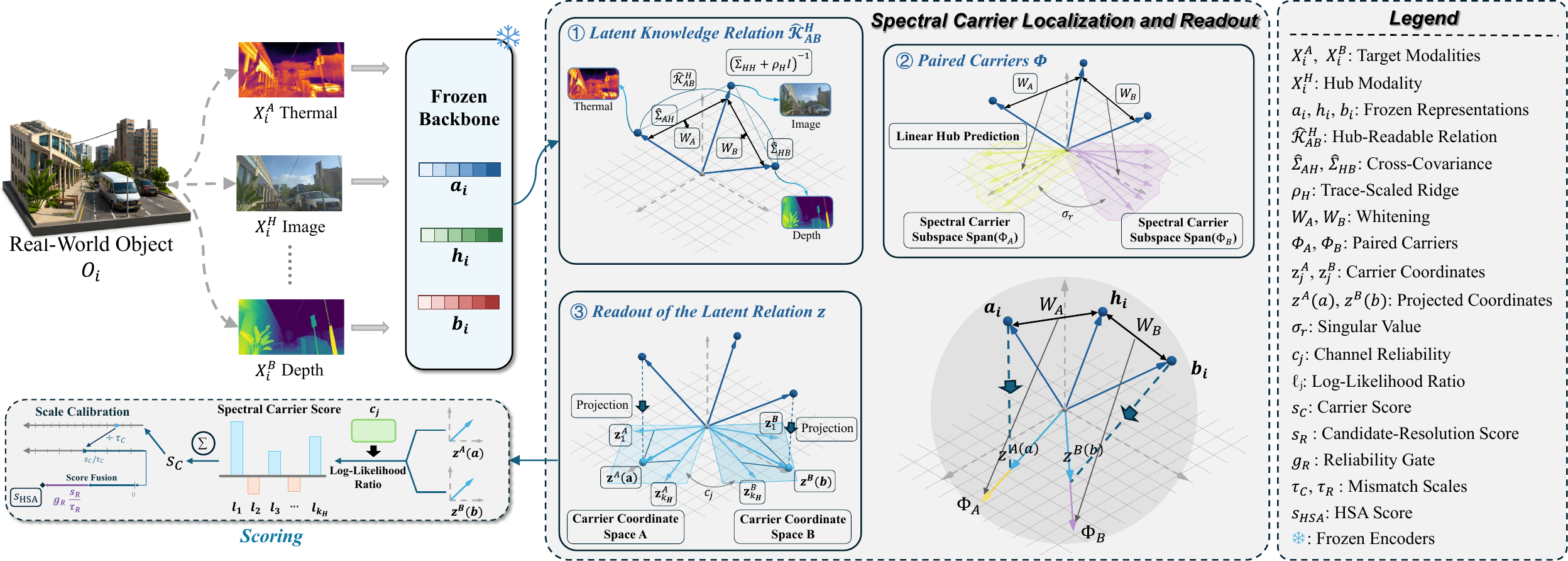}
\caption{\textbf{Overview of Hub-Spectral Activation.} Frozen representations from the observed $A$--$H$ and $H$--$B$ hub edges provide the moment systems used by HSA. Stages \ding{172}--\ding{174} construct and standardize the hub-readable relation $\widehat{\mathcal K}_{AB}^{H}$, identify paired carriers $\Phi_A$ and $\Phi_B$, and project candidates into readout coordinates $\mathbf z^A$ and $\mathbf z^B$. Their strength-aware agreement defines the carrier score $s_C$, which is calibrated
and combined with gated candidate resolution to form $s_{\mathrm{HSA}}$ for retrieval
and prototype classification.}
\label{fig:hsa_method_overview}
\end{figure*}

\subsection{\textcolor{black}{Hub-Readable Relation and Capacity}}
\label{sec:method_readout}

\noindent\textbf{Hub-Readable Relation.}
The available information contains the two observed hub edges while omitting the target cross-relation. The following theorem identifies the component determined by these edges and states when that component recovers the complete relation induced by the common source.
\begin{theorem}[Identifiable hub-readable latent knowledge]
\label{thm:hub_readable}
Under \hsaAssumptionRef{ass:hsa_source_model}, let
$\Psi_H=\operatorname{Cov}(\boldsymbol\varepsilon_H)$ and define
\begin{equation}
\begin{array}{@{}r@{\;}r@{\;}l@{}}
\mathcal P_H
& = & G_H^\top(G_HG_H^\top+\Psi_H)^{\dagger}G_H,
\qquad 0\preceq\mathcal P_H\preceq I_q,\\
\mathbf A^H
& = & \Sigma_{AH}\Sigma_{HH}^{\dagger}
(\mathbf H-\boldsymbol\mu_H),\\
\mathbf B^H
& = & \Sigma_{BH}\Sigma_{HH}^{\dagger}
(\mathbf H-\boldsymbol\mu_H),\\
\boxed{\mathcal K_{AB}^{H}}
& := & \Sigma_{AH}\Sigma_{HH}^{\dagger}\Sigma_{HB},\\
& = & \operatorname{Cov}(\mathbf A^H,\mathbf B^H)
=G_A\mathcal P_HG_B^\top.
\end{array}
\label{eq:hsa_hub_visible_knowledge}
\end{equation}
Then
\begin{equation}
\mathcal K_{AB}
=\mathcal K_{AB}^{H}
+G_A(I_q-\mathcal P_H)G_B^\top,
\label{eq:hsa_knowledge_decomposition}
\end{equation}
and exact hub readout holds if and only if
$G_A(I_q-\mathcal P_H)G_B^\top=0$.
The component $\mathcal K_{AB}^{H}$ is identified by $\mathfrak S_H$, while the same interface generally underidentifies $\mathcal K_{AB}$.
\end{theorem}

The matrix product in \hsaEqRef{eq:hsa_hub_visible_knowledge} is the covariance of the two best linear hub predictions, giving it a classical regression interpretation. \hsaTheoremRef{thm:hub_readable} specifies its role under the restricted information interface: $\mathcal P_H$ measures how much of the shared source state is linearly resolved by the frozen hub, and $\mathcal K_{AB}^{H}$ transports that resolved component into the two target spaces. \hsaEquationRef{eq:hsa_knowledge_decomposition} fixes the claim boundary by separating this identified component from the generally underidentified remainder of $\mathcal K_{AB}$. The theorem's proof and an explicit non-identifiability construction are provided in the appendix.
\par

\noindent\textbf{Hub Capacity.}
The same result limits how many independent relation channels the hub can expose.
\begin{corollary}[Hub-capacity bound]
\label{cor:hub_capacity}
The hub-readable relation satisfies
\begin{equation}
\operatorname{rank}(\mathcal K_{AB}^{H})
\leq \operatorname{rank}(\Sigma_{HH}).
\label{eq:hsa_hub_rank_bound}
\end{equation}
If the hub representation has at most $N_H$ distinct population states, then
$\operatorname{rank}(\mathcal K_{AB}^{H})\leq N_H-1$.
\end{corollary}

\subsection{\textcolor{black}{Spectral Carrier Identification}}
\label{sec:method_carrier}

\noindent\textbf{\ding{172} Standardized Hub Operator.}
HSA next estimates and standardizes this operator. Applying the standard leading-singular-subspace variational
principle~\cite{golub2013matrix} yields paired low-dimensional directions that preserve the greatest relation strength at a fixed dimension.
Let $\widehat\Sigma_{HH}^{A}$ and $\widehat\Sigma_{HH}^{B}$ denote the hub covariances estimated on the two edges. HSA pools them with equal edge weight as
$\overline\Sigma_{HH}=\tfrac12(\widehat\Sigma_{HH}^{A}+\widehat\Sigma_{HH}^{B})$.
It uses the trace-scaled ridges
$\rho_X=\lambda\operatorname{tr}(\widehat\Sigma_{XX})/d_X$ for $X\in\{A,B\}$ and
$\rho_H=\lambda\operatorname{tr}(\overline\Sigma_{HH})/d_H$.
The empirical hub-readable relation and its standardized decomposition are
\begin{equation}
\begin{array}{@{}r@{\;}r@{\;}l@{}}
\widehat{\mathcal K}_{AB}^{H}
& = & \widehat\Sigma_{AH}
(\overline\Sigma_{HH}+\rho_HI)^{-1}
\widehat\Sigma_{HB},\\
T_H
& := & W_A\widehat{\mathcal K}_{AB}^{H}W_B^\top
=U_C\operatorname{diag}(\sigma_1,\ldots,\sigma_r)V_C^\top,\\
W_X & = & (\widehat\Sigma_{XX}+\rho_XI)^{-1/2}.
\end{array}
\label{eq:hsa_spectral_carrier}
\end{equation}
A rounded-state empirical proxy motivated by \hsaCorollaryRef{cor:hub_capacity} sets a heuristic cap on the carrier count $k_H$. Details of this rule and its numerical safeguards are provided in the appendix.

\noindent\textbf{\ding{173} Optimal Paired Carriers.}
\begin{proposition}[Optimal paired spectral carriers]
\label{prop:optimal_carrier}
For any $1\leq k\leq r$, among all $k$-dimensional orthonormal paired projections in the standardized target spaces,
\begin{equation}
\max_{\substack{P^\top P=I_k\\Q^\top Q=I_k}}
\operatorname{tr}(P^\top T_HQ)
=\sum_{j=1}^{k}\sigma_j.
\label{eq:hsa_carrier_optimality}
\end{equation}
The maximum is attained by $P=U_{C,k}$ and $Q=V_{C,k}$. Mapping these directions to the original target coordinates as
$\Phi_A^{(k)}=W_AU_{C,k}$ and
$\Phi_B^{(k)}=W_BV_{C,k}$ gives
\begin{equation}
(\Phi_A^{(k)})^\top
\widehat{\mathcal K}_{AB}^{H}
\Phi_B^{(k)}
=\operatorname{diag}(\sigma_1,\ldots,\sigma_k).
\label{eq:hsa_carrier_diagonalization}
\end{equation}
\end{proposition}

\noindent\textbf{\ding{174} Carrier Coordinates.}
\hsaPropositionRef{prop:optimal_carrier} turns the recovered dense operator into ordered coordinate-wise evidence: the chosen pair of subspaces maximizes retained hub-readable relation strength, and each resulting channel has the known strength $\sigma_j$. HSA sets $k=k_H$, writes $\Phi_A=\Phi_A^{(k_H)}$ and $\Phi_B=\Phi_B^{(k_H)}$, and defines the projected knowledge coordinates
\begin{equation}
\mathbf z^A(\mathbf a)
=\Phi_A^\top(\mathbf a-\widehat{\boldsymbol\mu}_A),
\qquad
\mathbf z^B(\mathbf b)
=\Phi_B^\top(\mathbf b-\widehat{\boldsymbol\mu}_B).
\label{eq:hsa_knowledge_readout}
\end{equation}
Each of the $k_H$ columns of $\Phi_A$ represents a selected
carrier direction. The component $z_j^A(\mathbf a)$ is the
inner product of $\mathbf a-\widehat{\boldsymbol\mu}_A$
with column $j$, for $j=1,\ldots,k_H$. These components
form the coordinate vector $\mathbf z^A(\mathbf a)$,
with the same construction applying to modality $B$. Together, these steps perform spectral carrier localization and readout: the paired singular directions locate the hub-readable relation, while the projected coordinates expose it as sample-wise evidence. The retained singular triplets also give the best rank-$k_H$ Frobenius approximation of $T_H$. The proof of the proposition and derivation of the sample score below are provided in the appendix.

{\color{black}
\makeatletter
\renewcommand\fs@ruled{%
  \def\@fs@cfont{\color{black}\bfseries}%
  \def\@fs@capt##1##2{{\color{black}\bfseries ##1} {\color{black}##2}\par}%
  \def\@fs@pre{{\color{black}\hrule height.8pt depth0pt}\kern2pt}%
  \def\@fs@mid{\kern2pt{\color{black}\hrule}\kern2pt}%
  \def\@fs@post{\kern2pt{\color{black}\hrule}\relax}%
  \let\@fs@iftopcapt\iftrue}
\makeatother
\begin{algorithm}[t]
\color{black}
\small
\caption{HSA: Identify, Calibrate, and Read Out}
\label{alg:hsa}
\begin{algorithmic}[1]
\REQUIRE Hub-edge data, task readout sets, and registered HSA settings
\ENSURE Retrieval rankings or prototype-classification labels
\STATE \textbf{Identify:} estimate $\widehat{\mathfrak S}_H$ from the two hub edges as in \hyperref[eq:hsa_information_interface]{Eq.~(\ref*{eq:hsa_information_interface})}.
\STATE Compute $\widehat{\mathcal K}_{AB}^{H}$, $W_A$, $W_B$, and $T_H$ using \hyperref[eq:hsa_spectral_carrier]{Eq.~(\ref*{eq:hsa_spectral_carrier})}.
\STATE Select $k_H$ by the hub-capacity rule (see the appendix) and obtain $\Phi_A,\Phi_B$ from \hyperref[prop:optimal_carrier]{Proposition~\ref*{prop:optimal_carrier}}.
\STATE Define carrier evidence $s_C$ using \hyperref[eq:hsa_carrier_score]{Eq.~(\ref*{eq:hsa_carrier_score})}.
\STATE Select $k_R$ by the shuffled-edge null (see the appendix) and define $s_R$ using \hyperref[eq:hsa_complementary_score]{Eq.~(\ref*{eq:hsa_complementary_score})}.
\STATE \textbf{Calibrate:} estimate $\widehat\Omega_A,\widehat\Omega_B$ and compute $g_R$ using \hyperref[eq:hsa_visibility_gate]{Eq.~(\ref*{eq:hsa_visibility_gate})}.
\STATE Use five fixed derangements of the equal-count target marginals in $\mathcal D_{AH},\mathcal D_{HB}$ to compute $\tau_C,\tau_R$ using \hyperref[eq:hsa_mismatch_scale]{Eq.~(\ref*{eq:hsa_mismatch_scale})}.
\STATE \textbf{Read Out:} map $\mathcal A_{\mathrm{ro}},\mathcal B_{\mathrm{ro}}$ using \hyperref[eq:hsa_knowledge_readout]{Eq.~(\ref*{eq:hsa_knowledge_readout})} and evaluate all pairs with \hyperref[eq:complete_hsa]{Eq.~(\ref*{eq:complete_hsa})}.
\STATE For retrieval, form score matrices and bidirectional rankings using \hyperref[eq:hsa_score_matrices]{Eqs.~(\ref*{eq:hsa_score_matrices})}--\hyperref[eq:hsa_bidirectional_rankings]{(\ref*{eq:hsa_bidirectional_rankings})}.
\STATE For classification, treat $\mathcal A_{\mathrm{ro}}$ as the prototype bank and predict labels using \hyperref[eq:hsa_prototype_classification]{Eq.~(\ref*{eq:hsa_prototype_classification})}.
\RETURN The task-specific rankings or predicted labels
\end{algorithmic}
\end{algorithm}
}

\subsection{\textcolor{black}{HSA Scoring and Task Readouts}}
\label{sec:method_score}

\noindent\textbf{Carrier Evidence.} The paired carriers convert the hub-readable relation into aligned coordinates. HSA scores a candidate pair by asking whether these coordinates are more likely under a matched model than under a mismatched model. Set $c_j=\operatorname{clip}(\sigma_j,0,1-10^{-6})$. For channel $j$, HSA models a matched coordinate pair as a standardized bivariate Gaussian with correlation $c_j$ and a mismatched pair as independent standardized coordinates. Let
\begin{equation*}
\ell_j(x,y;c_j)
:=\log\frac{p([x,y]^\top\mid Y=1)}
{p([x,y]^\top\mid Y=0)}
\end{equation*}
be the coordinate-wise log-likelihood ratio. The spectral carrier score is
\begin{equation}
s_C(\mathbf a,\mathbf b)
=\sum_{j=1}^{k_H}
\ell_j\!\left(
z_j^A(\mathbf a),z_j^B(\mathbf b);c_j
\right).
\label{eq:hsa_carrier_score}
\end{equation}
The singular value therefore controls how strongly each carrier coordinate contributes to the same-source evidence. The closed-form expansion is provided in the appendix, where the Gaussian working model is confined to these standardized coordinates.

\noindent\textbf{Candidate Resolution.} Several gallery items can share similar carrier coordinates, so HSA distinguishes them in the orthogonal complement of the leading raw relation subspace. A fixed shuffled-edge null determines the subspace rank $k_R$; details of the selection rule and numerical safeguards are provided in the appendix. The normalized vectors $\mathbf r^A$ and $\mathbf r^B$ provide candidate-resolution coordinates, while $s_R$ measures fine-grained query-candidate compatibility after removing that subspace.
Let
$\widehat{\mathcal K}_{AB}^{H}=U_R\operatorname{diag}(\eta)V_R^\top$,
form the rank-$k_R$ projectors $\Pi_A^R$ and $\Pi_B^R$, and define
\begin{equation}
\begin{array}{@{}r@{\;}c@{\;}l@{}}
\mathbf r^A(\mathbf a)
& = & \operatorname{nrm}\!\left((I-\Pi_A^R)
(\mathbf a-\widehat{\boldsymbol\mu}_A)\right),\\
\mathbf r^B(\mathbf b)
& = & \operatorname{nrm}\!\left((I-\Pi_B^R)
(\mathbf b-\widehat{\boldsymbol\mu}_B)\right),\\
s_R(\mathbf a,\mathbf b)
& = & \mathbf r^A(\mathbf a)^\top\mathbf r^B(\mathbf b).
\end{array}
\label{eq:hsa_complementary_score}
\end{equation}
\noindent\textbf{Reliability and Calibration.} The candidate-resolution term should be used only when residual geometry is organized consistently across the two target spaces. HSA measures this reliability using only the observed hub edges. Let $\widehat\Omega_A$ and $\widehat\Omega_B$ be the ridge-regularized covariance matrices. Applying the same feature permutation $P_\pi$ to the rows and columns of $\widehat\Omega_B$ preserves its spectrum but breaks its coordinate correspondence with $\widehat\Omega_A$. The excess of the observed affinity over this permutation null defines a reliability gate:
\begin{equation}
\begin{aligned}
\alpha_R
&=
\frac{\langle\widehat\Omega_A,\widehat\Omega_B\rangle_F}
{\lVert\widehat\Omega_A\rVert_F\lVert\widehat\Omega_B\rVert_F},\\
q_R
&=
\operatorname{Quantile}_{0.95}
\left(\left\{
\frac{\langle\widehat\Omega_A,
P_\pi\widehat\Omega_BP_\pi^\top\rangle_F}
{\lVert\widehat\Omega_A\rVert_F\lVert\widehat\Omega_B\rVert_F}
\right\}_{\pi\in\Pi_G}\right),\\
g_R
&=
\operatorname{clip}\!\left(
\frac{\alpha_R-q_R}{1-q_R},0,1\right).
\end{aligned}
\label{eq:hsa_visibility_gate}
\end{equation}
The resulting gate uses only these hub-edge covariance matrices and leaves the carrier evidence unchanged. The edge-specific estimators, null construction, and zero-denominator rule are detailed in the appendix. In all reported runs, the two target-side training marginals have equal row counts: synchronized source rows supply equal-count marginals in the main protocol, and the disjoint-source control uses matched equal-count splits. Five fixed derangements of these marginals estimate mismatch scales $\tau_C$ and $\tau_R$, placing the carrier and candidate-resolution scores on comparable scales.
For $Q\in\{C,R\}$,
\begin{subequations}
\begin{gather}
\begin{aligned}
\tau_Q
&=\frac15\sum_{\pi\in\Pi_{\mathrm{cal}}}\\[-0.3ex]
&\quad{}
\max\!\left\{
\operatorname{Std}_{i}
[s_Q(\mathbf a_i,\mathbf b_{\pi(i)})],10^{-6}
\right\},
\end{aligned}
\label{eq:hsa_mismatch_scale}
\\[-0.5ex]
s_{\mathrm{HSA}}(\mathbf a,\mathbf b)
=\frac{s_C(\mathbf a,\mathbf b)}{\tau_C}
+g_R\frac{s_R(\mathbf a,\mathbf b)}{\tau_R}.
\label{eq:complete_hsa}
\end{gather}
\end{subequations}
\noindent\textbf{Bidirectional Ranking.} The calibrated score defines one match matrix: its rows rank $B$ candidates for $A$ queries, and its columns rank $A$ candidates for $B$ queries.
\begin{gather}
\begin{array}{@{}r@{\;}c@{\;}l@{}}
[S^{A\rightarrow B}]_{ij}
& = & s_{\mathrm{HSA}}(\mathbf a_i,\mathbf b_j),\\
S^{B\rightarrow A} & = & (S^{A\rightarrow B})^\top.
\end{array}
\label{eq:hsa_score_matrices}
\\[-0.5ex]
\begin{array}{@{}r@{\;}c@{\;}l@{}}
\pi_i^{A\rightarrow B}
& = & \operatorname{argsort}_{j}^{\downarrow}[S^{A\rightarrow B}]_{ij},\\
\pi_j^{B\rightarrow A}
& = & \operatorname{argsort}_{i}^{\downarrow}[S^{A\rightarrow B}]_{ij}.
\end{array}
\label{eq:hsa_bidirectional_rankings}
\end{gather}
\noindent\textbf{Prototype Classification.}
The same fixed pair score supports class prediction. Let $\mathbf p_c^A$ be a normalized class prototype in modality $A$, formed from labeled $A$-side training features or from a fixed bank of class prompts when text provides the prototype modality. A test query $\mathbf b$ from modality $B$ is classified by
\begin{equation}
\widehat y(\mathbf b)
=\underset{c\in\mathcal C}{\operatorname{argmax}}\;
s_{\mathrm{HSA}}(\mathbf p_c^A,\mathbf b).
\label{eq:hsa_prototype_classification}
\end{equation}
Within each relation, all compared methods use the same prototype bank. Class identities determine this bank; target test outcomes and $A$--$B$ pair identities do not enter HSA fitting or calibration. All directions, ranks, weights, and scales are fixed before target evaluation. When encoder coordinates are compatible, the relation readout for the same target-modality pair can support both retrieval and prototype classification with all fitted parameters fixed. The main classification comparisons estimate the relation and construct its readout using the specified source features. Further validation of both task readouts under a single fitted state is provided in the appendix. \hsaAlgorithmRef{alg:hsa} summarizes the two readouts; details of the score-block implementation are provided in the appendix.

\subsection{\textcolor{black}{Further Analysis}}
\label{sec:method_analysis}

\noindent\textbf{Knowledge, Recoverability, and Visibility.}
\hsaDefinitionRef{def:latent_knowledge} and \hsaTheoremRef{thm:hub_readable} separate the full source-induced relation $\mathcal K_{AB}$ from the hub-readable component $\mathcal K_{AB}^{H}$. Recoverability depends on which source factors the hub resolves, while native visibility further depends on how the paired directions of $\mathcal K_{AB}^{H}$ align with the cosine axes. Nonzero singular values can therefore coexist with weak native retrieval when the left and right directions couple different axes.
\par

\noindent\textbf{Why Spectral Activation Exposes the Relation.} HSA resolves the coordinate mismatch before scoring. \hsaPropositionRef{prop:optimal_carrier} selects paired singular directions that retain maximal relation strength at a fixed dimension. \hsaEquationRef{eq:hsa_carrier_diagonalization} represents the relation as diagonal channels with strengths $\sigma_j$, while \hsaEqRef{eq:hsa_carrier_score} converts channel-wise agreement into strength-aware sample evidence. Spectral activation thereby maps cross-coordinate coupling into matched carrier channels for direct scoring.
\par

\noindent\textbf{Testable Consequences.}
Three structural predictions follow. Shuffling either hub edge should weaken the
relation; gains should persist with disjoint source instances, up to finite-sample
variation. Preserving cross-coordinate hub correlations should improve readout.
Leading carriers should outperform subsequent or random carriers when counts and
assigned reliability spectra are matched. The scoring rule predicts gains from adding
informative carriers and increasing their reliability until estimation noise
dominates. Gating and candidate resolution should help when residual geometry is
reliable. \hsaSectionRef{sec:exp_analysis} evaluates these predictions through
correspondence, disjoint-source, spectral, scaling, and component-ablation
experiments.
\par

\noindent\textbf{Scope.}
These claims are limited to second-order relations that can be linearly resolved through the frozen hub. The identifiable component $\mathcal K_{AB}^{H}$ may cover only part of $\mathcal K_{AB}$, and the Gaussian model applies only to the standardized carrier coordinates. Relation formation during training remains outside the present scope.
\par

\section{Experiments}
\label{sec:experiments}

{\color{black}
The experiments address five questions:
\begin{enumerate}
    \item[\textbf{Q1}] \textbf{Effect:} Does HSA improve held-out retrieval across relations, backbones, and directions, while also improving prototype classification across relations and backbones?
    \item[\textbf{Q2}] \textbf{Source:} Given the retrieval effect, does it require valid correspondence within both hub edges and persist without shared source instances?
    \item[\textbf{Q3}] \textbf{Location:} Given its source, which hub-covariance structure and spectral carriers contain the usable relation, and how concentrated is it?
    \item[\textbf{Q4}] \textbf{Activation:} Given the localized carriers, how do reliability, gating, and candidate resolution turn them into ranking gains?
    \item[\textbf{Q5}] \textbf{Cost:} Once the readout is specified, what fitting, retrieval-scoring, and classification-readout overhead does it introduce?
\end{enumerate}
Each answer supplies the premise for the next; the following settings define their common protocol.
\par}

\begin{table*}[t]
\caption{\textbf{Held-Out Retrieval on Nine ImageBind Relations.} Boldface and underlining mark the best and second-best values among all methods except Full A--B. $^\dagger$ marks methods that use target-pair identities; blue parentheses show HSA gains over frozen cosine.}
\label{tab:imagebind_main}
\centering
\scriptsize
\setlength{\tabcolsep}{1.6pt}
\renewcommand{\arraystretch}{1.10}
\begin{adjustbox}{Clip=0pt 0pt 0pt 0pt}
\begin{tabularx}{\textwidth}{@{}l*{13}{Y}@{}}
\toprule
\tworowhead{Method} &
\multicolumn{9}{c}{Evaluation Relations: Bidirectional Recall@10 (\%) $\uparrow$} &
\tworowhead{\shortstack{Mean\\R@1}} &
\tworowhead{\shortstack{Mean\\R@5}} &
\tworowhead{\shortstack{Mean\\R@10}} &
\tworowhead{\shortstack{R@10\\Recovery}} \\
\cmidrule(lr){2-10}
&
\shortstack{Tx--Au\\VGS} &
\shortstack{Tx--Th\\TR} &
\shortstack{Tx--IMU\\E4D} &
\shortstack{Au--D\\BV} &
\shortstack{Au--Th\\MAVD} &
\shortstack{Au--IMU\\E4D} &
\shortstack{D--Th\\TR} &
\shortstack{D--IMU\\UTD} &
\shortstack{Th--IMU\\CA} & & & & \\
\midrule
\rowcolor{gray!12}
Full A--B$^\dagger$ & 84.57 & 74.71 & 4.01 & 7.13 & 3.09 & 7.15 & 15.76 & 6.01 & 4.15 & 16.28 & 20.34 & 22.95 & 100.0 \\
\midrule
Frozen Cosine & 75.20 & 7.68 & 1.53 & 1.97 & 1.62 & 1.75 & 0.34 & 3.02 & 4.08 & 7.65 & 9.80 & 10.80 & 47.0 \\
Hub-Relative~\cite{relrep} & 36.85 & 12.67 & 1.46 & 1.88 & 2.20 & 2.26 & 0.27 & 2.56 & 4.08 & 2.88 & 4.70 & 7.14 & 31.1 \\
Bi. Ridge~\cite{hoerl1970ridge} & 82.05 & 70.52 & \underline{6.85} & \underline{6.85} & \textbf{4.63} & 6.57 & 11.71 & \underline{7.21} & \underline{13.88} & \underline{15.24} & \underline{20.04} & \underline{23.36} & \underline{101.8} \\
Bi. Procrustes~\cite{schonemann1966} & 83.30 & 61.19 & 4.74 & 5.65 & 2.43 & 6.88 & 4.35 & 3.14 & 8.16 & 12.91 & 17.39 & 19.98 & 87.1 \\
\midrule
ReAlign~\cite{realign} & 75.60 & 9.10 & 1.82 & 2.91 & 2.31 & 1.95 & 0.33 & 1.74 & 4.29 & 7.77 & 9.63 & 11.12 & 48.4 \\
ERM~\cite{domainbed} & 79.28 & 71.46 & 3.67 & 5.91 & \underline{3.74} & \underline{7.26} & 12.84 & 5.47 & 4.63 & 14.97 & 18.99 & 21.58 & 94.0 \\
IRM~\cite{irm} & 79.75 & \underline{73.08} & 1.85 & 5.28 & 3.09 & 3.22 & \textbf{14.85} & 2.36 & 4.01 & 15.21 & 18.86 & 20.83 & 90.8 \\
VREx~\cite{vrex} & 81.18 & 68.57 & 4.64 & 5.91 & 2.70 & 5.24 & 12.97 & 6.94 & 4.76 & 14.43 & 18.98 & 21.43 & 93.4 \\
DANN~\cite{dann} & 77.90 & 70.14 & 4.20 & 5.74 & 3.43 & 6.54 & 13.10 & 5.19 & 3.88 & 14.34 & 18.62 & 21.12 & 92.0 \\
CORAL~\cite{deepcoral} & 79.28 & 71.46 & 3.57 & 5.91 & 3.70 & 6.26 & 12.84 & 5.62 & 4.97 & 15.00 & 18.96 & 21.51 & 93.7 \\
ASIF$^\dagger$~\cite{asif} & \underline{83.40} & 66.47 & 4.96 & 6.76 & 3.59 & 5.24 & \underline{14.15} & 3.26 & 4.29 & 14.02 & 19.25 & 21.35 & 93.0 \\
Paired-OP$^\dagger$~\cite{schonemann1966} & 82.95 & 59.07 & 3.86 & 5.39 & 2.43 & 4.52 & 4.37 & 3.84 & 6.94 & 12.35 & 17.10 & 19.26 & 83.9 \\
\midrule
\rowcolor{hsarow}
\gainrowcell{HSA (Ours)} & \gainmaincell{\textbf{84.45}}{+9.25} & \gainmaincell{\textbf{74.87}}{+67.20} & \gainmaincell{\textbf{8.97}}{+7.43} & \gainmaincell{\textbf{6.93}}{+4.97} & \gainmaincell{\textbf{4.63}}{+3.01} & \gainmaincell{\textbf{8.93}}{+7.19} & \gainmaincell{9.66}{+9.32} & \gainmaincell{\textbf{9.30}}{+6.28} & \gainmaincell{\textbf{16.33}}{+12.24} & \gainmaincell{\textbf{16.39}}{+8.73} & \gainmaincell{\textbf{21.53}}{+11.73} & \gainmaincell{\textbf{24.90}}{+14.10} & \gainrowcell{\textbf{108.5}} \\
\bottomrule
\end{tabularx}
\end{adjustbox}
\end{table*}

\subsection{Experimental Settings}
\label{sec:exp_setup}

Figure and table abbreviations are Tx (text), Im (image), Vi (video), Au (audio), D (depth), Th (thermal), and IMU (inertial measurement).

{\color{black}
\noindent\textbf{Scope and Data.}
We evaluate frozen ImageBind~\cite{imagebind} and LanguageBind~\cite{languagebind} on 19 nondegenerate held-out relations spanning seven modalities and ten datasets: VGGSound~\cite{vggsound}, NYUv2~\cite{nyuv2}, TartanRGBT distributed with AnyThermal~\cite{tartanrgbt}, Ego4D~\cite{ego4d}, BatVision~\cite{batvision}, MAVD~\cite{mavd}, UTD-MHAD~\cite{utdmhad}, Caltech Aerial RGBT~\cite{caltechaerialrgbt}, UCF101~\cite{ucf101}, and MSR-VTT~\cite{msrvtt}. An evaluation relation is a backbone--dataset--target-pair configuration. For each of the 19 retrieval relations, all methods share the same training and test partitions, test queries, and complete test gallery. The classification study reports eleven class-labeled relations, with six under ImageBind and five under LanguageBind. Class definitions, splits, and counts are detailed in the appendix.

\noindent\textbf{Comparisons and Information Boundary.}
The main tables distinguish source-only analytic, adapted trainable, target-paired, and HSA groups. Analytic references are frozen cosine, hub-relative similarity~\cite{relrep}, bidirectional ridge~\cite{hoerl1970ridge}, bidirectional Procrustes~\cite{schonemann1966}, and ReAlign~\cite{realign}, which applies Anchor, Trace, and Centroid Alignment to unpaired $A$ and $B$ training marginals. ERM~\cite{domainbed}, IRM~\cite{irm}, VREx~\cite{vrex}, DANN~\cite{dann}, and CORAL~\cite{deepcoral} use the same two-layer head architecture with the $A$--$H$ and $H$--$B$ edges as training environments; respectively, they minimize mean edge loss, enforce invariant optimality, penalize risk variance, adversarially suppress modality identity, and align covariances with the Deep CORAL penalty. Target-paired references are ASIF~\cite{asif}, Paired-OP~\cite{schonemann1966}, and Full A--B: ASIF builds a coupled dictionary, Paired-OP solves two orthogonal maps, and Full A--B trains heads with the same architecture. All compared methods keep the backbone frozen. Methods without target-pair supervision exclude $A$--$B$ identities and target-informed model selection. For classification, every method scores the same $A$-side class prototypes against $B$-side test queries. ASIF and Paired-OP take zero gradient steps; Full A--B follows the fixed training protocol below.

\noindent\textbf{Metrics and Reporting.}
Bidirectional Recall@$k$ averages both retrieval directions for $k\in\{1,5,10\}$; Recall@10 is primary. Recovery@10 divides each method's backbone-mean Recall@10 by Full A--B. Classification uses $B\!\to\!A$ prototype prediction, and its primary metric is macro Top-1, the mean of class-wise Top-1 accuracies. Retrieval aggregates weight the 19 relations equally, while the classification aggregate weights the eleven reported relations equally. ERM, IRM, VREx, DANN, CORAL, and Full A--B average seeds 40--42. Full A--B uses a single 200-step duration selected globally from train-only pilot splits and fixed across relations and seeds. Complete retrieval and classification protocols, including datasets, class definitions, and classification eligibility, are provided in the appendix. Further
retrieval results, seed variability, aggregation checks, diagnostics, and the excluded degenerate unit are documented there.
\par}

\Needspace{8\baselineskip}
\noindent\textbf{Across-Dataset Statistical Analysis.}
We assess cross-dataset stability by grouping the 19 retrieval
relations into ten datasets and averaging seeds within each
relation. Relation-equal mean differences and 95\% percentile
intervals use 100,000 dataset-cluster bootstrap draws (seed 42),
retaining all relations of each sampled dataset. One-sided exact
sign-flip tests act on entire dataset groups, with Holm correction
within predefined comparison families. These tests assume
independent dataset groups and joint sign symmetry. The appendix
reports dataset-equal estimates, leave-one-dataset-out ranges,
complete families, and the interpretation of within-relation
intervals.

\Needspace{6\baselineskip}
\subsection{\textcolor{black}{Does HSA Improve Held-Out Performance? (Q1)}}
\label{sec:exp_main}

{\color{black}
Q1 evaluates retrieval and prototype classification across
relations and backbones, including both retrieval directions.
Full A--B is the supervised reference, excluded from boldface
and underlining.}
\hyperref[tab:imagebind_main]{Tables~\ref*{tab:imagebind_main} }and~\hyperref[tab:languagebind_main]{\ref*{tab:languagebind_main}} compare all
methods on the nineteen evaluation relations. ReAlign raises cross-backbone
mean Recall@10 from 18.27\% with frozen cosine to 18.48\%. HSA reaches 31.15\%,
exceeds ReAlign on all nineteen relations by 12.67 points on average, improves
every relation, and attains the highest or tied-highest Recall@10 among all methods
except Full A--B on fifteen. Across the ten dataset groups, HSA exceeds frozen cosine by
12.88 points (95\% cluster interval [5.29, 18.68];
exact $p=0.00098$).
All nine comparisons with the other source-only methods remain
positive after family-wise Holm correction
($p_{\mathrm{Holm}}\leq0.00879$). The same closed-form readout
therefore operates across held-out relations formed through both image and
language hubs.
The gain remains positive under dataset-balanced weighting, every single-dataset
exclusion, and chance normalization; complete definitions and results are provided in
the appendix.

\begin{table}[t]
\caption{\textbf{ImageBind Prototype Classification.} Macro Top-1 (\%) on six relations. Boldface and underlining mark the best and second-best values among all methods except Full A--B. $^\dagger$ marks methods that use target-pair identities; blue parentheses show HSA gains over frozen cosine.}
\label{tab:imagebind_classification}
\centering
\scriptsize
\setlength{\tabcolsep}{1.5pt}
\renewcommand{\arraystretch}{1.10}
\begin{adjustbox}{Clip=0pt 0pt 0pt 0pt}
\begin{tabularx}{\linewidth}{@{}l*{6}{Y}@{}}
\toprule
\tworowhead{Method} &
\multicolumn{6}{c}{Prototype--Query Relations: Macro Top-1 (\%) $\uparrow$} \\
\cmidrule(lr){2-7}
&
\shortstack{Tx--Au\\VGS} &
\shortstack{Tx--Th\\TR} &
\shortstack{Tx--IMU\\E4D} &
\shortstack{Au--D\\BV} &
\shortstack{Au--IMU\\E4D} &
\shortstack{D--IMU\\UTD} \\
\midrule
\rowcolor{gray!12}
Full A--B$^\dagger$ & 33.57 & 78.13 & 19.57 & 53.64 & 19.21 & 27.00 \\
\midrule
Frozen Cosine & 27.41 & 28.10 & 7.16 & 25.77 & 6.56 & 2.87 \\
Hub-Relative~\cite{relrep} & 11.07 & 27.81 & 5.73 & 18.41 & 5.88 & 2.78 \\
Bi. Ridge~\cite{hoerl1970ridge} & 31.25 & \underline{84.43} & 19.86 & \underline{58.78} & 20.48 & 25.63 \\
Bi. Procrustes~\cite{schonemann1966} & 30.22 & 69.11 & 16.73 & 50.36 & 13.95 & 10.52 \\
\midrule
ReAlign~\cite{realign} & 27.08 & 30.45 & 10.70 & 34.18 & 10.23 & 4.91 \\
ERM~\cite{domainbed} & 29.65 & 68.89 & 19.73 & 53.78 & 20.29 & 24.72 \\
IRM~\cite{irm} & 31.00 & 72.55 & 15.52 & 52.31 & 10.23 & 5.36 \\
VREx~\cite{vrex} & 29.66 & 75.23 & 18.37 & 52.52 & 15.94 & 12.16 \\
DANN~\cite{dann} & 29.33 & 62.80 & \underline{20.90} & 49.51 & \underline{20.50} & 23.94 \\
CORAL~\cite{deepcoral} & 29.65 & 69.13 & 19.83 & 53.75 & 20.42 & \underline{27.29} \\
ASIF$^\dagger$~\cite{asif} & \underline{32.54} & 73.52 & 14.71 & 32.03 & 11.71 & 9.57 \\
Paired-OP$^\dagger$~\cite{schonemann1966} & 31.81 & 56.02 & 14.12 & 47.97 & 9.76 & 13.12 \\
\midrule
\rowcolor{hsarow}
\gainrowcell{HSA (Ours)} & \gainmaincell{\textbf{32.60}}{+5.19} & \gainmaincell{\textbf{90.92}}{+62.82} & \gainmaincell{\textbf{21.58}}{+14.42} & \gainmaincell{\textbf{61.40}}{+35.63} & \gainmaincell{\textbf{22.81}}{+16.25} & \gainmaincell{\textbf{34.75}}{+31.88} \\
\bottomrule
\end{tabularx}
\end{adjustbox}
\end{table}

\noindent\textbf{ImageBind.}
HSA reaches 24.90\% mean Recall@10, improving frozen cosine by 14.10 points and
bidirectional ridge by 1.53 points. Its smaller margin over ridge than on
LanguageBind is consistent with Q4, where candidate resolution and residual
orthogonalization mainly refine LanguageBind residual geometry.
Under the image hub, HSA improves all nine held-out relations, spanning the
natively visible text--audio relation and the near-chance text--thermal
relation, and achieves the highest mean Recall@1, Recall@5, and
Recall@10 among all methods except Full A--B. The largest increase is 67.20 points for text--thermal.
For ImageBind prototype classification, HSA reaches 44.01\% mean macro Top-1 accuracy across the six relations in \hsaTableRef{tab:imagebind_classification}. This is 27.70 points above frozen cosine and 3.94 points above bidirectional ridge. The gains cover text-, audio-, depth-, thermal-, and inertial-modality readouts across VGGSound, TartanRGBT, Ego4D, BatVision, and UTD-MHAD.

\begin{table*}[t]
\caption{\textbf{Held-Out Retrieval on Ten LanguageBind Relations.} Boldface and underlining mark the best and second-best values among all methods except Full A--B. $^\dagger$ marks methods that use target-pair identities; blue parentheses show HSA gains over frozen cosine.}
\label{tab:languagebind_main}
\centering
\scriptsize
\setlength{\tabcolsep}{1.0pt}
\renewcommand{\arraystretch}{1.10}
\begin{adjustbox}{Clip=0pt 0pt 0pt 0pt}
\begin{tabularx}{\textwidth}{@{}l*{14}{Y}@{}}
\toprule
\tworowhead{Method} &
\multicolumn{10}{c}{Evaluation Relations: Bidirectional Recall@10 (\%) $\uparrow$} &
\tworowhead{\shortstack{Mean\\R@1}} &
\tworowhead{\shortstack{Mean\\R@5}} &
\tworowhead{\shortstack{Mean\\R@10}} &
\tworowhead{\shortstack{R@10\\Recovery}} \\
\cmidrule(lr){2-11}
&
\shortstack{Im--Vi\\UCF} &
\shortstack{Im--Th\\TR} &
\shortstack{Im--D\\NYU} &
\shortstack{Im--Au\\VGS} &
\shortstack{Vi--Th\\TR} &
\shortstack{Vi--D\\TR} &
\shortstack{Vi--Au\\MSR} &
\shortstack{Th--D\\TR} &
\shortstack{Th--Au\\MAVD} &
\shortstack{D--Au\\BV} & & & & \\
\midrule
\rowcolor{gray!12}
Full A--B$^\dagger$ & 99.76 & 60.36 & 44.19 & 41.02 & 54.43 & 25.84 & 26.34 & 31.31 & 2.73 & 5.82 & 15.28 & 30.85 & 39.18 & 100.0 \\
\midrule
Frozen Cosine & 96.49 & 36.93 & 11.62 & 43.40 & 36.01 & 7.11 & 6.96 & 7.57 & 2.08 & 1.88 & 10.28 & 19.78 & 25.00 & 63.8 \\
Hub-Relative~\cite{relrep} & 35.66 & 14.45 & 5.66 & 18.15 & 18.23 & 3.44 & 6.28 & 3.56 & 2.42 & 1.63 & 1.93 & 6.93 & 10.95 & 27.9 \\
Bi. Ridge~\cite{hoerl1970ridge} & 75.68 & 25.92 & 10.09 & 55.80 & 23.17 & 11.24 & 20.31 & 14.91 & 3.46 & 4.79 & 5.51 & 16.27 & 24.54 & 62.6 \\
Bi. Procrustes~\cite{schonemann1966} & 86.55 & 33.14 & 7.95 & 56.95 & 28.56 & 13.42 & \underline{20.53} & 14.45 & 3.12 & 3.85 & 7.25 & 18.99 & 26.85 & 68.5 \\
\midrule
ReAlign~\cite{realign} & 97.84 & 43.23 & 7.87 & 43.60 & 38.88 & 4.24 & 5.20 & 6.08 & 2.89 & 1.28 & 11.12 & 20.13 & 25.11 & 64.1 \\
ERM~\cite{domainbed} & 81.11 & 23.36 & 7.21 & 51.22 & 18.85 & 9.98 & 15.03 & 10.47 & \textbf{4.12} & 4.99 & 5.00 & 15.33 & 22.63 & 57.8 \\
IRM~\cite{irm} & 81.31 & 21.41 & 8.15 & 51.93 & 18.23 & 9.17 & 18.25 & 10.70 & \underline{4.08} & 4.05 & 5.11 & 15.60 & 22.73 & 58.0 \\
VREx~\cite{vrex} & 81.15 & 22.67 & 7.49 & 51.18 & 19.76 & 9.79 & 16.84 & 11.70 & \underline{4.08} & \underline{5.22} & 5.02 & 15.53 & 22.99 & 58.7 \\
DANN~\cite{dann} & 81.17 & 22.40 & 6.22 & 49.98 & 16.55 & 9.59 & 16.16 & 11.96 & 2.89 & 4.14 & 4.91 & 15.00 & 22.11 & 56.4 \\
CORAL~\cite{deepcoral} & 81.11 & 22.44 & 7.21 & 51.22 & 19.57 & 10.02 & 15.03 & 11.28 & 4.00 & 4.79 & 5.02 & 15.39 & 22.67 & 57.9 \\
ASIF$^\dagger$~\cite{asif} & 92.52 & 44.61 & \textbf{18.35} & \underline{64.20} & 36.93 & 16.63 & 18.27 & 18.69 & 2.31 & 4.28 & 10.07 & 23.86 & 31.68 & 80.9 \\
Paired-OP$^\dagger$~\cite{schonemann1966} & \underline{98.45} & \underline{49.31} & \underline{15.37} & 54.85 & \underline{44.27} & \underline{17.78} & 19.29 & \textbf{19.95} & 2.31 & 3.68 & \underline{12.66} & \underline{25.32} & \underline{32.53} & \underline{83.0} \\
\midrule
\rowcolor{hsarow}
\gainrowcell{HSA (Ours)} & \gainmaincell{\textbf{99.72}}{+3.23} & \gainmaincell{\textbf{60.89}}{+23.97} & \gainmaincell{13.53}{+1.91} & \gainmaincell{\textbf{67.15}}{+23.75} & \gainmaincell{\textbf{53.33}}{+17.32} & \gainmaincell{\textbf{18.00}}{+10.89} & \gainmaincell{\textbf{26.64}}{+19.68} & \gainmaincell{\underline{19.27}}{+11.70} & \gainmaincell{3.93}{+1.85} & \gainmaincell{\textbf{5.39}}{+3.51} & \gainmaincell{\textbf{14.65}}{+4.38} & \gainmaincell{\textbf{28.27}}{+8.48} & \gainmaincell{\textbf{36.79}}{+11.78} & \gainrowcell{\textbf{93.9}} \\
\bottomrule
\end{tabularx}
\end{adjustbox}
\end{table*}
\begin{table}[t]
\caption{\textbf{LanguageBind Prototype Classification.} Macro Top-1 (\%) on five relations. Boldface and underlining mark the best and second-best values among all methods except Full A--B. $^\dagger$ marks methods that use target-pair identities; blue parentheses show HSA gains over frozen cosine.}
\label{tab:languagebind_classification}
\centering
\scriptsize
\setlength{\tabcolsep}{2.0pt}
\renewcommand{\arraystretch}{1.10}
\begin{adjustbox}{Clip=0pt 0pt 0pt 0pt}
\begin{tabularx}{\linewidth}{@{}l*{5}{Y}@{}}
\toprule
\tworowhead{Method} &
\multicolumn{5}{c}{Prototype--Query Relations: Macro Top-1 (\%) $\uparrow$} \\
\cmidrule(lr){2-6}
&
\shortstack{Im--D\\NYU} &
\shortstack{Vi--Th\\TR} &
\shortstack{Vi--D\\TR} &
\shortstack{Vi--Au\\MSR} &
\shortstack{Th--D\\TR} \\
\midrule
\rowcolor{gray!12}
Full A--B$^\dagger$ & 61.82 & 51.85 & 65.35 & 24.39 & 58.06 \\
\midrule
Frozen Cosine & 51.96 & 74.03 & 43.68 & 13.77 & 37.80 \\
Hub-Relative~\cite{relrep} & 43.00 & 68.76 & 35.02 & 7.32 & 31.35 \\
Bi. Ridge~\cite{hoerl1970ridge} & 60.24 & \underline{82.58} & \underline{67.10} & 22.12 & 65.02 \\
Bi. Procrustes~\cite{schonemann1966} & \underline{60.85} & 79.19 & 61.69 & 20.44 & 62.38 \\
\midrule
ReAlign~\cite{realign} & 46.63 & 79.89 & 37.96 & 13.02 & 35.85 \\
ERM~\cite{domainbed} & 55.69 & 77.52 & 66.76 & 22.92 & 65.38 \\
IRM~\cite{irm} & 53.70 & 75.99 & 61.24 & \underline{24.95} & 64.94 \\
VREx~\cite{vrex} & 53.58 & 75.39 & 64.60 & 22.55 & 65.89 \\
DANN~\cite{dann} & 50.58 & 81.27 & 63.06 & 23.12 & \underline{68.17} \\
CORAL~\cite{deepcoral} & 52.96 & 76.39 & 64.46 & 22.92 & 65.32 \\
ASIF$^\dagger$~\cite{asif} & 51.20 & 80.17 & 55.60 & 17.01 & 52.12 \\
Paired-OP$^\dagger$~\cite{schonemann1966} & 60.00 & 81.97 & 60.00 & 20.56 & 55.76 \\
\midrule
\rowcolor{hsarow}
\gainrowcell{HSA (Ours)} & \gainmaincell{\textbf{62.89}}{+10.92} & \gainmaincell{\textbf{84.37}}{+10.33} & \gainmaincell{\textbf{69.99}}{+26.32} & \gainmaincell{\textbf{26.37}}{+12.60} & \gainmaincell{\textbf{69.01}}{+31.21} \\
\bottomrule
\end{tabularx}
\end{adjustbox}
\end{table}
\noindent\textbf{LanguageBind.}
HSA reaches 36.79\% mean Recall@10, 11.78 points above frozen cosine and 4.26
points above Paired-OP, the strongest other method included in the ranking.
HSA improves all ten held-out relations and leads on seven among all methods except Full A--B. The gains
cover the nearly saturated image--video relation and the weaker video--audio
and image--thermal relations, spanning a broad range of native
retrieval strengths.
For LanguageBind prototype classification, HSA reaches 62.53\% mean macro Top-1 accuracy across the five relations in \hsaTableRef{tab:languagebind_classification}. This is 18.28 points above frozen cosine and 3.11 points above bidirectional ridge. The gains span image, video, audio, depth, and thermal observations from NYUv2, TartanRGBT, and MSR-VTT.

\noindent\textbf{Classification Summary and Parameter Reuse.} Across the eleven
relations in \hsaTableRef{tab:imagebind_classification} and
\hsaTableRef{tab:languagebind_classification}, HSA raises mean macro Top-1 from
29.01\% to 52.43\%. These comparisons fit the relation using classification source
features and evaluate a common prototype bank. We also test cross-task reuse by fixing
every HSA parameter estimated from the source hub edges and scoring class prototypes
as candidates. On eight relations with compatible coordinates and no overlap between
fitting samples and classification queries, reuse reaches 48.01\% macro Top-1. Frozen
cosine and separately fitted HSA reach 26.46\% and 48.79\%, respectively, on the same
relations. The reuse-minus-cosine gain is 21.56 points, with a dataset-cluster
interval of [15.16, 29.52]. Overlap and preprocessing boundaries are detailed in the
appendix.

\noindent\textbf{Target-Paired Reference.}
\mbox{\hyperref[tab:imagebind_main]{Tables~\ref*{tab:imagebind_main}}}
and~\hyperref[tab:languagebind_main]{\ref*{tab:languagebind_main}}
show that HSA reaches 31.15\% mean Recall@10, compared with 26.78\%
for ASIF and 26.24\% for Paired-OP, and outperforms each on
17 of 19 relations.
Full A--B reaches $31.49\pm0.42$\% overall, with HSA achieving higher
Recall@10 on 10 of 19 relations. The comparison varies across backbones:
HSA reaches 24.90\% on ImageBind, compared with $22.95\pm0.36$\%
for Full A--B, and 36.79\% on LanguageBind, compared with
$39.18\pm0.50$\% for Full A--B. The $\pm$ values denote sample
standard deviations across the three runs. The overall
HSA-minus-Full A--B difference is $-0.34$ points, with a descriptive
cluster interval of $[-5.43,5.57]$. This comparison has no predefined
equivalence or non-inferiority margin.

\noindent\textbf{Across-Rank Gains.}
In \hyperref[tab:imagebind_main]{Tables~\ref*{tab:imagebind_main}}
and~\hyperref[tab:languagebind_main]{\ref*{tab:languagebind_main}}, the HSA gain over
frozen cosine increases from Recall@1 to Recall@10: from 8.73 to 14.10 points on
ImageBind and from 4.38 to 11.78 points on LanguageBind. These increases show that HSA
recovers correct matches beyond rank 1.

\noindent\textbf{Directional Consistency.}
HSA improves 37 of 38 directed relations (details in the appendix).
Mean Recall@10 rises from 19.54\% to 32.88\% for $A\!\to\!B$
and from 17.01\% to 29.43\% for $B\!\to\!A$.
Forward and reverse gains are 15.43 and 12.77 points on ImageBind
and 11.45 and 12.11 points on LanguageBind.
Improvements thus extend across both query directions on both
backbones.
\par

\hsaTakeaway{Across both backbones, HSA improves retrieval across
relations and query directions, and prototype classification
across relations. These gains show that task performance can improve through a hub-based readout while the underlying representations remain frozen.}

\Needspace*{4\baselineskip}
\subsection{\textcolor{black}{Does the Signal Come from the Two Hub Edges? (Q2)}}
\label{sec:exp_interface}
\label{sec:exp_analysis}

{\color{black}
After Q1 establishes the effect, Q2 traces its information source by separating valid within-edge correspondence from the convenience of shared source rows. We test within-edge correspondence, disjoint-source estimation, and source-sample scaling in that order.
\par}

\begin{figure}[t]
\centering
\includegraphics[width=0.97\linewidth]{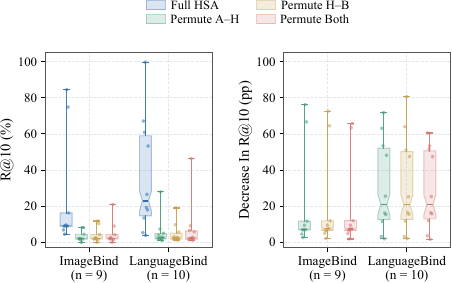}
\caption{\textbf{Hub-Edge Correspondence in Complete HSA.} Left: Recall@10 under intact correspondence and three permutation controls. Right: paired decreases in percentage points. Each dot represents a relation; permutation results average five refits. Boxes span the interquartile range, schematic notches mark medians, and whiskers span observed extrema.}
\label{fig:hsa_hub_correspondence}
\end{figure}

\noindent\textbf{Within-Edge Source Correspondence.}
We test the sample correspondence along both observed hub edges by permuting the hub-side rows of the $A$--$H$ edge, the $H$--$B$ edge, or both edges independently. Target-side row order, embeddings, marginals, estimator, hyperparameters, and scoring formula remain fixed. Every condition refits the relation directions, selected ranks, source gate, and mismatch scales under the same source-only rules, and evaluates the complete HSA score on the full gallery. Each intervention averages five seeds; the intact fit reproduces the main HSA retrieval results.
As shown in \hsaFigRef{fig:hsa_hub_correspondence}, permuting the $A$--$H$ edge, the $H$--$B$ edge, or both edges reduces mean Recall@10 from 31.15\% to 4.64\%, 4.58\%, and 6.56\%, respectively. The corresponding reductions are 26.52, 26.58, and 24.59 points, and all nineteen relation-level differences are positive under each intervention. Their dataset-cluster intervals are [9.83, 40.64], [10.03, 41.01], and [9.22, 36.53] points; all three Holm-adjusted $p$ values equal 0.00293. These interventions identify valid sample correspondence within both hub edges as a key source of HSA's recovery of target-modality associations. Shuffling either edge substantially reduces retrieval performance even after the complete readout is re-estimated.

\noindent\textbf{Disjoint Source Instances.}
HSA constructs its readout from two hub-edge moment systems, which predicts
that its principal retrieval gain should persist when the edges are estimated
from disjoint source instances. To test this prediction, we divide source rows
into equal halves for each of ten split seeds and estimate the two edges either
from the same half or from opposite halves. Per-edge sample budgets, source-only
calibration, gate construction, scoring, and evaluation remain matched. While removing cross-edge row sharing, the disjoint condition preserves valid correspondence within each edge.
\hsaTableRef{tab:hsa_disjoint_summary} shows that, across the nineteen relations, the sample-size-matched shared-row condition
gains 11.94 points over frozen cosine, of which strictly disjoint estimation
retains 10.77 points, or 90.2\%; every relation remains above frozen cosine.
The 95\% dataset-cluster interval for the mean disjoint gain is [3.83, 16.57].
Under equal per-edge budgets, shared-row estimation is 1.18 points above
strictly disjoint estimation on average, with a 95\% confidence interval of
[0.53, 2.30]. Thus, HSA retains most of its retrieval gain when the two moment systems are estimated from mutually disjoint source-instance sets. Its principal signal comes from valid correspondence within each hub edge and the resulting composition of the two moment systems.

\begin{table}[t]
\caption{\textbf{Retrieval with Shared and Disjoint Source Instances.} Mean bidirectional Recall@10 (\%) under matched per-edge budgets; parentheses show gains over frozen cosine.}
\label{tab:hsa_disjoint_summary}
\centering
\footnotesize
\setlength{\tabcolsep}{3.6pt}
\renewcommand{\arraystretch}{1.13}
\begin{tabularx}{\linewidth}{
  @{}
  >{\columncolor{white}[0pt][\tabcolsep]}l
  *{2}{Y}
  >{\columncolor{white}[\tabcolsep][0pt]}Y
  @{}
}
\toprule
Method & ImageBind & LanguageBind & Overall \\
\midrule
Frozen Cosine
& 10.80 & 25.00 & 18.27 \\
\midrule
\rowcolor{hsarow}
HSA, Shared Rows
& \gaincell{\textbf{24.22}}{+13.42}
& \gaincell{\textbf{35.62}}{+10.61}
& \gaincell{\textbf{30.22}}{+11.94} \\
\rowcolor{hsarow}
HSA, Disjoint Sources
& \gaincell{22.63}{+11.83}
& \gaincell{34.82}{+9.81}
& \gaincell{29.04}{+10.77} \\
\bottomrule
\end{tabularx}
\end{table}

\begin{figure}[t]
\centering
\includegraphics[width=\linewidth]{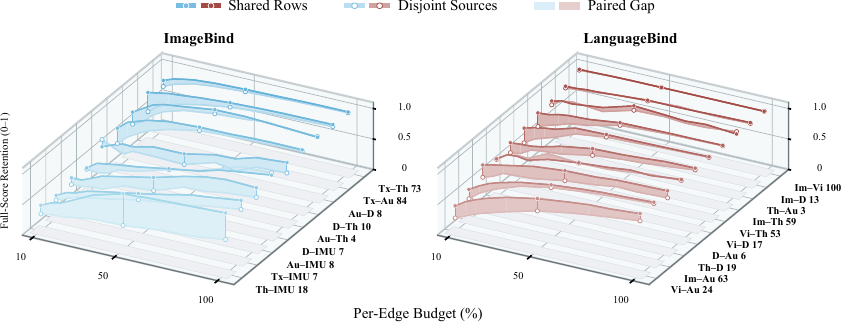}
\caption{\textbf{Source-Sample Scaling under Shared and Disjoint Estimation.} Bidirectional Recall@10 for each relation across ten nested per-edge budgets, normalized by the shared-row result at the full budget. Filled and hollow markers denote shared-row and disjoint-source estimation, respectively; shaded ribbons show their paired gaps. Numbers beside the relation labels report the shared-row Recall@10 (\%) at the full budget.}
\label{fig:hsa_sample_scaling}
\end{figure}

\noindent\textbf{Source-Sample Scaling.}
To determine how the remaining shared--disjoint difference changes with source
sample size, we estimate both conditions at ten nested per-edge budgets,
$\{10,15,20,25,35,50,65,75,85,100\}\%$, over split seeds 42--46.
For each relation, \hsaFigRef{fig:hsa_sample_scaling} normalizes both trajectories
by the shared-row result at the full budget; the endpoint label gives that
relation's absolute full-budget Recall@10.
As the per-edge source budget increases from 10\% to 100\%, the mean Recall@10
over all nineteen relations rises from 17.94\% to 29.05\% under disjoint
estimation and from 22.68\% to 30.21\% under shared-row estimation. The
difference between the two conditions contracts from 4.73 points at the 10\%
budget to 1.16 points at the full budget. At 50\%, the disjoint estimator
already reaches 26.59\%, or 91.5\% of its full-budget value. The gap narrows on both ImageBind and LanguageBind, showing a consistent sample-budget trend across the two backbones. Source-row sharing therefore primarily
improves moment estimation at low sample budgets; as source data increase, HSA
under disjoint estimation approaches the shared-row readout.

\hsaTakeaway{The main retrieval gain depends on valid correspondence
within both hub edges and persists with disjoint source instances.
The shrinking shared--disjoint gap with more data links the benefit
of source sharing to estimation precision.}

\Needspace{4\baselineskip}
\subsection{\textcolor{black}{Where Is Hub-Readable Knowledge Located? (Q3)}}
\label{sec:exp_spectral}

{\color{black}
Having identified the source information in Q2, Q3 follows it from relation construction to spectral localization and capacity: we test the required hub covariance, the informativeness of leading carriers, and the rate at which their gain accumulates.
\par}

\noindent\textbf{Hub-Readable Relation Construction.}
We compare full hub covariance construction, diagonal hub covariance construction, and trace-matched identity construction, fixing target standardization, candidate-resolution coordinates, gate construction, gallery, and retrieval protocol.
Full hub covariance construction preserves variances and cross-dimensional correlations, while diagonal hub covariance construction preserves variances alone.
Trace-matched identity construction weights dimensions equally and matches the trace of the full regularized inverse covariance matrix.
\mbox{\hsaTableRef{tab:relation_construction}} reports mean Recall@10 of 31.15\%, 27.02\%, and 26.95\%, respectively.
Full hub covariance construction improves 18 of 19 relations against each alternative, with mean gains of 4.13 and 4.21 points.
The respective 95\% intervals are [1.92, 5.76] and [2.50, 5.42], with Holm-adjusted $p=0.00391$ for both comparisons.
Preserving cross-dimensional hub correlations thus improves relation recovery and retrieval.

\begin{table}[t]
\caption{\textbf{Construction of the Hub-Readable Relation.} Mean bidirectional Recall@10 (\%); parentheses show changes from full HSA.}
\label{tab:relation_construction}
\centering
\footnotesize
\setlength{\tabcolsep}{3.5pt}
\renewcommand{\arraystretch}{1.13}
\begin{tabularx}{\linewidth}{
  @{}
  >{\columncolor{white}[0pt][\tabcolsep]}l
  *{2}{Y}
  >{\columncolor{white}[\tabcolsep][0pt]}Y
  @{}
}
\toprule
Construction & ImageBind & LanguageBind & Overall \\
\midrule
Diagonal Hub Covariance
& \lossinline{20.88}{-4.02}
& \lossinline{32.55}{-4.23}
& \lossinline{27.02}{-4.13} \\
Trace-Matched Identity
& \lossinline{20.76}{-4.13}
& \lossinline{32.51}{-4.27}
& \lossinline{26.95}{-4.21} \\
\midrule
\rowcolor{hsarow}
HSA (Ours)
& \textbf{24.90}
& \textbf{36.79}
& \textbf{31.15} \\
\bottomrule
\end{tabularx}
\end{table}

\noindent\textbf{Spectral Carrier Localization.}
We compare leading, subsequent, and random paired directions within the
complete HSA score. Carrier count, candidate-resolution coordinates,
the source gate, and the resolution scale remain fixed at the intact fit.
All direction sets receive the same element-wise leading reliability
spectrum, with carrier scales estimated separately using the same
source-only mismatch rule. The subsequent condition uses the next
singular directions, while the random condition averages five paired
draws from the complement of the leading block. Leading carriers
reproduce full HSA. \mbox{\hsaTableRef{tab:hsa_carrier_localization}}
reports 31.15\% mean Recall@10 for leading carriers, compared with
14.72\% for subsequent directions and 11.77\% for random directions.
Their mean advantages are 16.44 points [6.31, 25.41] and
19.38 points [7.06, 29.21], respectively. Both comparisons have
Holm-adjusted dataset-group $p=0.00195$.
\mbox{\hsaFigRef{fig:hsa_carrier_localization}} shows that leading
carriers win all nineteen relations. This result localizes usable
relation information to the leading paired directions under matched
carrier counts and assigned reliability spectra.

\begin{figure}[t]
\centering
\includegraphics[width=\linewidth]{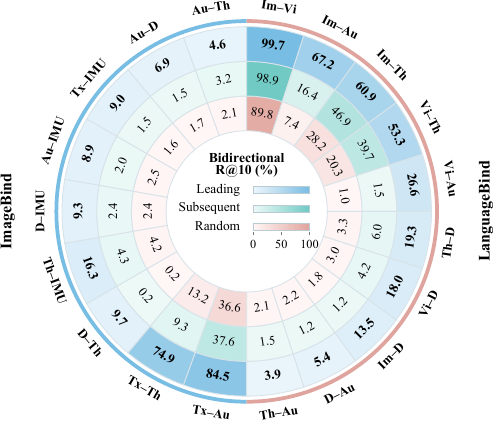}
\caption{\textbf{Spectral Carrier Localization in Complete HSA.} From outer to inner, the three heatmap rings report bidirectional Recall@10 (\%) for leading, subsequent, and random paired directions. Carrier counts and assigned reliability spectra are matched within each relation. Results for random directions are averaged over five seeds.}
\label{fig:hsa_carrier_localization}
\end{figure}

\begin{table}[t]
\caption{\textbf{Spectral Carrier Localization in the Complete Score.} All conditions receive the same leading reliability spectrum and use the same source-only rule for direction-specific carrier calibration. Values are mean bidirectional Recall@10 (\%); parentheses show changes from the leading condition.}
\label{tab:hsa_carrier_localization}
\centering\footnotesize
\setlength{\tabcolsep}{3pt}
\renewcommand{\arraystretch}{1.13}
\begin{tabularx}{\linewidth}{
  @{}
  >{\columncolor{white}[0pt][\tabcolsep]}l
  *{2}{Y}
  >{\columncolor{white}[\tabcolsep][0pt]}Y
  @{}
}
\toprule
Carrier Block & ImageBind & LanguageBind & Overall \\
\midrule
Subsequent & \lossinline{6.89}{-18.01} & \lossinline{21.76}{-15.03} & \lossinline{14.72}{-16.44} \\
Random & \lossinline{7.16}{-17.74} & \lossinline{15.92}{-20.86} & \lossinline{11.77}{-19.38} \\
\midrule\rowcolor{hsarow}
Leading (HSA) & \textbf{24.90} & \textbf{36.79} & \textbf{31.15} \\
\bottomrule\end{tabularx}
\end{table}

\noindent\textbf{Relation-Wise Carrier Capacity.}
To quantify carrier requirements, we retain nested leading prefixes
from 0\% to 100\% in 5\% steps, fixing the HSA fit,
candidate-resolution branch, gate, and gallery. Each prefix receives
source-only carrier-scale calibration; five equal-count random subsets
(seeds 42--46) provide capacity-matched controls. Let $R(c)$ denote
bidirectional Recall@10 at retained fraction $c$, and $R(0)$ the result
with zero carriers and all remaining score components retained.
For positive full gain $R(1)-R(0)$, C90 is the smallest evaluated fraction
at which the score reaches $R(0)+0.9[R(1)-R(0)]$ and remains at or above
this threshold at all subsequent checkpoints. Fractions are relative
to the full carrier set selected for each relation.
\mbox{\hsaFigRef{fig:hsa_relationwise_carrier_capacity}} shows median
C90 values of 10\% for ImageBind and 40\% for LanguageBind.
Eight of nine ImageBind relations and four of the nine defined
LanguageBind relations reach C90 by 25\% capacity. For LanguageBind
Im--D, Recall@10 falls from 14.53\% at zero carriers to 13.53\%
at full capacity, leaving C90 undefined. ImageBind typically concentrates
gains in shorter leading prefixes, whereas LanguageBind requires
broader carrier coverage.

\begin{figure}[t]
\centering
\includegraphics[width=\linewidth]{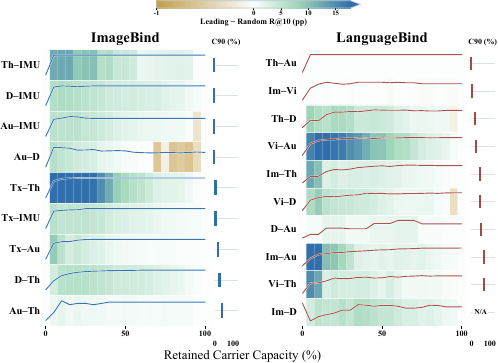}
\caption{\textbf{Relation-Wise Carrier Capacity in HSA.} Curves show retrieval performance as progressively more leading carriers are retained. Background colors indicate Recall@10 differences from equal-count random carrier subsets, in percentage points. C90 denotes the minimum retained capacity that sustains at least 90\% of the full-carrier gain over the zero-carrier baseline; N/A indicates a nonpositive full-carrier gain.}
\label{fig:hsa_relationwise_carrier_capacity}
\end{figure}

\begin{figure}[t]
\centering
\includegraphics[width=0.98\linewidth]{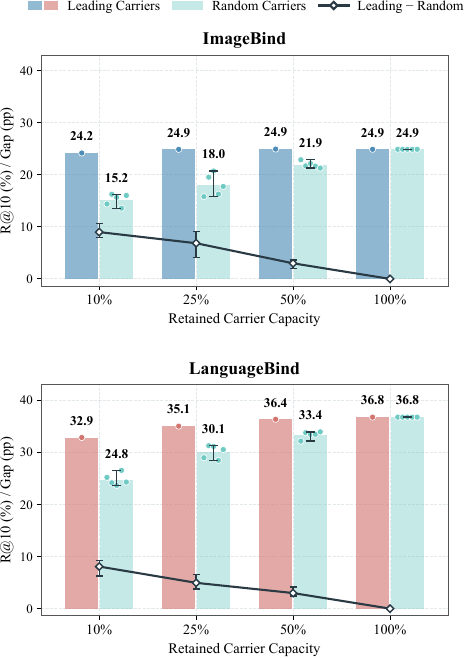}
\vspace{-0.3cm}
\caption{\textbf{Backbone-Level Carrier Accumulation in HSA.} Bars show mean bidirectional Recall@10 for leading and random carriers; circles mark the leading result and five seed-specific random means. Diamonds show leading-minus-random gaps in percentage points; error bars span the seed range.}
\label{fig:hsa_backbone_carrier_accumulation}
\end{figure}

\begin{table}[t]
\caption{\textbf{Carrier-Prefix Accumulation.} Mean bidirectional Recall@10 (\%) and median C90; score parentheses show changes from full HSA.}
\label{tab:hsa_carrier_capacity_summary}
\centering
\footnotesize
\setlength{\tabcolsep}{4.0pt}
\renewcommand{\arraystretch}{1.13}
\begin{tabularx}{\linewidth}{
  @{}
  >{\columncolor{white}[0pt][\tabcolsep]}l
  *{2}{Y}
  >{\columncolor{white}[\tabcolsep][0pt]}Y
  @{}
}
\toprule
Retained Carriers & ImageBind & LanguageBind & Overall \\
\midrule
Leading 25\%
& \lossinline{24.89}{-0.01}
& \lossinline{35.06}{-1.73}
& \lossinline{30.24}{-0.91} \\
Leading 50\%
& \gaininline{\textbf{24.94}}{+0.05}
& \lossinline{36.38}{-0.41}
& \lossinline{30.96}{-0.19} \\
Random 50\%
& \lossinline{21.95}{-2.95}
& \lossinline{33.38}{-3.40}
& \lossinline{27.97}{-3.19} \\
\midrule
\rowcolor{hsarow}
Full HSA (Ours)
& 24.90
& \textbf{36.79}
& \textbf{31.15} \\
\midrule
Median C90
& 10\% & 40\% & 20\% \\
\bottomrule
\end{tabularx}
\end{table}

\noindent\textbf{Backbone-Level Carrier Accumulation.}
\mbox{\hsaFigRef{fig:hsa_backbone_carrier_accumulation}} compares
leading prefixes with the mean of five equal-count random subsets
at 10\%, 25\%, 50\%, and 100\% retained capacity.
Under the complete HSA score,
\mbox{\hsaTableRef{tab:hsa_carrier_capacity_summary}} reports
mean bidirectional Recall@10 of 30.24\% and 30.96\% for the
leading 25\% and 50\% prefixes, respectively, compared with
31.15\% for full HSA. Increasing capacity from 25\% to 50\%
adds 0.72 percentage points, while including the remaining half
adds a further 0.19 points. Thus, the aggregate performance
approaches the full-capacity result with progressively smaller
additional gains over these intervals.
At each partial capacity shown, leading prefixes achieve higher
mean Recall@10 than random subsets for both backbones.
At 50\% capacity, random subsets reach 27.97\% overall,
approximately 3.0 percentage points below the leading prefix.
This comparison supports the value of spectral ordering when
the number of retained carriers is constrained.
The gap narrows as capacity increases and vanishes at 100\%,
where both conditions use the same full carrier set.
The accumulation patterns differ between the backbones on their
respective relation sets. ImageBind is already close to its
full-capacity mean at 25\% capacity (24.89\% versus 24.90\%)
and changes little at the larger reported capacities.
LanguageBind shows a more gradual increase, from 35.06\% at
25\% capacity to 36.38\% at 50\% and 36.79\% at full capacity.
Together with the localization controls, these results support
spectral ordering as a useful criterion for retaining retrieval
performance with fewer carriers. The extent to which a short
leading prefix preserves performance depends on the evaluated
backbone and relation set.

\hsaTakeaway{Cross-dimensional hub covariance improves recovery,
and usable evidence concentrates in leading paired carriers.
Their advantage under matched reliability spectra shows that
spectral direction matters; differing accumulation rates reveal
that this concentration varies across backbones.}

\Needspace*{4\baselineskip}
\subsection{\textcolor{black}{How Does HSA Turn Knowledge into Rankings? (Q4)}}
\label{sec:exp_components}

{\color{black}
Q4 tests how carrier reliability and readout components convert
the localized relation into ranking gains.
\par}

\noindent\textbf{Reliability-Controlled Activation.}
To test how the channel reliabilities $c_j$ control HSA's use of the localized
knowledge, we fix directions, retained rank, candidate-resolution coordinates,
the gate, and both mismatch scales, then replace every retained reliability by
$c_j(\alpha)=\alpha c_j$ for
$\alpha\in\{0,0.05,\ldots,1\}$. Each displayed trajectory is centered at its
zero-dose Recall@10 and divided by its largest observed increase; all reported
statistics use the unnormalized values.
As shown in \hsaFigRef{fig:hsa_activation_dose}, mean Recall@10 rises from 19.15\% at $\alpha=0$ to 31.15\% at $\alpha=1$.
Eighteen relations improve at full reliability, and the median within-relation
Spearman correlation between $\alpha$ and Recall@10 is 0.913. As the reliability parameters are restored, carrier evidence progressively contributes to the score and improves retrieval on most relations.

\noindent\textbf{Component Responsibilities.}
We assess carrier evidence with full-gallery Recall@10 and each remaining
component with its function-specific metric. Reliability uses
paired-versus-mismatched discrimination (P-AUC); candidate resolution uses
Top-1 accuracy in the fixed carrier top-ten candidate pool (C@1).
The source gate uses relation-level benefit discrimination (G-AUC),
and orthogonalization uses absolute correlation between the carrier and
resolution scores (Red.). These additional metrics were specified before
measurement and analyzed exploratorily.
\mbox{\hsaTableRef{tab:hsa_component_contribution}} reports
backbone-specific Recall@10 ablations and relation-equal function-metric
changes; G-AUC uses all nineteen relations jointly. Removing the carrier
score reduces mean Recall@10 by 12.00 points, the largest drop.
Uniform channel reliability and removal of candidate resolution,
the source gate, and orthogonalization reduce it by 2.48, 3.43, 3.86,
and 0.99 points, respectively. The corresponding metrics show a 5.20-point improvement in paired-versus-mismatched discrimination AUC, a 4.43-point improvement in Top-1 accuracy within the fixed carrier top-ten candidate pool, and lower cross-branch correlation. The source gate discriminates whether candidate resolution
improves retrieval with an AUC of 0.714. Candidate-resolution accuracy is conditional on the carrier top-ten pool containing a
valid match. Complete metric definitions,
backbone breakdowns, counts, and intervals are provided in the appendix.

\begin{figure}[t]
\centering
\includegraphics[width=0.98\linewidth]{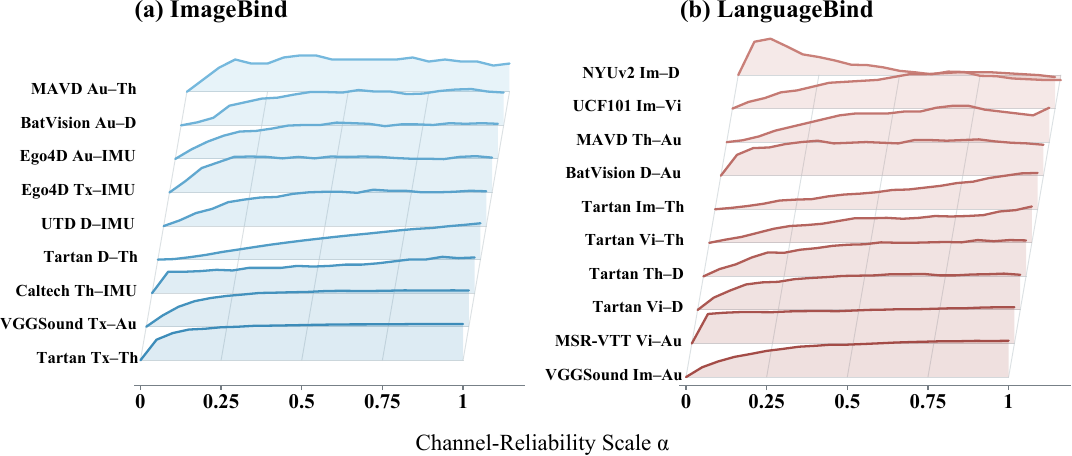}
\vspace{-0.3cm}
\caption{\textbf{Reliability-Controlled Spectral Activation in HSA.} Each curve shows the change in bidirectional Recall@10 relative to $\alpha=0$, normalized by the maximum observed increase for that relation across 21 channel-reliability scales in $[0,1]$. Panels (a) and (b) show results for ImageBind and LanguageBind, respectively.}
\label{fig:hsa_activation_dose}
\end{figure}

\begin{table}[t]
\caption{\textbf{Component Ablations.} Backbone-specific Recall@10 (\%)
and overall function-metric changes relative to full HSA.
Changes are in percentage points except Red., which uses absolute
correlation. Arrows indicate preferred directions; blue/red mark
improvements/degradations.}
\label{tab:hsa_component_contribution}
\centering\footnotesize
\setlength{\tabcolsep}{1.5pt}
\renewcommand{\arraystretch}{1.13}
\renewcommand{\tabularxcolumn}[1]{m{#1}}
\begin{adjustbox}{Clip=0pt 0pt 0pt 0pt}
\begin{tabularx}{\linewidth}{@{}>{\raggedright\arraybackslash}m{0.29\linewidth}*{2}{Y}M{0.13\linewidth}M{0.115\linewidth}@{}}
\toprule
Ablation & \shortstack{ImageBind\\R@10 ($\Delta$)} & \shortstack{LanguageBind\\R@10 ($\Delta$)} & Metric & $\Delta$ Metric \\
\midrule
w/o Carrier Score & 10.47\,{\scriptsize\color{hsaloss}(-14.43)} & 26.96\,{\scriptsize\color{hsaloss}(-9.82)} & R@10 $\uparrow$ & {\color{hsaloss}-12.00} \\
w/o Channel Reliability & 23.60\,{\scriptsize\color{hsaloss}(-1.30)} & 33.25\,{\scriptsize\color{hsaloss}(-3.53)} & P-AUC $\uparrow$ & {\color{hsaloss}-5.20} \\
w/o Candidate Resolution & \textbf{25.33}\,{\scriptsize\color{hsamain}(+0.44)} & 29.87\,{\scriptsize\color{hsaloss}(-6.92)} & C@1 $\uparrow$ & {\color{hsaloss}-4.43} \\
w/o Source Gate & 19.29\,{\scriptsize\color{hsaloss}(-5.61)} & 34.50\,{\scriptsize\color{hsaloss}(-2.29)} & G-AUC $\uparrow$ & {\color{hsaloss}-21.43} \\
w/o Orthogonalization & 25.15\,{\scriptsize\color{hsamain}(+0.26)} & 34.67\,{\scriptsize\color{hsaloss}(-2.12)} & Red. $\downarrow$ & {\color{hsaloss}+0.273} \\
\midrule
\rowcolor{hsarow} \textbf{Full HSA (Ours)} & 24.90 & \textbf{36.79} & -- & -- \\
\bottomrule
\end{tabularx}
\end{adjustbox}
\end{table}

\hsaTakeaway{Reliability weights carrier evidence, while the source
gate controls candidate resolution, making spectral activation
sensitive to both relation strength and residual geometry.}
\Needspace*{4\baselineskip}

\begin{table}[t]
\caption{\textbf{Relation Fitting and Retrieval Cost.} Median cost over 19 relations. Fit/Train is in seconds. Scoring is seconds per million directional pairs; Steps is the median final-fit count. $^\dagger$ marks methods that use target-pair identities.}
\label{tab:hsa_efficiency}
\centering
\scriptsize
\setlength{\tabcolsep}{1.8pt}
\renewcommand{\arraystretch}{1.02}
\begin{adjustbox}{Clip=0pt 0pt 0pt 0pt}
\begin{tabularx}{\linewidth}{@{}>{\raggedright\arraybackslash}p{0.31\linewidth}*{4}{Y}@{}}
\toprule
Method & Fit/Train & s/$10^6$ Pairs & Params. & Steps \\
\midrule
\rowcolor{gray!12}
Full A--B$^\dagger$ & 0.532 & 0.015 & 2.10--2.62M & 200 \\
\midrule
Frozen Cosine & -- & 0.015 & 0 & 0 \\
Hub-Relative~\cite{relrep} & 0.014 & 0.035 & 0 & 0 \\
Bi. Ridge~\cite{hoerl1970ridge} & 1.076 & 0.032 & 0 & 0 \\
Bi. Procrustes~\cite{schonemann1966} & 1.506 & 0.032 & 0 & 0 \\
ReAlign~\cite{realign} & 0.042 & 0.029 & 0 & 0 \\
\midrule
ERM~\cite{domainbed} & 5.376 & 0.019 & 3.15--3.94M & 400 \\
IRM~\cite{irm} & 9.024 & 0.020 & 3.15--3.94M & 400 \\
VREx~\cite{vrex} & 5.533 & 0.020 & 3.15--3.94M & 400 \\
DANN~\cite{dann} & 6.354 & 0.019 & 3.22--4.00M & 400 \\
CORAL~\cite{deepcoral} & 7.296 & 0.019 & 3.15--3.94M & 400 \\
\midrule
ASIF$^\dagger$~\cite{asif} & 0.015 & 0.730 & 0 & 0 \\
Paired-OP$^\dagger$~\cite{schonemann1966} & 1.732 & 0.041 & 0 & 0 \\
\midrule
\rowcolor{hsarow}
\textbf{HSA (Ours)} & 6.503 & 0.031 & 0 & 0 \\
\bottomrule
\end{tabularx}
\end{adjustbox}
\end{table}

\subsection{\textcolor{black}{What Does the Readout Cost? (Q5)}}
\label{sec:exp_efficiency}
We profile 14 methods across 19 retrieval and 11 classification relations, reporting fit/train time, scoring cost, trainable parameters, and median final-fit steps. For trainable baselines, training durations are selected using pilot validation splits drawn only from the training partition. Full A--B uses a single 200-step budget, selected globally by minimizing aggregate validation loss on these pilot splits and then fixed across relations and seeds. Fitting uses method-specific software environments, limiting direct timing comparisons; classification readouts share a common CPU configuration. Timing protocols and reproduction checks are detailed in the appendix.
\mbox{\hsaTableRef{tab:hsa_efficiency}} shows that HSA scores retrieval
pairs at a cost comparable to the other zero-target-pair analytic readouts. HSA fits a retrieval relation in a median of 6.503 seconds and scores one million
directional pairs in 0.031 seconds. Its
higher one-time fitting cost is amortized over subsequent queries. When encoder coordinates are compatible, retrieval and prototype classification can share the relation readout estimated from the source hub edges, without re-estimating its parameters separately for the two tasks. The classification timing table measures fitting under the separately specified classification protocol.
\mbox{\hsaTableRef{tab:hsa_classification_efficiency}} reports that, under the pooled workload across eleven classification relations, HSA has a CPU readout latency of 25.614 milliseconds per 1,000 queries and a throughput of 9.172 million query--prototype scores per second. Timing includes method-specific transforms, computation of the complete query--prototype score matrix, and Top-1 selection, quantifying the classification readout cost after fitting. Retrieval cost is normalized by directional pairs, while
classification latency includes scoring the class-prototype
bank and selecting a label. The two measurements therefore
describe the work required by their respective task readouts.
\par

\hsaTakeaway{HSA incurs fitting overhead once, then supports
low-cost retrieval and classification readouts, amortizing
relation recovery over repeated queries.}

\begin{table}[t]
\caption{\textbf{Prototype-Classification Cost.} Cost over eleven relations. Fit/Train is in seconds. Readout reports CPU milliseconds per 1,000 queries and millions of query--prototype scores per second. Steps is the median final-fit count; $^\dagger$ marks methods that use target-pair identities.}
\label{tab:hsa_classification_efficiency}
\centering
\scriptsize
\setlength{\tabcolsep}{1.0pt}
\renewcommand{\arraystretch}{1.02}
\begin{adjustbox}{Clip=0pt 0pt 0pt 0pt}
\begin{tabularx}{\linewidth}{@{}>{\raggedright\arraybackslash}p{0.27\linewidth}M{0.11\linewidth}M{0.14\linewidth}M{0.13\linewidth}M{0.18\linewidth}Y@{}}
\toprule
Method & Fit/Train & ms/$10^3$ q. & M Scores/s & Params. & Steps \\
\midrule
\rowcolor{gray!12}
Full A--B$^\dagger$ & 1.028 & 7.535 & 31.181 & 2.10--2.62M & 200 \\
\midrule
Frozen Cosine & -- & 1.139 & 206.181 & 0 & 0 \\
Hub-Relative~\cite{relrep} & 0.005 & 18.838 & 12.472 & 0 & 0 \\
Bi. Ridge~\cite{hoerl1970ridge} & 0.734 & 18.971 & 12.384 & 0 & 0 \\
Bi. Procrustes~\cite{schonemann1966} & 0.684 & 20.070 & 11.706 & 0 & 0 \\
ReAlign~\cite{realign} & 0.022 & 9.388 & 25.024 & 0 & 0 \\
\midrule
ERM~\cite{domainbed} & 4.958 & 10.894 & 21.566 & 3.15--3.94M & 400 \\
IRM~\cite{irm} & 8.638 & 10.813 & 21.727 & 3.15--3.94M & 400 \\
VREx~\cite{vrex} & 5.272 & 10.590 & 22.184 & 3.15--3.94M & 400 \\
DANN~\cite{dann} & 5.934 & 10.726 & 21.903 & 3.22--4.00M & 400 \\
CORAL~\cite{deepcoral} & 7.102 & 10.754 & 21.847 & 3.15--3.94M & 400 \\
\midrule
ASIF$^\dagger$~\cite{asif} & 0.013 & 151.247 & 1.553 & 0 & 0 \\
Paired-OP$^\dagger$~\cite{schonemann1966} & 0.698 & 21.868 & 10.743 & 0 & 0 \\
\midrule
\rowcolor{hsarow}
\textbf{HSA (Ours)} & 4.617 & 25.614 & 9.172 & 0 & 0 \\
\bottomrule
\end{tabularx}
\end{adjustbox}
\end{table}

\Needspace*{6\baselineskip}
\section{Conclusion}
This work characterizes the hub-readable component of latent multimodal
knowledge in frozen binding models and shows how it can support cross-modal
comparison. Our analysis establishes what can be recovered from the
second-order statistics of two trained hub connections and clarifies the
limits imposed by the hub. Building on this characterization, HSA expresses
the recovered relation through paired spectral carriers and constructs a
closed-form readout for retrieval and prototype classification, without
target-pair fitting or backbone updates. Experiments across two frozen
backbones, covering nineteen retrieval and eleven classification relations,
show improved average performance over native cosine scores.
Correspondence interventions and spectral controls further support the
roles of valid within-edge correspondence and leading paired directions
in the observed retrieval gains. Together, these findings connect the
recoverability of latent multimodal structure with its practical utility,
positioning HSA as both a relation readout and a diagnostic of the knowledge
accessible through a frozen hub.
Future work could extend relation recovery to nonlinear dependencies and
integrate complementary information from multiple hubs. 
Tracking the recovered
relations and their task utility during training could also help clarify how
latent multimodal knowledge emerges, evolves, and becomes usable.

\phantomsection
\label{sec:references}

\makeatletter
\begingroup
\edef\@currentlabel{\arabic{equation}}\label{hsa:main_last_equation}
\edef\@currentlabel{\arabic{figure}}\label{hsa:main_last_figure}
\edef\@currentlabel{\arabic{table}}\label{hsa:main_last_table}
\endgroup
\makeatother

\clearpage
\renewcommand{\hsaTakeaway}[1]{%
  \par\nobreak
  \ifdim\prevdepth>-1000pt \vskip-\prevdepth \fi
  \vskip8pt\relax\nointerlineskip
  \begingroup
  \setlength{\fboxsep}{4pt}%
  \vbox{\hbox{\colorbox{gray!12}{\parbox{\dimexpr\linewidth-2\fboxsep\relax}{\small\textcolor{black}{\textbf{Takeaway.} #1}}}}\kern0pt}%
  \endgroup
  \nobreak\vskip6pt\relax}

\appendices
\renewcommand*{\theHtable}{app.\arabic{table}}
\renewcommand{\thesubsection}{\thesection.\arabic{subsection}}
\renewcommand{\thesubsectiondis}{\thesection.\arabic{subsection}}

\definecolor{hsaappendixblue}{RGB}{78,112,136}
\definecolor{hsaappendixdots}{RGB}{164,178,188}
\newcommand{\hsaappendixdotfill}{%
  \leavevmode\leaders\hbox{\textcolor{hsaappendixdots}{.}\kern0.35em}\hfill\kern0pt}
\newlength{\hsaappendixlabelwidth}
\setlength{\hsaappendixlabelwidth}{17.5pt}
\newlength{\hsasubappendixlabelwidth}
\setlength{\hsasubappendixlabelwidth}{20pt}

\newcommand{\hsaappendixentry}[3]{%
  \par\noindent\normalsize
  \begin{tabularx}{\linewidth}{@{}>{\raggedright\arraybackslash}p{\hsaappendixlabelwidth}@{}>{\raggedright\arraybackslash}p{0.70\linewidth}@{\hspace{0.45em}}X@{\hspace{0.45em}}r@{}}
  \hyperref[#3]{\textcolor{hsaappendixblue}{\textbf{#1}}} &
  \hyperref[#3]{\textcolor{hsaappendixblue}{\textbf{#2}}} &
  \hsaappendixdotfill &
  \hyperref[#3]{\textcolor{black}{\pageref*{#3}}}
  \end{tabularx}\par
  \vspace{0.58em}}

\newcommand{\hsasubappendixentry}[3]{%
  \par\noindent\small
  \begin{tabularx}{\linewidth}{@{\hspace{\hsaappendixlabelwidth}}>{\raggedright\arraybackslash}p{\hsasubappendixlabelwidth}@{}>{\raggedright\arraybackslash}p{0.60\linewidth}@{\hspace{0.45em}}X@{\hspace{0.45em}}r@{}}
  \hyperref[#3]{\textcolor{hsaappendixblue!88!black}{#1}} &
  \hyperref[#3]{\textcolor{hsaappendixblue!88!black}{#2}} &
  \hsaappendixdotfill &
  \hyperref[#3]{\textcolor{black}{\pageref*{#3}}}
  \end{tabularx}\par
  \vspace{0.58em}}

\twocolumn[{%
\vspace*{1.2\baselineskip}
\pdfbookmark[1]{Appendix Contents}{appendix.contents}
\noindent{\fontsize{24}{28}\selectfont\bfseries Appendix\par}
\vspace{1.9em}
}]
\raggedbottom
\begingroup
\hsaappendixentry{A}{Source Model, Knowledge Definition, and Information Interface}{app:source_model}
\hsasubappendixentry{A.1}{Latent-State Standardization and Factor Model}{app:latent_state_factor_model}
\hsasubappendixentry{A.2}{Source-Induced Knowledge and Target Residuals}{app:knowledge_residuals}
\hsasubappendixentry{A.3}{Invariance and the Observable Information Interface}{app:information_interface}
\hsaappendixentry{B}{Hub Visibility and the Identifiable Relation}{app:hub_visibility}
\hsasubappendixentry{B.1}{Hub-Visible State and Relation}{app:hub_visible_derivation}
\hsasubappendixentry{B.2}{Exact Recovery and Non-Identifiability}{app:identifiability_boundary}
\hsasubappendixentry{B.3}{Hub-Capacity Bound}{app:hub_capacity_proof}
\hsaappendixentry{C}{Finite-Sample Estimation, Regularization, and Rank Roles}{app:finite_sample}
\hsasubappendixentry{C.1}{Edge-Wise Moment Estimation}{app:edge_moments}
\hsasubappendixentry{C.2}{Regularized Relation and Standardized Carrier}{app:regularized_relation}
\hsasubappendixentry{C.3}{Rank Selection and Hub-Capacity Safeguards}{app:rank_rules}
\hsaappendixentry{D}{Carrier Readout and Gaussian Evidence}{app:carrier_llr}
\hsasubappendixentry{D.1}{Optimal Carrier Directions}{app:optimal_carrier_directions}
\hsasubappendixentry{D.2}{Gaussian Evidence Expansion}{app:gaussian_evidence}
\hsaappendixentry{E}{Candidate Resolution, Calibration, and Score Formation}{app:fusion}
\hsasubappendixentry{E.1}{Orthogonal Candidate-Resolution Score}{app:candidate_resolution_score}
\hsasubappendixentry{E.2}{Source-Derived Reliability Gate}{app:reliability_gate}
\hsasubappendixentry{E.3}{Mismatch Calibration and Final Score}{app:score_calibration}
\hsaappendixentry{F}{Score Blocks, Rankings, Algorithm, and Complexity}{app:retrieval}
\hsasubappendixentry{F.1}{Blockwise Scoring and Bidirectional Rankings}{app:blockwise_scoring}
\hsasubappendixentry{F.2}{Execution Phases and Complexity}{app:execution_complexity}
\hsaappendixentry{G}{Classification Readout, Datasets, and Results}{app:classification}
\hsasubappendixentry{G.1}{Classification Protocol and Metrics}{app:classification_protocol}
\hsasubappendixentry{G.2}{Datasets, Splits, and Class Definitions}{app:classification_data}
\hsasubappendixentry{G.3}{Classification Eligibility}{app:classification_boundaries}
\hsasubappendixentry{G.4}{Retrieval and Prototype Classification Under a Single Fitted State}{app:frozen_fit_reuse}
\hsaappendixentry{H}{Retrieval Protocol and Reproducibility}{app:experimental_protocol}
\hsasubappendixentry{H.1}{Detailed Retrieval Protocol}{app:exp_protocol}
\hsasubappendixentry{H.2}{Efficiency Measurement and Outcome Checks}{app:efficiency_protocol}
\hsasubappendixentry{H.3}{Seed Variability of Trainable Baselines}{app:seed_variability}
\hsasubappendixentry{H.4}{Excluded Degenerate Unit}{app:excluded_unit}
\hsaappendixentry{I}{Additional Retrieval Results}{app:supp_experiments}
\hsasubappendixentry{I.1}{Direction-Wise Retrieval Consistency}{app:directional_consistency}
\hsasubappendixentry{I.2}{Recovery Relative to Target-Paired Training}{app:recovery_relative}
\hsasubappendixentry{I.3}{Robustness to Aggregation and Chance Level}{app:aggregate_robustness}
\hsaappendixentry{J}{Mechanism and Diagnostic Analyses}{app:diagnostics}
\hsasubappendixentry{J.1}{Complete Aggregate Intervention Statistics}{app:intervention_statistics}
\hsasubappendixentry{J.2}{Complete Component-Ablation Visualization}{app:component_ablation_details}
\hsasubappendixentry{J.3}{Matched Shared-Row and Disjoint-Source Control}{app:disjoint_source}
\hsasubappendixentry{J.4}{Correspondence Resolution under Group-Preserving Permutations}{app:granularity}
\hsaappendixentry{K}{Dataset-Cluster Statistics and Robustness}{app:cluster_statistics}
\hsasubappendixentry{K.1}{Estimands, Resampling, and Complete Families}{app:cluster_protocol}
\hsasubappendixentry{K.2}{Component Responsibilities: Definitions and Complete Results}{app:component_functions}
\hsasubappendixentry{K.3}{NYUv2 Evaluation Protocol and Cross-Dataset Robustness}{app:nyuv2_boundary}
\endgroup
\clearpage
\flushbottom

\setlength{\textfloatsep}{12pt plus 2pt minus 2pt}
\setlength{\dbltextfloatsep}{12pt plus 2pt minus 2pt}
\setlength{\floatsep}{8pt plus 2pt minus 2pt}
\setlength{\dblfloatsep}{8pt plus 2pt minus 2pt}

\renewcommand{\tabularxcolumn}[1]{m{#1}}

\section{Source Model, Knowledge Definition, and Information Interface}
\label{app:source_model}

\noindent\textbf{Recall.}
\hsaDefinitionRef{def:latent_knowledge} introduces
$\mathcal K_{AB}$ through source-conditioned means.
\hsaAssumptionRef{ass:hsa_source_model} specifies the second-order model used to read this relation through the hub. This appendix derives the factor form, separates source-induced knowledge from target-residual dependence, and states the observable information interface.

\subsection{Latent-State Standardization and Factor Model}
\label{app:latent_state_factor_model}

Let $\mathbf L^{(0)}$ have mean $\boldsymbol\mu_L$ and covariance $C_L$. After restricting $C_L$ to its support when necessary, define the standardized state
\begin{equation}
\mathbf L=C_L^{-1/2}(\mathbf L^{(0)}-\boldsymbol\mu_L),
\qquad
G_m=G_m^{(0)}C_L^{1/2}
\label{eq:app_latent_standardization}
\end{equation}
The standardized state satisfies $\mathbb E[\mathbf L]=0$ and
$\operatorname{Cov}(\mathbf L)=I_q$.

Under the conditional-mean model
\begin{equation}
\begin{aligned}
\mathbf A&=\boldsymbol\mu_A+G_A\mathbf L+\boldsymbol\varepsilon_A,\\
\mathbf H&=\boldsymbol\mu_H+G_H\mathbf L+\boldsymbol\varepsilon_H,\\
\mathbf B&=\boldsymbol\mu_B+G_B\mathbf L+\boldsymbol\varepsilon_B,
\end{aligned}
\label{eq:app_factor_model}
\end{equation}
the residual conditions used by HSA are
\begin{equation}
\begin{aligned}
\mathbb E[\boldsymbol\varepsilon_m\mid\mathbf L]&=0,\\
\operatorname{Cov}(\boldsymbol\varepsilon_A,\boldsymbol\varepsilon_H)&=0,
\qquad
\operatorname{Cov}(\boldsymbol\varepsilon_H,\boldsymbol\varepsilon_B)=0.
\end{aligned}
\label{eq:app_residual_assumptions}
\end{equation}

\subsection{Source-Induced Knowledge and Target Residuals}
\label{app:knowledge_residuals}

The conditional means are
$\mathbb E[\mathbf A\mid\mathbf L]=\boldsymbol\mu_A+G_A\mathbf L$ and
$\mathbb E[\mathbf B\mid\mathbf L]=\boldsymbol\mu_B+G_B\mathbf L$. Hence \hsaDefinitionRef{def:latent_knowledge} gives
\begin{align}
\mathcal K_{AB}
&=\operatorname{Cov}(G_A\mathbf L,G_B\mathbf L)
=G_A\operatorname{Cov}(\mathbf L)G_B^\top
=G_AG_B^\top.
\label{eq:app_factor_knowledge}
\end{align}
The observed target cross-covariance expands separately as
\begin{align}
\Sigma_{AB}
&=\mathbb E\!\left[
(G_A\mathbf L+\boldsymbol\varepsilon_A)
(G_B\mathbf L+\boldsymbol\varepsilon_B)^\top
\right]\nonumber\\
&=G_A\mathbb E[\mathbf L\mathbf L^\top]G_B^\top
+G_A\mathbb E[\mathbf L\boldsymbol\varepsilon_B^\top]\nonumber\\
&\quad+\mathbb E[\boldsymbol\varepsilon_A\mathbf L^\top]G_B^\top
+\mathbb E[\boldsymbol\varepsilon_A\boldsymbol\varepsilon_B^\top]\nonumber\\
&=\mathcal K_{AB}+\Omega_{AB},
\qquad
\Omega_{AB}:=\operatorname{Cov}
(\boldsymbol\varepsilon_A,\boldsymbol\varepsilon_B).
\label{eq:app_target_covariance_expansion}
\end{align}
Conditional residual centering eliminates the two latent--residual terms. \hsaAssumptionRef{ass:hsa_source_model} leaves $\Omega_{AB}$ unrestricted because the target edge is unavailable to HSA. The two observed-edge residual conditions then yield
$\Sigma_{AH}=G_AG_H^\top$ and $\Sigma_{HB}=G_HG_B^\top$.

\subsection{Invariance and the Observable Information Interface}
\label{app:information_interface}

For any orthogonal matrix $R\in\mathbb R^{q\times q}$, define
$\mathbf L'=R\mathbf L$ and $G_m'=G_mR^\top$. Then
\begin{equation}
G_A'G_B'^\top
=G_AR^\top(G_BR^\top)^\top
=G_AR^\top RG_B^\top
=G_AG_B^\top.
\label{eq:app_rotation_invariance}
\end{equation}
Hence $\mathcal K_{AB}$ depends on the relation expressed in the frozen target spaces and is invariant under orthogonal changes of basis in the latent state.

Fitting uses two paired hub-edge datasets,
$\mathcal D_{AH}=\{(\mathbf a_i,\mathbf h_i^{AH})\}_{i=1}^{n_{AH}}$
and
$\mathcal D_{HB}=\{(\mathbf h_j^{HB},\mathbf b_j)\}_{j=1}^{n_{HB}}$.
The datasets may use disjoint source instances and different sample sizes. With synchronized observations, one source row may instead contribute one pair to each dataset. Both sampling designs induce the same population information interface:
\begin{equation}
\mathfrak S_H=\{\boldsymbol\mu_A,\boldsymbol\mu_H,\boldsymbol\mu_B,
\Sigma_{AA},\Sigma_{AH},\Sigma_{HH},\Sigma_{HB},\Sigma_{BB}\}.
\label{eq:app_information_interface}
\end{equation}
Its empirical counterpart contains the two edge-moment systems and target marginals. Moment estimation does not use cross-dataset row identities. The interface excludes
$n^{-1}\sum_i\widetilde{\mathbf a}_i\widetilde{\mathbf b}_i^\top$
as well as any $A$--$B$ training loss or target-informed rank selection.
The population construction permits $n_{AH}$ and $n_{HB}$ to differ. The reported finite-sample protocol enforces equal per-edge counts so that the target marginals support the derangement calibration in \hsaAppendixRef{app:fusion}.

\section{Hub Visibility and the Identifiable Relation}
\label{app:hub_visibility}

\noindent\textbf{Recall.}
\hsaTheoremRef{thm:hub_readable} characterizes the latent multimodal knowledge identified by the two-hub-edge interface and states its exact-recovery boundary.

\subsection{Hub-Visible State and Relation}
\label{app:hub_visible_derivation}

\noindent\textit{Proof of \hsaTheoremRef{thm:hub_readable}.}
From~\eqref{eq:app_factor_model}--\eqref{eq:app_residual_assumptions},

\begin{equation}
\begin{aligned}
\Sigma_{AH}&=G_AG_H^\top,
&\Sigma_{HB}&=G_HG_B^\top,\\
\Sigma_{HH}&=G_HG_H^\top+\Psi_H.
\end{aligned}
\label{eq:app_observable_edges}
\end{equation}
Because $\operatorname{Cov}(\mathbf L,\mathbf H)=G_H^\top$, the best linear predictor of the centered state from the centered hub is
\begin{equation}
\mathbf L^H
=G_H^\top\Sigma_{HH}^{\dagger}
(\mathbf H-\boldsymbol\mu_H).
\label{eq:app_visible_state}
\end{equation}
Symmetry of $\Sigma_{HH}^{\dagger}$ and the identity
$\Sigma_{HH}^{\dagger}\Sigma_{HH}\Sigma_{HH}^{\dagger}
=\Sigma_{HH}^{\dagger}$ imply
\begin{align}
\operatorname{Cov}(\mathbf L^H)
&=G_H^\top\Sigma_{HH}^{\dagger}
\Sigma_{HH}\Sigma_{HH}^{\dagger}G_H\nonumber\\
&=G_H^\top\Sigma_{HH}^{\dagger}G_H
=:\mathcal P_H,
\label{eq:app_visible_state_covariance}\\
\operatorname{Cov}(\mathbf L,\mathbf L^H)
&=\operatorname{Cov}(\mathbf L,\mathbf H)
\Sigma_{HH}^{\dagger}G_H
=\mathcal P_H.
\label{eq:app_state_visible_cross_covariance}
\end{align}
For $\mathbf L^{\perp H}=\mathbf L-\mathbf L^H$,
\begin{align}
\operatorname{Cov}(\mathbf L^{\perp H})
&=I_q-\mathcal P_H-\mathcal P_H+\mathcal P_H
=I_q-\mathcal P_H.
\label{eq:app_unexplained_state_covariance}
\end{align}
Therefore $\mathcal P_H\succeq0$ and $I_q-\mathcal P_H\succeq0$, yielding
$0\preceq\mathcal P_H\preceq I_q$.

The best linear target predictions are
\begin{equation}
\begin{aligned}
\mathbf A^H
&=\Sigma_{AH}\Sigma_{HH}^{\dagger}
(\mathbf H-\boldsymbol\mu_H)=G_A\mathbf L^H,\\
\mathbf B^H
&=\Sigma_{BH}\Sigma_{HH}^{\dagger}
(\mathbf H-\boldsymbol\mu_H)=G_B\mathbf L^H.
\end{aligned}
\label{eq:app_target_predictions}
\end{equation}
Their cross-covariance has equivalent target-space and observable-moment forms:
\begin{align}
\operatorname{Cov}(\mathbf A^H,\mathbf B^H)
&=G_A\operatorname{Cov}(\mathbf L^H)G_B^\top
=G_A\mathcal P_HG_B^\top\nonumber\\
&=\Sigma_{AH}\Sigma_{HH}^{\dagger}
\Sigma_{HH}\Sigma_{HH}^{\dagger}\Sigma_{HB}\nonumber\\
&=\Sigma_{AH}\Sigma_{HH}^{\dagger}\Sigma_{HB}
=\mathcal K_{AB}^{H}.
\label{eq:app_hub_composition}
\end{align}

\subsection{Exact Recovery and Non-Identifiability}
\label{app:identifiability_boundary}

The population knowledge relation decomposes as
\begin{equation}
\mathcal K_{AB}
=\underbrace{G_A\mathcal P_HG_B^\top}_{\mathcal K_{AB}^{H}}
+\underbrace{G_A(I_q-\mathcal P_H)G_B^\top}_{\mathcal K_{AB}^{\mathrm{rem}}}.
\label{eq:app_knowledge_decomposition}
\end{equation}
The second term equals
$\operatorname{Cov}(G_A\mathbf L^{\perp H},G_B\mathbf L^{\perp H})$ and represents the shared-source relation beyond the hub's linear resolution. Consequently,
$\mathcal K_{AB}^{H}=\mathcal K_{AB}$ holds exactly when
$G_A(I_q-\mathcal P_H)G_B^\top=0$.

The formula
$\mathcal K_{AB}^{H}=\Sigma_{AH}\Sigma_{HH}^{\dagger}\Sigma_{HB}$
is a deterministic functional of $\mathfrak S_H$, which establishes its identifiability. To show the boundary for $\mathcal K_{AB}$, fix
$\alpha,\beta\in(-1,1)$ and consider
\begin{equation}
\Gamma(t)=
\begin{bmatrix}
1&\alpha&t\\
\alpha&1&\beta\\
t&\beta&1
\end{bmatrix},
\qquad
|t-\alpha\beta|
\leq\sqrt{(1-\alpha^2)(1-\beta^2)}.
\label{eq:app_nonidentifiable_family}
\end{equation}
The stated interval is exactly the positive-semidefinite condition obtained from the Schur complement because
$\det\Gamma(t)=(1-\alpha^2)(1-\beta^2)-(t-\alpha\beta)^2$.
For each admissible $t$, choose a square root $C_tC_t^\top=\Gamma(t)$, draw
$\mathbf L_t\sim\mathcal N(0,I_3)$, and set
$[A,H,B]^\top=C_t\mathbf L_t$ with zero residuals. Every resulting model satisfies \hsaAssumptionRef{ass:hsa_source_model} and has the same means and unit marginals. Each also has
$\Sigma_{AH}=\alpha$, $\Sigma_{HH}=1$, and $\Sigma_{HB}=\beta$, so all models expose the same $\mathfrak S_H$. Yet \hsaDefinitionRef{def:latent_knowledge} gives
$\mathcal K_{AB}=t$, which varies across the interval, while the identified component remains
$\mathcal K_{AB}^{H}=\alpha\beta$. Therefore the interface generally underidentifies the complete latent knowledge relation.
\unskip\nobreak\hfil\penalty50
\hskip1em\null\nobreak\hfil$\square$\par

When $\Psi_H=0$,
$\mathcal P_H=G_H^\top(G_HG_H^\top)^{\dagger}G_H$ is the orthogonal projector onto the row space of $G_H$. If $G_H$ has full column rank, this row space is $\mathbb R^q$. Consequently, $\mathcal P_H=I_q$ and
$\mathcal K_{AB}^{H}=\mathcal K_{AB}$. With $\Psi_H\succeq0$, the eigenvalues of $\mathcal P_H$ lie in $[0,1]$ and quantify soft linear visibility.

\subsection{Hub-Capacity Bound}
\label{app:hub_capacity_proof}

\noindent\textit{Proof of \hsaCorollaryRef{cor:hub_capacity}.}
The rank inequality for a matrix product gives
\begin{equation}
\operatorname{rank}(\mathcal K_{AB}^{H})
\leq\operatorname{rank}(\Sigma_{HH}^{\dagger})
=\operatorname{rank}(\Sigma_{HH}).
\label{eq:app_hub_rank_proof}
\end{equation}
If $\mathbf H$ has support on at most $N_H$ distinct vectors, its centered support spans at most $N_H-1$ dimensions. The range of $\Sigma_{HH}$ lies within this span. Therefore,
$\operatorname{rank}(\Sigma_{HH})\leq N_H-1$ and the result follows.
\unskip\nobreak\hfil\penalty50
\hskip1em\null\nobreak\hfil$\square$\par

\section{Finite-Sample Estimation, Regularization, and Rank Roles}
\label{app:finite_sample}

\noindent\textbf{Recall.}
\hsaEqRef{eq:hsa_spectral_carrier} estimates the hub-readable relation and decomposes its standardized form. This appendix specifies the edge statistics, regularization, null rule, and empirical capacity rule used by HSA.

\subsection{Edge-Wise Moment Estimation}
\label{app:edge_moments}

For one hub edge represented by matrices $X,H\in\mathbb R^{n\times d}$, let
\begin{equation}
\widehat\mu_X=\frac1n\sum_i\mathbf x_i,
\qquad
\widehat\mu_H=\frac1n\sum_i\mathbf h_i,
\label{eq:app_leg_means}
\end{equation}
and define
\begin{align}
\widehat\Sigma_{XX}
&=\frac1n\sum_i
(\mathbf x_i-\widehat\mu_X)
(\mathbf x_i-\widehat\mu_X)^\top,
\nonumber\\
\widehat\Sigma_{XH}
&=\frac1n\sum_i
(\mathbf x_i-\widehat\mu_X)
(\mathbf h_i-\widehat\mu_H)^\top,
\nonumber\\
\widehat\Sigma_{HH}
&=\frac1n\sum_i
(\mathbf h_i-\widehat\mu_H)
(\mathbf h_i-\widehat\mu_H)^\top.
\label{eq:app_leg_covariances}
\end{align}
The covariance marginals are symmetrized numerically. HSA estimates one such system for each observed hub edge and pools the two hub covariances as
\begin{equation}
\overline\Sigma_{HH}
=\tfrac12(\widehat\Sigma_{HH}^{A}
+\widehat\Sigma_{HH}^{B}).
\label{eq:app_pooled_hub}
\end{equation}
When both edges reuse identical hub rows, the two empirical hub covariances coincide. With disjoint source rows, they are estimated separately and may differ at finite sample size. \hsaAssumptionRef{ass:hsa_source_model} requires both to estimate compatible population hub statistics.

\subsection{Regularized Relation and Standardized Carrier}
\label{app:regularized_relation}

For any $d$-dimensional covariance $C$, the trace-scaled ridge is
\begin{equation}
\rho(C)=\lambda\frac{\operatorname{tr}(C)}{d},
\qquad \lambda=0.5.
\label{eq:app_ridge}
\end{equation}
HSA evaluates
\begin{equation}
\widehat{\mathcal K}_{AB}^{H}
=\widehat\Sigma_{AH}
(\overline\Sigma_{HH}+\rho(\overline\Sigma_{HH})I)^{-1}
\widehat\Sigma_{HB}
\label{eq:app_empirical_knowledge}
\end{equation}
using a linear solve. Its raw singular-value decomposition is
\begin{equation}
\widehat{\mathcal K}_{AB}^{H}
=U_R\operatorname{diag}(\eta_1,\ldots,\eta_r)V_R^\top.
\label{eq:app_raw_svd}
\end{equation}

Shared-row and disjoint-source estimation target the same regularized population composition under \hsaAssumptionRef{ass:hsa_source_model}. Their finite-sample difference arises from cross-edge diagonal terms when both edges use the same source instances. These terms can transmit target-residual dependence through the hub quadratic form; disjoint estimation removes their co-occurrence. The matched-budget control in \hsaAppendixRef{app:disjoint_source} quantifies this effect on retrieval.

For $X\in\{A,B\}$, eigendecompose
$\widehat\Sigma_{XX}+\rho_XI
=E_X\operatorname{diag}(\nu_{X,j})E_X^\top$ and set
\begin{equation}
W_X
=E_X\operatorname{diag}
\!\left(\max(\nu_{X,j},10^{-8})^{-1/2}\right)E_X^\top.
\label{eq:app_inverse_sqrt}
\end{equation}
Then
\begin{equation}
T_H=W_A\widehat{\mathcal K}_{AB}^{H}W_B^\top
=U_C\operatorname{diag}(\sigma_1,\ldots,\sigma_r)V_C^\top.
\label{eq:app_standardized_svd}
\end{equation}

\subsection{Rank Selection and Hub-Capacity Safeguards}
\label{app:rank_rules}

To obtain $k_R$, seed 142 generates one random permutation $\pi_R$ that reorders the left hub rows relative to $A$. Fixed points are allowed. The resulting null cross-covariance is
\begin{equation}
\widehat\Sigma_{AH}^{\pi_R}
=\frac1n\sum_i
(\mathbf a_i-\widehat\mu_A)
(\mathbf h_{\pi_R(i)}-\widehat\mu_H)^\top.
\label{eq:app_null_cross_covariance}
\end{equation}
The observed target whiteners are retained to form
\begin{equation}
T_H^{\pi_R}
=W_A\widehat\Sigma_{AH}^{\pi_R}
(\overline\Sigma_{HH}+\rho(\overline\Sigma_{HH})I)^{-1}
\widehat\Sigma_{HB}W_B^\top.
\label{eq:app_null_operator}
\end{equation}
With $\sigma_{1,\mathrm{null}}=\sigma_1(T_H^{\pi_R})$,
\begin{equation}
\begin{aligned}
k_R&=\max\!\left\{1,
\min\!\left(k_{\max},r,
\sum_{j=1}^{r}
\mathbf1[\sigma_j^2>\kappa\sigma_{1,\mathrm{null}}^2]
\right)\right\},\\
\kappa&=2,
\qquad k_{\max}=256.
\end{aligned}
\label{eq:app_relation_rank}
\end{equation}
This rank determines how many columns of $U_R$ and $V_R$ define the raw subspace removed during candidate resolution.

The carrier coordinate count follows a separate rule motivated by \hsaCorollaryRef{cor:hub_capacity}. Quantize hub rows by
$q(\mathbf h)=\operatorname{round}(10^3\mathbf h)$, count distinct states $N_H^A,N_H^B$, and define
\begin{equation}
\begin{aligned}
g(N)&=\max\{1,\min(k_{\max},N-1,d)\},\\
k_H&=\min\{g(N_H^A),g(N_H^B),r\}.
\end{aligned}
\label{eq:app_hub_capacity}
\end{equation}
For genuinely distinct hub states, \hsaCorollaryRef{cor:hub_capacity} gives the exact $N-1$ bound. \hsaEquationRef{eq:app_hub_capacity} counts rounded states, whereas the covariance and singular directions use unrounded embeddings. Its state count is therefore an empirical proxy for structural capacity. The outer maximum retains one coordinate when $N=1$, and the remaining minima enforce the limits $k_{\max}$, $d$, and $r$. Thus $k_R$ sets the null-resolved dimension of the raw projector, whereas $k_H$ is a separately motivated heuristic cap on carrier coordinates.

\section{Carrier Readout and Gaussian Evidence}
\label{app:carrier_llr}

\noindent\textbf{Recall.}
\hsaPropositionRef{prop:optimal_carrier} selects the paired directions that retain the greatest standardized hub-readable relation strength. \hsaEqRef{eq:hsa_knowledge_readout} defines the resulting sample coordinates. \hsaEqRef{eq:hsa_carrier_score} converts these coordinates into carrier evidence. This appendix proves the variational result and derives the likelihood-ratio expansion.

\subsection{Optimal Carrier Directions}
\label{app:optimal_carrier_directions}

\noindent\textit{Proof of \hsaPropositionRef{prop:optimal_carrier}.}
Let $P\in\mathbb R^{d_A\times k}$ and $Q\in\mathbb R^{d_B\times k}$ satisfy
$P^\top P=Q^\top Q=I_k$. The matrix $PQ^\top$ has exactly $k$ nonzero singular values, all equal to one. Applying the von Neumann trace inequality~\cite{golub2013matrix} gives
\begin{align}
\operatorname{tr}(P^\top T_HQ)
&=\langle T_H,PQ^\top\rangle_F
\leq\sum_{j=1}^{k}\sigma_j(T_H)
=\sum_{j=1}^{k}\sigma_j.
\label{eq:app_carrier_variational_bound}
\end{align}
Choosing $P=U_{C,k}$ and $Q=V_{C,k}$ attains equality and proves~\textup{(9)}. Mapping the maximizing directions back through the target whiteners gives
\begin{align}
(\Phi_A^{(k)})^\top\widehat{\mathcal K}_{AB}^{H}\Phi_B^{(k)}
&=U_{C,k}^\top
W_A\widehat{\mathcal K}_{AB}^{H}W_B^\top
V_{C,k}\nonumber\\
&=U_{C,k}^\top T_HV_{C,k}
=\operatorname{diag}(\sigma_1,\ldots,\sigma_k),
\label{eq:app_carrier_diagonalization}
\end{align}
which proves~\textup{(10)}.

HSA takes $k=k_H$ and defines
\begin{equation}
\begin{aligned}
\Phi_A&=W_AU_{C,k_H},
&\Phi_B&=W_BV_{C,k_H},\\
\mathbf z^A&=\Phi_A^\top(\mathbf a-\widehat\mu_A),
&\mathbf z^B&=\Phi_B^\top(\mathbf b-\widehat\mu_B).
\end{aligned}
\label{eq:app_carrier_coordinates}
\end{equation}
At population level,
\begin{equation}
\mathbf Z^A=\Phi_A^\top G_A\mathbf L
+\Phi_A^\top\boldsymbol\varepsilon_A,
\qquad
\mathbf Z^B=\Phi_B^\top G_B\mathbf L
+\Phi_B^\top\boldsymbol\varepsilon_B.
\label{eq:app_carrier_latent_readout}
\end{equation}
The Eckart--Young--Mirsky theorem~\cite{golub2013matrix} also gives
\begin{equation}
U_{C,k_H}\operatorname{diag}(\sigma_1,\ldots,\sigma_{k_H})
V_{C,k_H}^\top
\in\underset{\operatorname{rank}(R)\leq k_H}{\arg\min}
\lVert T_H-R\rVert_F.
\label{eq:app_best_rank_approximation}
\end{equation}
\unskip\nobreak\hfil\penalty50
\hskip1em\null\nobreak\hfil$\square$\par

\subsection{Gaussian Evidence Expansion}
\label{app:gaussian_evidence}

Set $c_j=\operatorname{clip}(\sigma_j,0,1-10^{-6})$. The working covariance matrices for matched and mismatched values of coordinate $j$ are
\begin{equation}
C_j^+=\begin{bmatrix}1&c_j\\c_j&1\end{bmatrix},
\qquad
C_j^-=I_2.
\label{eq:app_coordinate_covariances}
\end{equation}
The determinant and inverse required by the Gaussian density are
\begin{equation}
|C_j^+|=1-c_j^2,
\qquad
(C_j^+)^{-1}
=\frac1{1-c_j^2}
\begin{bmatrix}1&-c_j\\-c_j&1\end{bmatrix}.
\label{eq:app_covariance_inverse}
\end{equation}
For $\mathbf u=[x,y]^\top$, the shared normalizing constants of the two zero-mean Gaussian densities cancel, leaving the determinant term:
\begin{align}
\ell_j(x,y;c_j)
&=\log\frac{p(\mathbf u\mid C_j^+)}{p(\mathbf u\mid I_2)}\nonumber\\
&=-\frac12\log(1-c_j^2)
-\frac12\mathbf u^\top[(C_j^+)^{-1}-I_2]\mathbf u.
\label{eq:app_llr_matrix_form}
\end{align}
Furthermore,
\begin{equation}
(C_j^+)^{-1}-I_2
=\frac1{1-c_j^2}
\begin{bmatrix}c_j^2&-c_j\\-c_j&c_j^2\end{bmatrix},
\label{eq:app_inverse_difference}
\end{equation}
so expansion of the quadratic form yields
\begin{equation}
\ell_j(x,y;c_j)
=\frac{c_jxy}{1-c_j^2}
-\frac{c_j^2(x^2+y^2)}{2(1-c_j^2)}
-\frac12\log(1-c_j^2).
\label{eq:app_llr_expansion}
\end{equation}
Summing~\eqref{eq:app_llr_expansion} over $j=1,\ldots,k_H$ gives the carrier score in \hsaEqRef{eq:hsa_carrier_score}. The Gaussian working model provides a closed-form evidence score in the standardized readout coordinates, leaving the full embedding distribution unrestricted.
\section{Candidate Resolution, Calibration, and Score Formation}
\label{app:fusion}

\noindent\textbf{Recall.}
\hsaEqRef{eq:hsa_complementary_score} constructs candidate-resolution coordinates orthogonal to the leading raw relation subspace. \hsaEqRef{eq:complete_hsa} combines their score with carrier evidence after mismatch calibration.

\subsection{Orthogonal Candidate-Resolution Score}
\label{app:candidate_resolution_score}

From the raw decomposition~\eqref{eq:app_raw_svd}, define
\begin{equation}
\Pi_A^R=U_{R,k_R}U_{R,k_R}^\top,
\qquad
\Pi_B^R=V_{R,k_R}V_{R,k_R}^\top.
\label{eq:app_raw_projectors}
\end{equation}
The candidate-resolution coordinates and score are
\begin{align}
\mathbf r^A(\mathbf a)
&=\frac{(I-\Pi_A^R)(\mathbf a-\widehat\mu_A)}
{\max(\lVert(I-\Pi_A^R)(\mathbf a-\widehat\mu_A)\rVert_2,10^{-12})},
\nonumber\\
\mathbf r^B(\mathbf b)
&=\frac{(I-\Pi_B^R)(\mathbf b-\widehat\mu_B)}
{\max(\lVert(I-\Pi_B^R)(\mathbf b-\widehat\mu_B)\rVert_2,10^{-12})},
\nonumber\\
s_R(\mathbf a,\mathbf b)
&=\mathbf r^A(\mathbf a)^\top\mathbf r^B(\mathbf b).
\label{eq:app_complementary_score}
\end{align}

\subsection{Source-Derived Reliability Gate}
\label{app:reliability_gate}

For each observed hub edge, define the ridge-regularized covariance matrices as
\begin{align}
\widehat\Omega_A
&=\widehat\Sigma_{AA}
-\widehat\Sigma_{AH}
(\widehat\Sigma_{HH}^{A}+\rho_H^AI)^{-1}
\widehat\Sigma_{HA},
\nonumber\\
\widehat\Omega_B
&=\widehat\Sigma_{BB}
-\widehat\Sigma_{BH}
(\widehat\Sigma_{HH}^{B}+\rho_H^BI)^{-1}
\widehat\Sigma_{HB},
\label{eq:app_residual_covariances}
\end{align}
where
$\rho_H^A=\rho(\widehat\Sigma_{HH}^{A})$ and
$\rho_H^B=\rho(\widehat\Sigma_{HH}^{B})$.
Let $d_R=\lVert\widehat\Omega_A\rVert_F
\lVert\widehat\Omega_B\rVert_F$. The specified rule is
\begin{equation}
\alpha_R=
\begin{cases}
\operatorname{clip}\!\left(
\dfrac{\langle\widehat\Omega_A,\widehat\Omega_B\rangle_F}{d_R},-1,1
\right),&d_R>10^{-20},\\[6pt]
0,&d_R\leq10^{-20},
\end{cases}.
\label{eq:app_residual_affinity}
\end{equation}
For each of the 256 coordinate permutations generated with seed 42, let $P_\pi$ be the corresponding permutation matrix and compute
\begin{align}
\alpha_R^{\pi}
&=
\begin{cases}
\operatorname{clip}\!\left(
\dfrac{\langle\widehat\Omega_A,
P_\pi\widehat\Omega_BP_\pi^\top\rangle_F}{d_R},-1,1
\right),&d_R>10^{-20},\\[6pt]
0,&d_R\leq10^{-20},
\end{cases}
\nonumber\\
q_R&=\operatorname{Quantile}_{0.95}
\left(\{\alpha_R^{\pi}:\pi\in\Pi_G\}\right),
\nonumber\\
g_R&=\operatorname{clip}\!\left(
\frac{\alpha_R-q_R}{\max(1-q_R,10^{-12})},0,1\right).
\label{eq:app_visibility_gate}
\end{align}
Each coordinate permutation preserves the spectrum of the regularized covariance matrix while disrupting coordinate-wise agreement between the target spaces. The resolution-reliability gate is zero when the observed affinity does not exceed the 95th percentile of this null distribution. Otherwise, it rescales the excess over $q_R$ from $[q_R,1]$ to $[0,1]$. Candidate-resolution evidence therefore enters the ranking according to cross-modal agreement estimated entirely from the hub edges.

\subsection{Mismatch Calibration and Final Score}
\label{app:score_calibration}

The reported protocol supplies equal-size target-side training marginals, denoted by $\{\mathbf a_i\}_{i=1}^{n}$ and $\{\mathbf b_i\}_{i=1}^{n}$. Synchronized runs inherit these counts from common source rows, while the disjoint-source control uses matched splits. The common index supplies rows for mismatch construction, while the two hub-edge moment systems remain separately estimated.

For calibration, each seed in $\{42,43,44,45,46\}$ generates a derangement $\pi$, satisfying $\pi(i)\neq i$ for every row. For $Q\in\{C,R\}$,
\begin{equation}
\begin{aligned}
\tau_Q^{\pi}
&=\max\!\left\{
\operatorname{Std}_{i=1}^{n}
[s_Q(\mathbf a_i,\mathbf b_{\pi(i)})],10^{-6}
\right\},\\
\tau_Q&=\frac15\sum_{\pi}\tau_Q^{\pi}.
\end{aligned}
\label{eq:app_branch_scales}
\end{equation}
The final rule is
\begin{equation}
s_{\mathrm{HSA}}
=\frac{s_C}{\max(\tau_C,10^{-6})}
+g_R\frac{s_R}{\max(\tau_R,10^{-6})}.
\label{eq:app_complete_score}
\end{equation}
\hsaEquationRef{eq:app_complete_score} combines likelihood-ratio carrier evidence $s_C$ with candidate-resolution evidence $s_R$. The source-only gate $g_R$ weights the latter, and $\tau_C$ and $\tau_R$ calibrate the two score scales. This construction supplies the same fitted scoring rule for retrieval and prototype classification.
\section{Score Blocks, Rankings, Algorithm, and Complexity}
\label{app:retrieval}

\noindent\textbf{Recall.}
\hyperref[eq:hsa_score_matrices]{Eqs.~\hsaEqNumber{eq:hsa_score_matrices}}--\hyperref[eq:hsa_bidirectional_rankings]{\hsaEqNumber{eq:hsa_bidirectional_rankings}} turn the fitted readout into two retrieval directions. This appendix gives the block form used for efficient evaluation and summarizes the execution phases and computational complexity.

\subsection{Blockwise Scoring and Bidirectional Rankings}
\label{app:blockwise_scoring}

Let $Z_A\in\mathbb R^{N_A\times k_H}$ and
$Z_B\in\mathbb R^{N_B\times k_H}$ collect carrier coordinates. Define
\begin{equation}
\begin{aligned}
d_j&=\frac{c_j}{1-c_j^2},
&e_j&=\frac{c_j^2}{2(1-c_j^2)},\\
\beta&=-\frac12\sum_{j=1}^{k_H}\log(1-c_j^2).
\end{aligned}
\label{eq:app_block_weights}
\end{equation}
Let
$\mathbf q_A=(Z_A\odot Z_A)\mathbf e$ and
$\mathbf q_B=(Z_B\odot Z_B)\mathbf e$, where
$\mathbf e=[e_1,\ldots,e_{k_H}]^\top$. The full carrier block is
\begin{equation}
S_C
=Z_A\operatorname{diag}(\mathbf d)Z_B^\top
-\mathbf q_A\mathbf1_{N_B}^\top
-\mathbf1_{N_A}\mathbf q_B^\top
+\beta\mathbf1_{N_A}\mathbf1_{N_B}^\top.
\label{eq:app_carrier_score_block}
\end{equation}
If $R_A$ and $R_B$ collect candidate-resolution coordinates, then
\begin{equation}
\begin{aligned}
S_R&=R_AR_B^\top,\\
S^{A\rightarrow B}
&=\frac{S_C}{\tau_C}+g_R\frac{S_R}{\tau_R},\\
S^{B\rightarrow A}&=(S^{A\rightarrow B})^\top.
\end{aligned}
\label{eq:app_full_score_blocks}
\end{equation}
The rankings are
\begin{equation}
\begin{aligned}
\pi_i^{A\rightarrow B}
&=\operatorname{argsort}_{j}^{\downarrow}[S^{A\rightarrow B}]_{ij},\\
\pi_j^{B\rightarrow A}
&=\operatorname{argsort}_{i}^{\downarrow}[S^{A\rightarrow B}]_{ij}.
\end{aligned}
\label{eq:app_bidirectional_ranking}
\end{equation}

\subsection{Execution Phases and Complexity}
\label{app:execution_complexity}

Algorithm~1 contains three phases. During fitting, HSA estimates edge and marginal statistics and composes $\widehat{\mathcal K}_{AB}^{H}$. It then performs both decompositions, selects $k_R$ and $k_H$, and fixes the carrier directions, raw projectors, regularized covariance matrices, and resolution-reliability gate. During calibration, HSA evaluates five deranged target-marginal pairings and fixes $\tau_C$ and $\tau_R$. During task readout, it projects both task sets and evaluates their pair scores. Retrieval forms both score blocks and sorts both axes; prototype classification scores each query against the fixed class-prototype bank. Target test identities or labels never enter HSA fitting, calibration, or score construction. Retrieval labels are consulted only after scoring to determine hits, while classification labels are used to form prototypes and evaluate predictions under the protocol in \hsaAppendixRef{app:classification_protocol}.

For $n$ anchor rows and common dimension $d$, dense covariance estimation costs $O(nd^2)$ time and $O(d^2)$ memory. The dense solve, eigendecompositions, and singular-value decompositions cost $O(d^3)$ time. Query projection costs
$O((N_A+N_B)d(k_H+k_R))$. The carrier and candidate-resolution score blocks cost
$O(N_AN_Bk_H)$ and $O(N_AN_Bd)$, respectively. Query rows are evaluated in batches, so intermediate storage scales with batch size and avoids allocating the complete $N_AN_B$ score matrix.

\section{Classification Readout, Datasets, and Results}
\label{app:classification}

This section documents the prototype-classification protocol underlying Tables~II and~IV, including the readout direction, class definitions, splits, eligibility criteria, and fixed-state reuse.

\subsection{Classification Protocol and Metrics}
\label{app:classification_protocol}

\noindent\textbf{Prototype Direction.}
Each relation follows the modality order used in the main tables. The first modality is $A$ and supplies one prototype $\mathbf p_c^A$ for each class $c$; the second modality is $B$ and supplies the test queries. Classification therefore evaluates the $B\!\to\!A$ direction in \hsaEqRef{eq:hsa_prototype_classification}. For every relation except ImageBind text--audio on VGGSound, each prototype is formed by averaging the normalized labeled training features in modality $A$ and then normalizing the mean. The VGGSound ImageBind relation uses a fixed 80-template text-prompt bank to form its $A$-side class prototypes. Every compared method receives the same prototype bank and the same test queries within a relation.

\noindent\textbf{Information Boundary.}
The main classification comparisons use the classification source features to estimate the target-modality relation from the two hub-edge moment systems and construct the HSA readout, under the same zero-target-pair-identity boundary as retrieval. Both tasks use the same relation-estimation and readout method; Appendix~\ref{app:frozen_fit_reuse} further validates task readout with all parameters fixed. Class labels define the $A$-side prototype bank; target test outcomes do not enter carrier selection, rank selection, reliability gating, or mismatch calibration. Frozen cosine, hub-relative similarity, bidirectional ridge, bidirectional Procrustes, ReAlign, ERM, IRM, VREx, DANN, CORAL, and HSA use no $A$--$B$ pair identities. ASIF, Paired-OP, and Full A--B use labeled target-training pairs and retain the $^\dagger$ marker. ASIF and Paired-OP participate in the boldface and underlining; only Full A--B is excluded. ERM, IRM, VREx, DANN, CORAL, and Full A--B report the mean over seeds 40--42; analytic methods use their fixed registered run.

\noindent\textbf{Metric and Aggregation.}
Let $\mathcal I_c^{\mathrm{te}}$ be the test queries from class $c$. Macro Top-1 averages class-wise accuracy,
\begin{equation}
\operatorname{mT1}
=\frac{1}{|\mathcal C|}
\sum_{c\in\mathcal C}
\frac{1}{|\mathcal I_c^{\mathrm{te}}|}
\sum_{i\in\mathcal I_c^{\mathrm{te}}}
\mathbf 1[\widehat y_i=c].
\label{eq:supp_macro_top1}
\end{equation}
Macro Top-1 is primary because several datasets are class-imbalanced. The classification means weight the six ImageBind relations or the five LanguageBind relations equally. Their combined mean weights all eleven relations equally. Diagnostic results do not enter these means.

\subsection{Datasets, Splits, and Class Definitions}
\label{app:classification_data}

\begin{table*}[t]
\caption{\textbf{Classification Dataset Coverage.} Pair abbreviations follow the main tables. Counts give prototype-training and test queries. The two Ego4D counts follow the displayed relation order and reflect the evaluated modality intersection.}
\label{tab:supp_classification_datasets}
\centering
\footnotesize
\setlength{\tabcolsep}{3pt}
\renewcommand{\arraystretch}{1.24}
\begin{tabularx}{\textwidth}{@{}>{\raggedright\arraybackslash}X >{\raggedright\arraybackslash}m{0.15\linewidth} >{\raggedright\arraybackslash}m{0.20\linewidth} M{0.065\linewidth} M{0.105\linewidth} M{0.105\linewidth}@{}}
\toprule
\multirow{2}{*}{Dataset} & \multicolumn{2}{c}{Target Relations} & \multirow{2}{*}{Classes} & \multirow{2}{*}{Train} & \multirow{2}{*}{Test} \\
\cmidrule(lr){2-3}
& ImageBind & LanguageBind & & & \\
\midrule
VGGSound-309 & Tx--Au & -- & 309 & 9,258 & 15,421 \\
\addlinespace[3pt]
NYUv2 & -- & Im--D & 10 & 795 & 654 \\
\addlinespace[3pt]
TartanRGBT & Tx--Th & \shortstack[l]{Vi--Th, Vi--D\\Th--D} & 4 & 1,184 & 436 \\
\addlinespace[3pt]
Ego4D & \shortstack[l]{Tx--IMU\\Au--IMU} & -- & 17 & \shortstack{1,084\\875} & \shortstack{434\\391} \\
\addlinespace[3pt]
BatVision & Au--D & -- & 7 & 2,536 & 584 \\
\addlinespace[3pt]
UTD-MHAD & D--IMU & -- & 27 & 431 & 430 \\
\addlinespace[3pt]
MSR-VTT & -- & Vi--Au & 20 & 6,176 & 884 \\
\midrule
\multicolumn{6}{@{}l@{}}{Total: 11 Classification Relations (6 ImageBind, 5 LanguageBind).} \\
\bottomrule
\end{tabularx}
\end{table*}

\noindent\textbf{VGGSound-309.}
Classification uses all 309 classes in the official VGGSound vocabulary. Its 9,258 training clips come from one cached source shard; all classes are present, with 2--232 clips per class and a median of 27. The 15,421 test clips are the intersection of the registered official test list and the available media snapshot, covering all 309 classes; 25 registered clips are absent from that snapshot. ImageBind forms fixed text-prompt prototypes. Retrieval uses the registered 100-class, 1,500/1,000 split described in \hsaAppendixRef{app:exp_protocol}.

\noindent\textbf{NYUv2.}
NYUv2 uses the official 795/654 scene split and ten scene classes: bathroom, bedroom, bookstore, classroom, dining room, home office, kitchen, living room, office, and others. Prototypes are formed from training images, and test queries are depth observations.

\noindent\textbf{TartanRGBT.}
The four environment classes are indoor, offroad, outdoor, and urban. They are derived from the first component of each registered trajectory identifier. Training and test trajectories are disjoint. The 1,184 training samples contain 322, 458, 303, and 101 examples in these classes; the 436 test samples contain 252, 135, 21, and 28. The label \emph{park} has no registered test trajectory and is excluded by the locked class rule. The same identities and split are used for all four ImageBind and LanguageBind relations.

\noindent\textbf{Ego4D.}
Scenario metadata define seventeen fixed activity classes after requiring at least ten training and five test samples per relation and taking the class intersection across the text--inertial and audio--inertial relations: Bike mechanic; Carpenter; Cleaning/laundry; Cooking; Crafting/knitting/sewing/drawing/painting; Eating; Farmer; Gardening; Household management--caring for kids; Indoor navigation (walking); Playing board games; Playing games/video games; Playing with pets; Reading books; Watching TV; Working out at home; and jobs related to a construction or renovation company. Registered video-identity splits are applied before this fixed-class filtering.

\noindent\textbf{BatVision.}
The seven classes are the collection locations 2nd Floor Luxembourg, 3rd Floor Luxembourg, Attic, Outdoor Cobblestone Path, Salle Chevalier, Salle des Colonnes, and V119 Cake Corridors. The registered training-plus-validation split provides 2,536 labeled samples for constructing seven class prototypes; the test split provides 584 queries. These labels support location recognition and may reflect location-specific visual and acoustic structure.

\noindent\textbf{UTD-MHAD.}
The 27 official actions are: right-arm swipe left; right-arm swipe right; right-hand wave; two-hand front clap; right-arm throw; cross arms at the chest; basketball shoot; draw X; draw a clockwise circle; draw a counter-clockwise circle; draw a triangle; bowling; front boxing; baseball swing; tennis forehand swing; two-arm curl; tennis serve; two-hand push; knock on a door; catch an object; pick up and throw; jog in place; walk in place; sit to stand; stand to sit; forward lunge; and squat. Subjects 1, 3, 5, and 7 form the 431-sample prototype split, while subjects 2, 4, 6, and 8 form the 430-sample test split.

\noindent\textbf{MSR-VTT.}
The twenty official video categories are music, people, gaming, sports/actions, news/events/politics, education, TV shows, movie, animation, vehicles, how-to, travel, science/technology, animal, kids/family, documentary, food, cooking, beauty/fashion, and advertisement. All twenty categories occur in the registered 6,176/884 training/test split. Video prototypes classify audio queries.

\subsection{Classification Eligibility}
\label{app:classification_boundaries}

\noindent\textbf{MAVD Drive-Level Diagnostic.}
MAVD provides day-drive and night-drive labels, with one recording drive per class. These labels jointly encode time of day and drive identity. HSA reaches 100.00\% on ImageBind audio--thermal and 76.65\% on LanguageBind thermal--audio, compared with 50.00\% for frozen cosine on both relations. These drive-level diagnostics are reported separately from the eleven-relation classification aggregate.

\noindent\textbf{Caltech Label Availability.}
The Caltech Aerial RGBT cache provides timestamps and visual, thermal, and inertial features from one recording bag. Its annotations support paired retrieval but supply no semantic scene, terrain, or activity classes. Classification uses dataset-defined labels, so the thermal--inertial relation participates in retrieval evaluation only.

\subsection{Retrieval and Prototype Classification Under a Single Fitted State}

\label{app:frozen_fit_reuse}

The main classification tables estimate the HSA relation readout using the source features and splits specified for classification, with common prototypes and queries across methods. To test whether the same estimated relation can support both retrieval and prototype classification, we fix the complete HSA state estimated from the source hub edges: centers, whitening matrices, paired directions, ranks, reliability spectrum, source gate, and both mismatch scales. This state is identical to that used in the corresponding retrieval evaluation; classification performs only projection and scoring on the class prototypes and queries, without refitting or recalibration. Input checks cover encoder provenance, feature coordinates, and actual sample identities.

Eight of the eleven classification relations meet both fixed-state reuse criteria: compatible encoder coordinates and query identities disjoint from retrieval fitting samples. \hsaTableRef{tab:supp_frozen_fit_reuse} reports the individual results and their relation-equal mean across these eight eligible relations and five datasets. NYUv2 is evaluated under its task-specific preprocessing protocols (\hsaAppendixRef{app:nyuv2_boundary}). ImageBind Tx--Th and Tx--Au provide interface-reuse diagnostics with overlapping queries; the two MAVD drive diagnostics are also reported separately.

\begin{table*}[t]

\caption{\textbf{Prototype Classification with a Fixed Relation Readout.} Values are macro Top-1 (\%). The aggregate uses eight eligible relations across five datasets. Overlap and drive diagnostics are grouped separately and excluded from the mean. Overlap counts use actual sample identities; boldface marks the highest accuracy within each row.}

\label{tab:supp_frozen_fit_reuse}

\centering
\footnotesize
\setlength{\tabcolsep}{3pt}
\renewcommand{\arraystretch}{1.20}
\begin{adjustbox}{Clip=0pt 0pt 0pt 0pt}
\begin{tabularx}{\textwidth}{@{}M{0.11\linewidth} >{\raggedright\arraybackslash}m{0.20\linewidth} M{0.08\linewidth} *{4}{Y}@{}}
\toprule
Pair & Dataset & Scope & \shortstack{Reused\\HSA} & \shortstack{Separate\\HSA} & \shortstack{Frozen\\Cosine} & Overlap \\
\midrule
\multicolumn{7}{@{}l}{\textbf{Eligible Relations: Included in the Mean}} \\
\midrule
\rowcolor{black!5}
\multicolumn{7}{@{}l}{\textbf{ImageBind}} \\
\addlinespace[2pt]
Au--D & BatVision & Eligible & \textbf{61.40} & \textbf{61.40} & 25.77 & 0 \\
Au--IMU & Ego4D & Eligible & 16.46 & \textbf{22.81} & 6.56 & 0 \\
D--IMU & UTD-MHAD & Eligible & \textbf{34.75} & \textbf{34.75} & 2.87 & 0 \\
Tx--IMU & Ego4D & Eligible & \textbf{21.79} & 21.58 & 7.16 & 0 \\
\midrule
\rowcolor{black!5}
\multicolumn{7}{@{}l}{\textbf{LanguageBind}} \\
\addlinespace[2pt]
Th--D & TartanRGBT & Eligible & \textbf{71.98} & 69.01 & 37.80 & 0 \\
Vi--Au & MSR-VTT & Eligible & \textbf{26.37} & \textbf{26.37} & 13.77 & 0 \\
Vi--D & TartanRGBT & Eligible & \textbf{73.76} & 69.99 & 43.68 & 0 \\
Vi--Th & TartanRGBT & Eligible & 77.61 & \textbf{84.37} & 74.03 & 0 \\
\midrule
\rowcolor{hsarow}
\multicolumn{3}{l}{\textbf{Mean over 8 Eligible Relations}} & 48.01 & \textbf{48.79} & 26.46 & -- \\
\midrule
\multicolumn{7}{@{}l}{\textbf{Diagnostics: Excluded from the Mean}} \\
\midrule
\rowcolor{black!5}
\multicolumn{7}{@{}l}{\textbf{ImageBind}} \\
\addlinespace[2pt]
Au--Th & MAVD & Drive & \textbf{100.00} & \textbf{100.00} & 50.00 & 0 \\
Tx--Au & VGGSound & Overlap & 27.99 & \textbf{32.60} & 27.41 & 89 \\
Tx--Th & TartanRGBT & Overlap & \textbf{91.90} & 90.92 & 28.10 & 80 \\
\midrule
\rowcolor{black!5}
\multicolumn{7}{@{}l}{\textbf{LanguageBind}} \\
\addlinespace[2pt]
Th--Au & MAVD & Drive & \textbf{76.65} & \textbf{76.65} & 50.00 & 0 \\
\bottomrule\end{tabularx}
\end{adjustbox}
\par\vspace{2pt}
\begin{minipage}{\textwidth}
\scriptsize
The aggregate covers five datasets. Reuse--cosine: 21.56 points [95\% CI: 15.16, 29.52]; reuse--separate: $-0.77$ points [95\% CI: $-2.30$, 0.00]. Intervals use dataset-cluster resampling; differences are computed before rounding.
\end{minipage}

\end{table*}

Across the eight eligible relations, the unchanged retrieval state achieves 48.01\% macro Top-1, compared with 26.46\% for frozen cosine and 48.79\% for HSA fitted separately for classification. The 21.56-point gain over cosine shows that the fitted relation supports class-prototype prediction as well as instance retrieval. All classification queries in this aggregate are disjoint from the retrieval fitting samples under the dataset splits in \hsaAppendixRef{app:classification_data}; identity checks use the underlying samples across feature-extraction protocols. The separate overlap diagnostics contain 80/436 queries for ImageBind text--thermal on TartanRGBT and 89/15,421 for ImageBind text--audio on VGGSound.

\section{Retrieval Protocol and Reproducibility}
\label{app:experimental_protocol}

This section documents the datasets, information boundaries, training rules, implementation settings, seed variability, and excluded degenerate unit used by the main-paper retrieval comparisons.

\subsection{Detailed Retrieval Protocol}
\label{app:exp_protocol}

\noindent\textbf{Datasets and Evaluation Relations.}
Each evaluation relation combines one backbone, one dataset, and one held-out target pair. ImageBind~\cite{imagebind} uses image as its hub and contributes nine nondegenerate relations among text, audio, depth, thermal, and inertial modalities. LanguageBind~\cite{languagebind} uses language as its hub and contributes all ten relations among image, video, audio, depth, and thermal. Across both backbones, the nineteen relations span seven modalities and ten datasets. The datasets are VGGSound~\cite{vggsound}, NYUv2~\cite{nyuv2}, TartanRGBT~\cite{tartanrgbt}, Ego4D~\cite{ego4d}, and BatVision~\cite{batvision}. They also include MAVD~\cite{mavd}, UTD-MHAD~\cite{utdmhad}, Caltech Aerial RGBT~\cite{caltechaerialrgbt}, UCF101~\cite{ucf101}, and MSR-VTT~\cite{msrvtt}. \hsaTableRef{tab:supp_dataset_statistics} reports modality coverage, positive definitions, sample counts, and relation counts.

\begin{table*}[t]
\caption{\textbf{Retrieval Dataset Coverage.} Tx, Im, Vi, Au, D, Th, and IMU denote text, image, video, audio, depth, thermal, and inertial modalities.}
\label{tab:supp_dataset_statistics}
\centering
\footnotesize
\setlength{\tabcolsep}{3pt}
\renewcommand{\arraystretch}{1.15}
\begin{tabularx}{\textwidth}{@{}>{\raggedright\arraybackslash}m{0.15\linewidth}X >{\raggedright\arraybackslash}m{0.205\linewidth} M{0.10\linewidth} M{0.10\linewidth} M{0.065\linewidth}@{}}
\toprule
Dataset & Modalities & Positive Definition & Train & Test & Relations \\
\midrule
VGGSound (VGS) & Tx, Im, Au & Class (100), Instance & 1,500 & 1,000 & 2 \\
NYUv2 (NYU) & Tx, Im, D & Instance & 47,584 & 654 & 1 \\
TartanRGBT (TR) & Tx, Im, Vi, D, Th & Trajectory (46), Instance & 1,733--8,602 & 436--8,601 & 6 \\
Ego4D (E4D) & Tx, Im, Au, IMU & Instance & 1,138--1,600 & 487--686 & 2 \\
BatVision (BV) & Tx, Im, Au, D & Instance & 2,536 & 584 & 2 \\
MAVD & Tx, Im, Au, Th & Instance & 647--648 & 432--433 & 2 \\
UTD-MHAD (UTD) & Im, D, IMU & Instance & 431 & 430 & 1 \\
Caltech Aerial RGBT (CA) & Im, Th, IMU & Instance & 245 & 245 & 1 \\
UCF101 (UCF) & Tx, Im, Vi & Instance & 7,070 & 3,030 & 1 \\
MSR-VTT (MSR) & Tx, Vi, Au & Instance & 6,176 & 884 & 1 \\
\midrule
Total & -- & -- & -- & -- & 19 \\
\bottomrule
\end{tabularx}
\end{table*}

VGGSound contributes one class-level and one paired-instance relation. TartanRGBT contributes one trajectory-level and five paired-instance relations. Seventeen of the nineteen relations use paired-instance positives. Every query ranks the complete test gallery. The ImageBind text--depth configuration on NYUv2 contains one positive group and is excluded because every method obtains 100\% Recall@10.

\noindent\textbf{Data Splits.}
VGGSound uses class-stratified sampling, and Ego4D and Caltech Aerial RGB--Thermal use random sample-level splits. TartanRGBT uses sample-level splits with shared trajectories for ImageBind and trajectory-disjoint splits for LanguageBind. UTD-MHAD uses subjects 1, 3, 5, and 7 for training and subjects 2, 4, 6, and 8 for testing. MAVD uses chronological splits within each drive. BatVision follows the dataset-provided splits within each location, combining the training and validation subsets for fitting. MSR-VTT follows the 7K/1K training/test annotations, retaining samples with available video and successfully decoded audio. UCF101 and NYUv2 use fixed training and test sample lists. These partitions remain unchanged across all compared methods for each relation.

\noindent\textbf{Compared Methods and Information Boundary.}
All compared methods use all available training rows and keep the backbone frozen. For methods without target-pair supervision, $A$--$B$ pair identities are masked throughout fitting, rank selection, gating, calibration, stopping, and validation. Frozen cosine compares target features directly. ReAlign~\cite{realign} collects the $A$ and $B$ training marginals from the two hub edges, discards cross-modal row identities, and applies fixed $A\!\rightarrow\!B$ Anchor, Trace, and Centroid Alignment. Hub-relative similarity~\cite{relrep} compares target samples through their similarity vectors to hub samples. Bidirectional ridge~\cite{hoerl1970ridge} and Procrustes~\cite{schonemann1966} independently map both targets into the hub space. ERM~\cite{domainbed}, IRM~\cite{irm}, VREx~\cite{vrex}, DANN~\cite{dann}, and CORAL~\cite{deepcoral} use the same two-layer projection-head architecture. Each modality has a separate projection head with independent weights, and all heads share a learnable temperature parameter. These methods treat the two hub edges as training environments. HSA is the complete closed-form readout.

ASIF, Paired-OP, and Full A--B use $A$--$B$ training-row identities. ASIF~\cite{asif} builds a coupled dictionary from these paired rows. Paired-OP~\cite{schonemann1966} uses every paired row to solve one orthogonal map per retrieval direction. Both methods take zero gradient steps. Full A--B trains the same projection architecture on target-paired data and averages seeds 40--42 under the globally selected 200-step duration. In Tables~I--IV, $^\dagger$ marks ASIF, Paired-OP, and Full A--B because all three use target-pair identities. Boldface and underlining identify the best and second-best values among all methods except Full A--B. ASIF and Paired-OP participate in these markings.

\noindent\textbf{Baseline Training and Selection.}
For each trainable hub-edge baseline and seed, 90\% of source rows fit a pilot model. The remaining 10\% select the training duration using mean $A$--$H$ and $H$--$B$ validation InfoNCE~\cite{oord2018cpc}. The loss is checked every 100 steps. After 400 steps, training stops following five checks without a 0.2\% relative improvement, subject to a 4,000-step cap. The selected duration is the best validation step, with a minimum of 400 steps. Pilot weights are discarded before retraining on all source rows for the selected duration. Reported results average seeds 40, 41, and 42.

Full A--B uses one training duration across all relations and seeds. For each relation and seed in $\{40,41,42\}$, a deterministic 90/10 split of target-pair training rows fits a pilot. Each pilot records validation InfoNCE every 100 steps from 100 through 4,000. For candidate duration $t$, the selection score is the equal-relation mean of the seed-averaged log loss ratio,
\begin{equation}
\mathcal J(t)=\frac{1}{19}\sum_{e=1}^{19}\frac{1}{3}\sum_{s\in\{40,41,42\}}\log\frac{L_{e,s}(t)}{L_{e,s}(100)}.
\label{eq:supp_full_ab_duration}
\end{equation}
Minimizing $\mathcal J(t)$ selects 200 steps, with the shorter duration breaking an exact tie. Pilot weights are then discarded. All $19\times3=57$ final models are retrained on every target-pair training row for exactly 200 steps. No target test sample enters duration selection. The target test set is accessed only after the protocol is frozen.

\noindent\textbf{Metrics and Implementation.}
For relation $e$ and $k\in\{1,5,10\}$, bidirectional Recall@$k$ averages the two retrieval directions over the complete gallery. Recall@10 is the primary metric. Recovery@10 is defined in \hsaEqRef{eq:supp_recovery_at_10}. Blue parentheses in the main tables show point gains over frozen cosine, while the final columns report ratios of backbone means to Full A--B. Following the source-to-anchor construction, the first listed modality on each evaluation edge is $A$, and the second is $B$. ReAlign uses a $10^{-10}$ stabilizer and fixes the $A\!\rightarrow\!B$ orientation without performance-based selection. ASIF uses $k=\min(800,n_{\mathrm{train}})$ nonzero anchors and a similarity exponent of 8. Its anchor selection and sparse retrieval follow a fixed CPU path that determines the order of equal-similarity anchors. Paired-OP row-normalizes inputs, applies no centering, and solves both retrieval directions independently by singular-value decomposition. Neither ASIF nor Paired-OP selects hyperparameters on the test set. HSA uses $\lambda=0.5$, $\kappa=2$, $k_{\max}=256$, and rank-permutation seed 142. Score calibration uses mismatch seeds 42--46. The reliability gate uses 256 simultaneous row-and-column permutations of the regularized covariance matrix, the 0.95 quantile, and seed 42. Trainable baselines use a $\operatorname{Linear}(d,1024)$--GELU--$\operatorname{Linear}(1024,256)$ head. They use AdamW~\cite{loshchilov2019adamw} with learning rate $10^{-3}$ and weight decay $10^{-4}$. All settings remain fixed across relations.
\par

\subsection{Efficiency Measurement and Outcome Checks}
\label{app:efficiency_protocol}

Table~X measures all fourteen methods over 19
nondegenerate relations and 494 independent runs. Analytic fits use one warm-up
and five timed repeats, while trainable methods use seeds 40--42. Each timed
score call reproduces the registered retrieval hits, yielding 2,964 passing
checks. All GPU paths use the same NVIDIA RTX 4070 SUPER. ASIF retains its fixed
CPU path, and each method retains its original software environment. Fit and training times quantify the
realized workload in those environments. \emph{Steps} is the median final-fit
gradient-step count over relation--seed runs. The archived records retain the
complete selection-search and total-gradient-step distributions.

Table~XI repeats the measurement over
the eleven classification relations. Analytic fits run on the CPU, while
trainable methods replay model selection and final fitting for seeds 40--42 on
the same GPU. Relation values are seed medians. All post-fit readouts use a
12-thread BLAS and 6-thread PyTorch CPU configuration after one warm-up and
five timed repeats. The timed boundary includes method-specific transforms,
the complete query--prototype score matrix, and Top-1 selection; shared
prototype construction is excluded. We report latency per 1,000 queries and
score throughput because class counts vary. The 286 method-level runs reproduce
every registered macro Top-1 value. All 198 trainable runs match the prescribed
architectures, seeds, workloads, inputs, and final steps.

Classification fitting times are relation medians; readout latency and throughput pool time, queries, and scores across all eleven relations. On the shared CPU configuration, HSA processes 1,000 queries in 25.614\,ms at 9.172 million query--prototype scores per second. Its latency is approximately 2.4 times that of the source-trained neural baselines (Table~XI), measured after fitting on the same pooled workload.

\subsection{Seed Variability of Trainable Baselines}
\label{app:seed_variability}

\hsaTableRef{tab:supp_seed_variability} reports sample standard deviations across seeds 40--42. The first three columns give the standard deviation of each seed-level aggregate. The final two summarize relation-wise standard deviations. All values were recomputed from the archived seed-level results using denominator $n-1$.

\begin{table}[t]
\caption{\textbf{Recall@10 Variability across Three Training Seeds.} Entries are sample standard deviations in percentage points. ``Mean relation SD'' averages the nineteen relation-wise standard deviations; ``Max relation SD'' gives their maximum.}
\label{tab:supp_seed_variability}
\centering
\scriptsize
\setlength{\tabcolsep}{1.4pt}
\renewcommand{\arraystretch}{1.15}
\begin{tabularx}{\columnwidth}{@{}l*{5}{Y}@{}}
\toprule
Method & \shortstack{ImageBind\\Mean SD} & \shortstack{LanguageBind\\Mean SD} & \shortstack{All-19\\Mean SD} & \shortstack{Mean\\Relation SD} & \shortstack{Max\\Relation SD} \\
\midrule
ERM   & 0.25 & 0.17 & 0.12 & 0.58 & 1.33 \\
IRM   & 0.03 & 0.18 & 0.09 & 0.52 & 1.30 \\
VREx  & 0.28 & 0.17 & 0.22 & 0.71 & 1.55 \\
DANN  & 0.31 & 0.37 & 0.25 & 0.71 & 1.69 \\
CORAL & 0.10 & 0.23 & 0.08 & 0.62 & 1.41 \\
Full A--B & 0.36 & 0.50 & 0.42 & 0.64 & 1.49 \\
\bottomrule
\end{tabularx}
\end{table}

\subsection{Excluded Degenerate Unit}
\label{app:excluded_unit}

The archived ImageBind text--depth evaluation on NYUv2 contains one positive group. Every evaluated readout obtains 100\% Recall@10 because all candidates belong to the same positive set. This unit remains in the protocol record but is excluded from the nineteen-relation aggregate, method ranking, and recovery analysis.

\section{Additional Retrieval Results}
\label{app:supp_experiments}

This section reports direction-wise consistency, recovery relative to target-paired training, and robustness under alternative aggregate definitions. Unless stated otherwise, all aggregates use the same nineteen nondegenerate relations and the bidirectional full-gallery retrieval protocol.

\subsection{Direction-Wise Retrieval Consistency}
\label{app:directional_consistency}

\hsaTableRef{tab:hsa_directional_consistency} reports the direction-wise aggregates summarized in Q1. HSA improves both directions on each backbone, with all-relation gains of 13.33 and 12.43 points.

\begingroup
\begin{table}[t]
\caption{\textbf{Direction-Wise Retrieval Consistency.} Mean Recall@10 (\%) in both retrieval directions; parentheses show HSA gains over frozen cosine.}
\label{tab:hsa_directional_consistency}
\centering
\scriptsize
\setlength{\tabcolsep}{1.4pt}
\renewcommand{\arraystretch}{1.15}
\begin{adjustbox}{Clip=0pt 0pt 0pt 0pt}
\begin{tabularx}{\columnwidth}{@{}>{\raggedright\arraybackslash}m{0.17\linewidth}*{6}{Y}@{}}
\toprule
& \multicolumn{2}{c}{ImageBind} &
  \multicolumn{2}{c}{LanguageBind} &
  \multicolumn{2}{c}{Overall} \\
\cmidrule(lr){2-3}\cmidrule(lr){4-5}\cmidrule(l){6-7}
Method &
$A\!\to\!B$ & $B\!\to\!A$ &
$A\!\to\!B$ & $B\!\to\!A$ &
$A\!\to\!B$ & $B\!\to\!A$ \\
\midrule
Frozen Cosine
& 13.74 & 7.86
& 24.77 & 25.24
& 19.54 & 17.01 \\
\midrule
\rowcolor{hsarow}
\raisebox{-2.25pt}[0pt][0pt]{\textbf{HSA (Ours)}}
& \gaincell{\textbf{29.16}}{+15.43}
& \gaincell{\textbf{20.63}}{+12.77}
& \gaincell{\textbf{36.22}}{+11.45}
& \gaincell{\textbf{37.35}}{+12.11}
& \gaincell{\textbf{32.88}}{+13.33}
& \gaincell{\textbf{29.43}}{+12.43} \\
\bottomrule
\end{tabularx}
\end{adjustbox}
\end{table}
\endgroup

\subsection{Recovery Relative to Target-Paired Training}
\label{app:recovery_relative}

For each relation $e$, the reported ratio is
\begin{equation}
\mathrm{Recovery@10}(e)
=100\times
\frac{\mathrm{R@10}_{\mathrm{HSA}}(e)}
{\mathrm{R@10}_{\mathrm{Full}\ A\text{-}B}(e)}.
\label{eq:supp_recovery_at_10}
\end{equation}
\hsaTableRef{tab:supp_recovery_at_10} gives the relation-wise values underlying the recovery columns in Tables~I and~III. The aggregate rows report ratios of backbone-mean Recall@10 values rather than averages of the relation-wise percentages.

\begin{table*}[t]
\caption{\textbf{Relation-Wise Recall@10 Relative to Full A--B.} Pair abbreviations follow the main tables. HSA and Full A--B are bidirectional percentages. Recovery uses the unrounded values in \hsaEqRef{eq:supp_recovery_at_10}; displayed recalls are independently rounded to two decimals.}
\label{tab:supp_recovery_at_10}
\centering
\footnotesize
\setlength{\tabcolsep}{3pt}
\renewcommand{\arraystretch}{1.15}
\begin{tabularx}{\textwidth}{@{}>{\raggedright\arraybackslash}m{0.13\linewidth} M{0.10\linewidth} X *{3}{Y}@{}}
\toprule
Backbone & Held-Out Pair & Dataset & HSA & Full A--B & Recovery (\%) \\
\midrule
\multirow{10}{*}{ImageBind}
& Tx--Au & VGGSound & 84.45 & 84.57 & 99.9 \\
& Tx--Th & TartanRGBT & 74.87 & 74.71 & 100.2 \\
& Tx--IMU & Ego4D & 8.97 & 4.01 & 223.6 \\
& Au--D & BatVision & 6.93 & 7.13 & 97.2 \\
& Au--Th & MAVD & 4.63 & 3.09 & 150.0 \\
& Au--IMU & Ego4D & 8.93 & 7.15 & 124.9 \\
& D--Th & TartanRGBT & 9.66 & 15.76 & 61.3 \\
& D--IMU & UTD-MHAD & 9.30 & 6.01 & 154.8 \\
& Th--IMU & Caltech Aerial RGBT & 16.33 & 4.15 & 393.4 \\
\cmidrule(lr){2-6}
& \multicolumn{2}{l}{Ratio of Backbone Means} & 24.90 & 22.95 & 108.5 \\
\midrule
\multirow{11}{*}{LanguageBind}
& Im--Vi & UCF101 & 99.72 & 99.76 & 100.0 \\
& Im--Th & TartanRGBT & 60.89 & 60.36 & 100.9 \\
& Im--D & NYUv2 & 13.53 & 44.19 & 30.6 \\
& Im--Au & VGGSound & 67.15 & 41.02 & 163.7 \\
& Vi--Th & TartanRGBT & 53.33 & 54.43 & 98.0 \\
& Vi--D & TartanRGBT & 18.00 & 25.84 & 69.7 \\
& Vi--Au & MSR-VTT & 26.64 & 26.34 & 101.1 \\
& Th--D & TartanRGBT & 19.27 & 31.31 & 61.5 \\
& Th--Au & MAVD & 3.93 & 2.73 & 143.7 \\
& D--Au & BatVision & 5.39 & 5.82 & 92.6 \\
\cmidrule(lr){2-6}
& \multicolumn{2}{l}{Ratio of Backbone Means} & 36.79 & 39.18 & 93.9 \\
\bottomrule
\end{tabularx}
\end{table*}

HSA reaches 31.15\% mean Recall@10 without target-pair fitting and exceeds Full A--B on 10 of 19 relation-level point estimates. With its specified projection head and objective, Full A--B uses a globally selected 200-step duration on the same data splits. It reaches $22.95\pm0.36$\% on ImageBind and $39.18\pm0.50$\% on LanguageBind, with an overall mean of $31.49\pm0.42$\%. Here $\pm$ denotes the sample standard deviation across seeds 40--42. Recovery quantifies HSA's performance relative to this fitted target-paired reference; values above 100\% indicate higher Recall@10 on the corresponding relation.

\subsection{Robustness to Aggregation and Chance Level}
\label{app:aggregate_robustness}

The main paper reports an unweighted mean over the nineteen configured relations. \hsaTableRef{tab:supp_aggregate_robustness} evaluates the HSA gain over frozen cosine under alternative summaries. The dataset-balanced value first averages relations within each dataset and then weights the ten datasets equally. The leave-one-dataset-out range recomputes the relation mean after excluding each dataset in turn. For relation $e$, the chance-normalized gain is
\begin{equation}
G_e^{\mathrm{chance}}
=100\times\frac{R_{\mathrm{HSA},e}-R_{C0,e}}{100-C_e},
\label{eq:supp_chance_normalized_gain}
\end{equation}
where $C_e$ is the Recall@10 chance level implied by the complete gallery and the registered positive definition.

\begin{table}[t]
\caption{\textbf{Robustness of the HSA--Frozen-Cosine Recall@10 Gain.} Point gains are in percentage points; the final row is the percentage of chance-adjusted headroom recovered.}
\label{tab:supp_aggregate_robustness}
\centering
\footnotesize
\setlength{\tabcolsep}{2pt}
\renewcommand{\arraystretch}{1.15}
\begin{tabularx}{\columnwidth}{@{}>{\raggedright\arraybackslash}m{0.35\linewidth}X M{0.18\linewidth}@{}}
\toprule
Summary & Weighting Rule & Gain \\
\midrule
Relation Mean & Equal weight over 19 relations & 12.88 \\
Relation Median & Median over 19 relations & 9.25 \\
Dataset-Balanced Macro & Equal weight over 10 datasets & 9.72 \\
Leave One Dataset Out & Relation Mean after each exclusion & 8.02--14.11 \\
Chance Normalized & Mean of \hsaEqRef{eq:supp_chance_normalized_gain} & 13.98\% \\
\bottomrule
\end{tabularx}
\end{table}

Relation-specific chance levels range from 0.12\% to 19.96\%. HSA exceeds chance on all nineteen relations, whereas frozen cosine does so on fourteen. Every alternative summary remains positive, showing that the aggregate gain persists across the observed mixture of datasets and relation difficulties. The 9.72-point dataset-balanced macro and the 8.02-point leave-one-dataset-out lower bound further demonstrate stable gains under equal dataset weighting and every single-dataset exclusion.

\section{Mechanism and Diagnostic Analyses}
\label{app:diagnostics}

This section collects the complete carrier interventions, component ablations, disjoint-source control, and correspondence-resolution diagnostics supporting the main-paper mechanism analysis.

\subsection{Complete Aggregate Intervention Statistics}

\label{app:intervention_statistics}

Source permutations evaluate complete HSA: each condition refits directions, ranks, the gate, and scales while preserving the algorithm, hyperparameters, and input marginals. Carrier localization holds the intact candidate-resolution branch, gate, and resolution scale fixed, assigns the same reliability spectrum to each direction set, and recalibrates its carrier scale by the same source-only rule. The reliability-dose endpoint comparison changes only $\alpha$ between 0 and 1 with all fitted quantities and scales fixed. Table~\ref{tab:supp_mechanism_interventions} summarizes these three interventions; Appendix~\ref{app:cluster_statistics} lists the complete test families.

\begin{table*}[t]

\caption{\textbf{Aggregate Complete-Score Interventions.} Differences are intact-minus-permuted, leading-minus-control, or $\alpha=1$ minus $\alpha=0$ bidirectional Recall@10 in percentage points. Intervals and sign flips use dataset groups; Holm correction uses each complete comparison family.}

\label{tab:supp_mechanism_interventions}

\centering
\footnotesize
\setlength{\tabcolsep}{3pt}
\renewcommand{\arraystretch}{1.15}
\begin{tabularx}{\textwidth}{@{}X M{0.14\linewidth} M{0.14\linewidth} M{0.22\linewidth} M{0.15\linewidth}@{}}
\toprule
Comparison & \shortstack{Positive\\Relations} & \shortstack{Mean\\Difference} & 95\% Cluster Interval & Holm $p$ \\
\midrule
Permute $A$--$H$ & 19/19 & 26.52 & [9.83, 40.64] & $2.93\times10^{-3}$ \\
Permute $H$--$B$ & 19/19 & 26.58 & [10.03, 41.01] & $2.93\times10^{-3}$ \\
Permute Both Edges & 19/19 & 24.59 & [9.22, 36.53] & $2.93\times10^{-3}$ \\
\midrule
Subsequent Carriers & 19/19 & 16.44 & [6.31, 25.41] & $1.95\times10^{-3}$ \\
Random Carriers & 19/19 & 19.38 & [7.06, 29.21] & $1.95\times10^{-3}$ \\
\midrule
Reliability $\alpha=0$ & 18/19 & 12.00 & [4.75, 16.96] & 0.01172 \\
\bottomrule\end{tabularx}

\end{table*}

All relations from one dataset receive the same sign flip, so the smallest one-sided probability is determined by ten groups: $2^{-10}=1/1024$. Seeds are averaged within relations and do not add independent observations.

\subsection{Complete Component-Ablation Visualization}
\label{app:component_ablation_details}

\hsaFigureRef{fig:hsa_component_ablation} visualizes the complete Recall@10 ablations, with dataset-cluster comparisons in Table~\ref{tab:supp_cluster_comparisons}. Table~IX in the main paper reports backbone-specific Recall@10 results together with the corresponding function-metric changes defined in Appendix~\ref{app:component_functions}. Each comparison changes one scoring component while retaining the fitted directions, evaluation relations, and remaining scoring pipeline.
Removing the carrier score produces the largest all-relation reduction, at 12.00 points. Removing the source-derived gate and using uniform channel reliabilities reduce the mean by 3.86 and 2.48 points. Candidate resolution contributes 3.43 points overall, concentrated on LanguageBind. Residual orthogonalization has the smallest overall effect and shows the same backbone dependence.

\begin{figure}[t]
\centering
\includegraphics[width=\columnwidth]{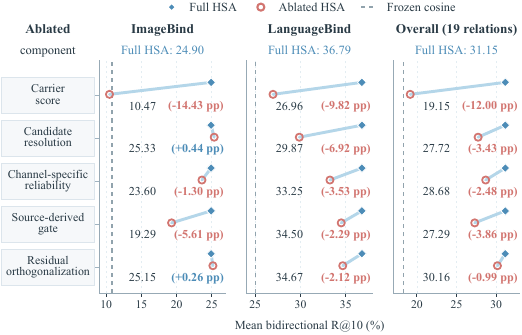}
\caption{\textbf{Complete Component Ablation of the HSA Readout.} Mean bidirectional Recall@10 for full HSA and five matched single-component variants on ImageBind, LanguageBind, and all nineteen relations.}
\label{fig:hsa_component_ablation}
\end{figure}

\subsection{Matched Shared-Row and Disjoint-Source Control}
\label{app:disjoint_source}

\hsaTheoremRef{thm:hub_readable} expresses the population relation as a function of two hub-edge moment systems. We therefore test whether HSA retains its retrieval gain when the hub edges use mutually disjoint source instances. For each split seed in $\{42,\ldots,51\}$, the source rows are randomly partitioned into equal halves $I_1$ and $I_2$. When the row count is odd, one predetermined row is omitted. Shared-row conditions estimate both edges from the same half. Disjoint conditions estimate $A$--$H$ and $H$--$B$ from opposite halves. All conditions use equal per-edge sample counts and matched cross-half calibration rows. The gate construction, scoring rule, and evaluation set remain unchanged.

For every metric, the two shared orientations and two disjoint orientations are first averaged within each split seed. Within-relation confidence intervals use 10,000 bootstrap resamples with seed 42. The cross-relation mean row instead uses the 100,000-draw dataset-cluster analysis in Appendix~\ref{app:cluster_statistics}. The individual relation intervals resample each relation's natural independent unit: retrieval group, paired instance, scene, trajectory, subject, drive, or video. \hsaTableRef{tab:supp_disjoint_source} reports the independent-unit count and the primary Recall@10 comparison. Here $C0$ denotes frozen cosine, $S$ the sample-size-matched shared-row condition, and $D$ the strictly disjoint condition.

\begin{table*}[t]
\caption{\textbf{Recall@10 under Matched Shared-Row and Disjoint-Source Estimation.} Values are bidirectional percentages. $n_{\mathrm{unit}}$ gives the number of clustered bootstrap units; the final columns report $D-C0$ and $S-D$ with natural-unit 95\% intervals for each relation; the final aggregate row uses a dataset-cluster interval.}
\label{tab:supp_disjoint_source}
\centering
\scriptsize
\setlength{\tabcolsep}{2pt}
\renewcommand{\arraystretch}{1.20}
\begin{adjustbox}{Clip=0pt 0pt 0pt 0pt}
\begin{tabularx}{\textwidth}{@{}>{\raggedright\arraybackslash}m{0.115\linewidth} M{0.08\linewidth} >{\raggedright\arraybackslash}m{0.14\linewidth} M{0.045\linewidth} *{3}{Y} M{0.18\linewidth} M{0.18\linewidth}@{}}
\toprule
Backbone & Held-Out Pair & Dataset & $n_{\mathrm{unit}}$ & $C0$ & $S$ & $D$ & \shortstack{$D-C0$\\{[95\% CI]}} & \shortstack{$S-D$\\{[95\% CI]}} \\
\midrule
\multirow{9}{*}{ImageBind}
& Tx--Au & VGGSound & 100 & 75.20 & 83.80 & 81.06 & \shortstack{5.86\\{[2.75,9.37]}} & \shortstack{2.74\\{[1.94,3.59]}} \\
& Tx--Th & TartanRGBT & 46 & 7.68 & 73.40 & 72.42 & \shortstack{64.74\\{[58.02,71.09]}} & \shortstack{0.99\\{[0.48,1.52]}} \\
& Tx--IMU & Ego4D & 686 & 1.53 & 7.03 & 6.01 & \shortstack{4.48\\{[3.35,5.67]}} & \shortstack{1.03\\{[0.44,1.63]}} \\
& Au--D & BatVision & 7 & 1.97 & 7.71 & 7.82 & \shortstack{5.85\\{[4.17,7.42]}} & \shortstack{-0.11\\{[-0.54,0.42]}} \\
& Au--Th & MAVD & 2 & 1.62 & 4.11 & 3.18 & \shortstack{1.56\\{[0.44,25.79]}} & \shortstack{0.93\\{[0.63,7.50]}} \\
& Au--IMU & Ego4D & 487 & 1.75 & 7.84 & 6.42 & \shortstack{4.67\\{[3.27,6.08]}} & \shortstack{1.42\\{[0.60,2.24]}} \\
& D--Th & TartanRGBT & 46 & 0.34 & 9.71 & 9.50 & \shortstack{9.16\\{[7.98,10.40]}} & \shortstack{0.21\\{[0.01,0.43]}} \\
& D--IMU & UTD-MHAD & 4 & 3.02 & 6.87 & 6.64 & \shortstack{3.62\\{[1.23,6.00]}} & \shortstack{0.23\\{[-0.19,0.69]}} \\
& Th--IMU & Caltech Aerial RGBT & 245 & 4.08 & 17.53 & 10.64 & \shortstack{6.56\\{[4.63,8.43]}} & \shortstack{6.89\\{[4.52,9.34]}} \\
\midrule
\multirow{10}{*}{LanguageBind}
& Im--Vi & UCF101 & 3030 & 96.49 & 99.61 & 99.60 & \shortstack{3.11\\{[2.65,3.61]}} & \shortstack{0.02\\{[-0.00,0.04]}} \\
& Im--Th & TartanRGBT & 9 & 36.93 & 58.54 & 58.15 & \shortstack{21.22\\{[16.77,25.79]}} & \shortstack{0.40\\{[-0.14,0.81]}} \\
& Im--D & NYUv2 & 654 & 11.62 & 13.47 & 13.30 & \shortstack{1.68\\{[-0.80,4.14]}} & \shortstack{0.17\\{[0.01,0.34]}} \\
& Im--Au & VGGSound & 1000 & 43.40 & 62.92 & 61.43 & \shortstack{18.04\\{[15.87,20.19]}} & \shortstack{1.48\\{[1.12,1.85]}} \\
& Vi--Th & TartanRGBT & 9 & 36.01 & 52.99 & 52.57 & \shortstack{16.56\\{[12.86,22.58]}} & \shortstack{0.42\\{[0.10,0.78]}} \\
& Vi--D & TartanRGBT & 9 & 7.11 & 16.99 & 16.25 & \shortstack{9.14\\{[5.09,18.06]}} & \shortstack{0.74\\{[0.07,1.80]}} \\
& Vi--Au & MSR-VTT & 884 & 6.96 & 24.13 & 20.91 & \shortstack{13.95\\{[11.83,16.06]}} & \shortstack{3.22\\{[2.64,3.81]}} \\
& Th--D & TartanRGBT & 9 & 7.57 & 18.57 & 16.97 & \shortstack{9.40\\{[3.50,17.44]}} & \shortstack{1.59\\{[0.93,2.11]}} \\
& Th--Au & MAVD & 2 & 2.08 & 3.45 & 3.53 & \shortstack{1.45\\{[1.35,3.68]}} & \shortstack{-0.08\\{[-1.71,-0.00]}} \\
& D--Au & BatVision & 7 & 1.88 & 5.48 & 5.44 & \shortstack{3.56\\{[1.07,6.60]}} & \shortstack{0.04\\{[-0.25,0.38]}} \\
\midrule
\rowcolor{hsarow}
\multicolumn{3}{l}{\textbf{Mean over 19 Relations}} & -- & 18.27 & 30.22 & 29.04 & \shortstack{10.77\\{[3.83,16.57]}} & \shortstack{1.18\\{[0.53,2.30]}} \\
\bottomrule
\end{tabularx}
\end{adjustbox}
\end{table*}

The clustered intervals reflect variation across the dataset's sampling units: two drives for MAVD, four subjects for UTD-MHAD, and seven scenes for BatVision. For ImageBind audio--thermal, the drive-specific $D-C0$ gains are 25.79 points for the 19-row drive and 0.44 points for the 413-row drive. These two support points determine the $[0.44,25.79]$ interval and describe the observed between-drive heterogeneity.

Disjoint estimation retains a 10.77-point mean gain over frozen cosine, with dataset-cluster interval $[3.83,16.57]$, supporting relation recovery from independently sampled hub edges. Recall@10 gains are positive on all nineteen relations, with eighteen relation-wise intervals entirely above zero. LanguageBind image--depth gains 1.68 points, with interval $[-0.80,4.14]$. Shared-row estimation adds 1.18 points on average, with interval $[0.53,2.30]$, extending beyond the prespecified $\pm1$-point equivalence range. At Recall@1 and Recall@5, disjoint estimation improves eighteen relations; LanguageBind image--depth changes by $-0.74$ and $-1.15$ points, respectively.

All $19\times10\times4=760$ relation--seed--condition fits satisfy the prescribed hub-edge and calibration-row separation in the disjoint conditions. Numerical decomposition uses NumPy for 759 fits. For LanguageBind image--video, seed 50 and orientation $D21$, double-precision PyTorch CPU SVD supplies a converged solution after NumPy non-convergence. Its relative reconstruction error is $2.9\times10^{-15}$.

\subsection{Correspondence Resolution under Group-Preserving Permutations}
\label{app:granularity}

This section measures correspondence resolution with group-preserving permutations. The experiments use disjoint $A$--$H$ and $H$--$B$ training rows. Target-side permutations preserve the relevant frame, trajectory, or semantic-group structure. \hsaTableRef{tab:supp_granularity} records the finest resolution supported by the dataset and hub observations in each configuration.

\begin{table}[t]
\caption{\textbf{Correspondence-Resolution Diagnostics.} Effects are bidirectional Recall@10 differences in percentage points with 95\% intervals when available. An interval containing zero leaves the corresponding resolution unresolved.}
\label{tab:supp_granularity}
\centering
\scriptsize
\setlength{\tabcolsep}{1.5pt}
\renewcommand{\arraystretch}{1.15}
\begin{tabularx}{\columnwidth}{@{}>{\raggedright\arraybackslash}m{0.31\linewidth}YY >{\raggedright\arraybackslash}m{0.17\linewidth}@{}}
\toprule
Dataset, Pair, and Backbone & Target-Level Evidence & Group-Level Evidence & Supported Resolution \\
\midrule
TartanRGBT D--Th (ImageBind) & \shortstack{3.13\\{[2.26,4.34]}} & -- & Frame \\
MSR-VTT Vi--Au (LanguageBind) & \shortstack{interval contains\\zero} & \shortstack{28.54\\{[26.47,30.58]}} & Semantic Group \\
BatVision Au--D (ImageBind) & \shortstack{0.24\\{[-0.16,0.79]}} & -- & Unresolved \\
\bottomrule
\end{tabularx}
\end{table}

\section{Dataset-Cluster Statistics and Robustness}

\label{app:cluster_statistics}

\subsection{Estimands, Resampling, and Complete Families}

\label{app:cluster_protocol}

This post hoc statistical supplement averages seeds within each relation, then weights the nineteen relations equally. Its ten dataset groups are BatVision (2), Caltech (1), Ego4D (2), MAVD (2), MSR-VTT (1), NYUv2 (1), Tartan (6), UCF101 (1), UTD-MHAD (1), and VGGSound (2), with relation counts in parentheses. Relations sharing a dataset use one group identifier across backbones and modalities.

For relation differences $d_e$, each cluster-bootstrap draw samples $G$ datasets with replacement and carries every relation in each sampled group. The replicate estimate divides the sum of sampled differences by the actual number of sampled relations. We use seed 42, 100,000 draws, and the 2.5th and 97.5th percentiles. This preserves the relation-equal estimand. We additionally report dataset-equal means and leave-one-dataset-out ranges. One-sided exact sign flips enumerate $2^G$ whole-group signs, using the relation mean as the statistic and including ties in the tail. The assumptions are between-group independence and joint sign symmetry of group differences.

The seven families are F1, HSA versus frozen cosine; F2, the other nine source-only comparators; F3, the three source permutations; F4, the two carrier-position controls; F5, all five original Recall@10 component ablations; F6, the two relation-construction controls; and F7, disjoint-minus-cosine and shared-minus-disjoint estimation. Holm correction is applied within each complete family. ASIF, Paired-OP, and Full A--B use target-pair information and are reported separately as descriptive references. The responsibility metrics, fixed-state classification, and NYUv2 exclusion analysis receive descriptive intervals only.

\begin{table*}[t]

\caption{\textbf{Complete Dataset-Cluster Comparisons.} All differences and intervals are in percentage points. Relation-equal means, dataset-equal means, and leave-one-dataset-out ranges have distinct weighting interpretations. Every comparison covers nineteen relations and ten datasets.}

\label{tab:supp_cluster_comparisons}

\centering
\scriptsize
\setlength{\tabcolsep}{2pt}
\renewcommand{\arraystretch}{1.20}
\begin{tabularx}{\textwidth}{@{}M{0.05\linewidth} >{\raggedright\arraybackslash}m{0.215\linewidth} M{0.07\linewidth} M{0.07\linewidth} M{0.16\linewidth} M{0.15\linewidth} YY@{}}
\toprule
Family & Control/Comparison & \shortstack{Relation\\Mean} & \shortstack{Dataset\\Mean} & 95\% Interval & \shortstack{Leave-One-Out\\Range} & Raw $p$ & Holm $p$ \\
\midrule
F1 & Frozen Cosine & 12.88 & 9.72 & [5.29, 18.68] & [8.02, 14.11] & $9.77\times10^{-4}$ & $9.77\times10^{-4}$ \\
\midrule
F2 & Hub-Relative & 22.01 & 20.36 & [8.74, 32.28] & [18.06, 24.37] & $9.77\times10^{-4}$ & $8.79\times10^{-3}$ \\
F2 & Bi. Ridge & 7.17 & 6.11 & [2.14, 11.07] & [4.44, 7.99] & $9.77\times10^{-4}$ & $8.79\times10^{-3}$ \\
F2 & Bi. Procrustes & 7.56 & 6.44 & [3.49, 10.95] & [4.82, 8.28] & $9.77\times10^{-4}$ & $8.79\times10^{-3}$ \\
F2 & ReAlign & 12.67 & 9.99 & [5.41, 18.16] & [8.20, 13.96] & $9.77\times10^{-4}$ & $8.79\times10^{-3}$ \\
F2 & ERM & 9.02 & 8.20 & [3.92, 12.77] & [6.33, 10.04] & $9.77\times10^{-4}$ & $8.79\times10^{-3}$ \\
F2 & IRM & 9.32 & 8.48 & [4.60, 12.76] & [6.81, 10.34] & $9.77\times10^{-4}$ & $8.79\times10^{-3}$ \\
F2 & VREx & 8.90 & 7.86 & [3.79, 12.73] & [6.04, 9.88] & $9.77\times10^{-4}$ & $8.79\times10^{-3}$ \\
F2 & DANN & 9.51 & 8.61 & [4.40, 13.27] & [6.81, 10.50] & $9.77\times10^{-4}$ & $8.79\times10^{-3}$ \\
F2 & CORAL & 9.03 & 8.21 & [4.03, 12.70] & [6.40, 10.05] & $9.77\times10^{-4}$ & $8.79\times10^{-3}$ \\
\midrule
F3 & Permuted $A$--$H$ & 26.52 & 24.60 & [9.83, 40.64] & [21.44, 29.35] & $9.77\times10^{-4}$ & $2.93\times10^{-3}$ \\
F3 & Permuted $H$--$B$ & 26.58 & 25.25 & [10.03, 41.01] & [21.67, 29.44] & $9.77\times10^{-4}$ & $2.93\times10^{-3}$ \\
F3 & Both Edges Permuted & 24.59 & 21.98 & [9.22, 36.53] & [20.04, 27.26] & $9.77\times10^{-4}$ & $2.93\times10^{-3}$ \\
\midrule
F4 & Subsequent Carriers & 16.44 & 14.15 & [6.31, 25.41] & [12.63, 18.15] & $9.77\times10^{-4}$ & $1.95\times10^{-3}$ \\
F4 & Random Carriers & 19.38 & 16.13 & [7.06, 29.21] & [15.33, 21.41] & $9.77\times10^{-4}$ & $1.95\times10^{-3}$ \\
\midrule
F5 & w/o Carrier Score & 12.00 & 9.72 & [4.75, 16.96] & [9.00, 13.20] & $2.93\times10^{-3}$ & 0.01172 \\
F5 & w/o Channel Reliability & 2.48 & 2.69 & [1.49, 3.55] & [2.13, 2.69] & $9.77\times10^{-4}$ & $4.88\times10^{-3}$ \\
F5 & w/o Candidate Resolution & 3.43 & 2.53 & [0.26, 5.97] & [1.52, 3.93] & 0.03906 & 0.08789 \\
F5 & w/o Source Gate & 3.86 & 2.61 & [0.55, 5.89] & [2.29, 4.58] & 0.0293 & 0.08789 \\
F5 & w/o Orthogonalization & 0.99 & 0.59 & [-0.02, 1.84] & [0.28, 1.18] & 0.07422 & 0.08789 \\
\midrule
F6 & Diagonal Hub Covariance & 4.13 & 3.66 & [1.92, 5.76] & [3.03, 4.57] & $1.95\times10^{-3}$ & $3.91\times10^{-3}$ \\
F6 & Trace-Matched Identity & 4.21 & 4.01 & [2.50, 5.42] & [3.55, 4.57] & $1.95\times10^{-3}$ & $3.91\times10^{-3}$ \\
\midrule
F7 & Disjoint Minus Cosine & 10.77 & 7.34 & [3.83, 16.57] & [5.72, 11.86] & $9.77\times10^{-4}$ & $1.95\times10^{-3}$ \\
F7 & Shared Minus Disjoint & 1.18 & 1.50 & [0.53, 2.30] & [0.86, 1.38] & $2.93\times10^{-3}$ & $2.93\times10^{-3}$ \\
\bottomrule\end{tabularx}

\end{table*}

\begin{table*}[t]

\caption{\textbf{Descriptive Target-Paired References.} Differences are HSA minus reference in percentage points, across nineteen relations and ten datasets.}

\label{tab:supp_target_paired_cluster}

\centering
\footnotesize
\setlength{\tabcolsep}{3pt}
\renewcommand{\arraystretch}{1.15}
\begin{tabularx}{\textwidth}{@{}>{\raggedright\arraybackslash}m{0.20\linewidth}YY M{0.22\linewidth} M{0.22\linewidth}@{}}
\toprule
Reference & \shortstack{Relation Mean\\Difference} & \shortstack{Dataset Mean\\Difference} & 95\% Interval & Leave-One-Out Range \\
\midrule
ASIF & 4.37 & 4.31 & [1.84, 6.30] & [3.42, 4.88] \\
Paired-OP & 4.91 & 4.37 & [2.64, 6.41] & [4.00, 5.30] \\
Full A--B & -0.34 & -0.19 & [-5.43, 5.57] & [-1.91, 1.53] \\
\bottomrule\end{tabularx}

\end{table*}

Candidate resolution, source gating, and residual orthogonalization contribute positive mean Recall@10 gains, each with $p_{\mathrm{Holm}}=0.08789$. Function-specific metrics assess discrimination and redundancy exploratorily. Full A--B is a descriptive reference: its difference interval spans zero, with no prespecified equivalence or non-inferiority margin. Dataset-cluster intervals describe the evaluated benchmark collection across two shared backbones.

\begin{table*}[t]

\caption{\textbf{Complete Responsibility Results and Exploratory Intervals.} The first four metrics are displayed as percentages, with changes and intervals in percentage points; Red.\ uses absolute correlation units. Backbone columns give within-backbone changes; intervals apply to the overall change. The gate change is constant-one minus source gate, with its interval obtained by reversing the G-AUC interval.}

\label{tab:supp_component_functions}

\centering
\footnotesize
\setlength{\tabcolsep}{3pt}
\renewcommand{\arraystretch}{1.15}
\begin{tabularx}{\textwidth}{@{}>{\raggedright\arraybackslash}m{0.105\linewidth}*{5}{Y} M{0.22\linewidth}@{}}
\toprule
Metric & \hsacell{\textbf{HSA (Ours)}} & Ablation & Overall $\Delta$ & \shortstack{ImageBind\\$\Delta$} & \shortstack{LanguageBind\\$\Delta$} & 95\% Cluster Interval \\
\midrule
R@10 $\uparrow$ & \hsacell{\textbf{31.15}} & 19.15 & {\color{hsaloss}-12.00} & {\color{hsaloss}-14.43} & {\color{hsaloss}-9.82} & [-16.96, -4.75] \\
P-AUC $\uparrow$ & \hsacell{\textbf{81.17}} & 75.97 & {\color{hsaloss}-5.20} & {\color{hsaloss}-6.79} & {\color{hsaloss}-3.76} & [-7.90, -3.37] \\
C@1 $\uparrow$ & \hsacell{\textbf{28.80}} & 24.37 & {\color{hsaloss}-4.43} & {\color{hsamain}+1.22} & {\color{hsaloss}-9.51} & [-8.83, -0.00] \\
G-AUC $\uparrow$ & \hsacell{\textbf{71.43}} & 50.00 & {\color{hsaloss}-21.43} & {\color{hsamain}+37.50} & {\color{hsaloss}-41.67} & [-50.00, 32.22] \\
Red. $\downarrow$ & \hsacell{\textbf{0.034}} & 0.308 & {\color{hsaloss}+0.273} & {\color{hsaloss}+0.148} & {\color{hsaloss}+0.386} & [0.098, 0.403] \\
\bottomrule\end{tabularx}

\end{table*}

\begin{table*}[t]

\caption{\textbf{Sensitivity Excluding NYUv2.} Means use percentages; differences and intervals use percentage points. Every row covers eighteen relations and nine datasets.}

\label{tab:supp_nyuv2_exclusion}

\centering
\footnotesize
\setlength{\tabcolsep}{3pt}
\renewcommand{\arraystretch}{1.15}
\begin{tabularx}{\textwidth}{@{}>{\raggedright\arraybackslash}m{0.31\linewidth}YY M{0.20\linewidth} M{0.20\linewidth}@{}}
\toprule
Quantity & \shortstack{Relation\\Mean} & \shortstack{Dataset\\Mean} & 95\% Interval & Leave-One-Out Range \\
\midrule
HSA Mean & 32.13 & 31.84 & [13.27, 49.78] & [26.67, 35.62] \\
Frozen Cosine Mean & 18.64 & 21.24 & [3.96, 37.38] & [13.56, 20.77] \\
\midrule
HSA Minus Cosine & 13.49 & 10.59 & [5.73, 19.30] & [8.53, 14.87] \\
HSA Minus Full A--B & 1.34 & 3.20 & [-2.16, 6.98] & [-0.11, 4.22] \\
\midrule
Intact Minus Permuted $A$--$H$ & 27.34 & 26.04 & [10.05, 42.37] & [22.05, 30.46] \\
Intact Minus Permuted $H$--$B$ & 27.40 & 26.74 & [10.38, 42.80] & [22.29, 30.54] \\
Intact Minus Both Permuted & 25.29 & 23.09 & [9.21, 37.89] & [20.70, 28.22] \\
\midrule
Leading Minus Subsequent & 16.67 & 14.35 & [5.84, 26.20] & [12.65, 18.51] \\
Leading Minus Random & 19.81 & 16.63 & [6.72, 30.16] & [15.56, 22.01] \\
\bottomrule\end{tabularx}

\end{table*}

\subsection{Component Responsibilities: Definitions and Complete Results}
\label{app:component_functions}
The responsibility metrics were specified before their additional measurement and use the original complete score and registered test inputs. Fitted quantities remain fixed except for the removed component and its prescribed source-scale calibration. Carrier evidence is evaluated by the original bidirectional full-gallery Recall@10. All changes are ablation minus full HSA; lower Red.\ and higher values of the other metrics are preferred.

\noindent\textbf{Reliability Discrimination, P-AUC.} Positives are all registered $(A_i,B_i)$ pairs, once per row. Negatives use five fixed seeds, 47--51, excluding same-class candidates under class-positive protocols. Calibrated carrier scores distinguish these positives from mismatched pairs. We compare the original channel reliabilities with a constant spectrum having the same mean; the uniform condition recalibrates its carrier scale by the same source-only rule. P-AUC measures discrimination for these registered pairs, while full-gallery retrieval preserves the original multiple-positive definitions.

\noindent\textbf{Candidate Resolution, C@1.} In each direction, carrier scores define a fixed top-ten candidate pool. Full HSA and carrier-only scores rank the same pool. C@1 is Top-1 accuracy conditional on that pool containing at least one valid positive. Pool formation and Top-1 selection share the same candidate order and tie handling, with no label-based tie breaking. Coverage is the proportion of queries meeting the condition. Counts are combined over both directions within each relation before relation-equal aggregation.

\noindent\textbf{Gate Discrimination, G-AUC.} The source gate $g_R$ is the predictor for each relation. The evaluation label records whether adding the ungated candidate-resolution term improves Recall@10 over carrier-only scoring. This label evaluates gate ordering after scoring; gate estimation continues to use source statistics alone. A constant-one gate has AUC 0.5. The pooled comparison uses the complete nineteen-relation set. The descriptive backbone breakdown in Table~\ref{tab:supp_component_functions} recomputes G-AUC from the same predictions and labels within each backbone, giving 0.125 for ImageBind and 0.917 for LanguageBind. These within-backbone values measure a different ordering from the pooled G-AUC of 0.714.

\noindent\textbf{Score Redundancy, Red.} On the same registered positive and negative pairs used for reliability, Red.\ is the absolute Pearson correlation between carrier and candidate-resolution scores. Orthogonal residual coordinates and raw centered, row-normalized residual coordinates use their respective resolution scales obtained by the same source-only rule. This metric measures linear score redundancy between the two branches.
Candidate coverage is defined for all nineteen relations. Its relation-equal mean is 27.72\%, with 95\% dataset-cluster interval [12.04, 43.41]\%. The conditional set contains 15,282 of 45,170 bidirectional queries, giving pooled coverage of 33.83\%. These summaries weight relations and queries differently; the main analysis uses relation-equal aggregation. C@1 improvement is conditional on the fixed carrier pool containing a positive.
Candidate-resolution benefit varies across relations: ungated resolution benefits 5 of nineteen, with beneficial/non-beneficial counts of 1/8 for ImageBind and 4/6 for LanguageBind. Pooled G-AUC is 0.714, with 95\% cluster interval $[0.178,1.000]$. The interval uses 99,438 bootstrap draws containing both labels; 562 single-label draws are undefined. P-AUC and redundancy intervals are narrower; the candidate-resolution lower bound approaches zero. These exploratory metrics characterize discrimination, candidate ranking, and score redundancy.

The auxiliary metrics follow distinct stages of retrieval: P-AUC assesses pair discrimination, C@1 assesses selection within carrier-localized candidates, G-AUC assesses gate ordering across relations, and Red.\ measures score overlap. C@1 should be interpreted alongside candidate coverage because its denominator includes only covered queries. Full-gallery Recall@10 evaluates all queries, connecting these conditional diagnostics to the complete retrieval task.

\subsection{NYUv2 Evaluation Protocol and Cross-Dataset Robustness}

\label{app:nyuv2_boundary}

NYUv2 uses shared frozen features across methods and controls within each task. Retrieval min--max normalizes depth per image and quantizes it to 8-bit values (0--255), subsequently interpreted as millimeters. Classification converts filled metric depth from meters to millimeters before encoding. RGB resize and tensor-conversion orders also differ across tasks. Each task is evaluated under its own protocol; cross-task fixed-state reuse covers the eight compatible relations in \hsaAppendixRef{app:frozen_fit_reuse}.
Excluding NYUv2 leaves eighteen relations across nine datasets (Table~\ref{tab:supp_nyuv2_exclusion}). HSA reaches 32.13\% mean Recall@10, exceeding frozen cosine by 13.49 points, with descriptive 95\% dataset-cluster interval $[5.73,19.30]$. All intact-minus-permuted and leading-minus-control intervals remain positive. This subset preserves both the retrieval advantage and the key mechanism results, supporting cross-dataset robustness.

The two averaging schemes in Table~\ref{tab:supp_nyuv2_exclusion} address complementary aspects of benchmark composition. Relation-equal averaging gives each retained target-pair configuration the same weight, whereas dataset-equal averaging first combines configurations within a dataset. TartanRGBT contributes one mean for six relations. The corresponding HSA means are 32.13\% and 31.84\%, with gains over cosine of 13.49 and 10.59 points. Every leave-one-dataset-out gain within this subset is positive, ranging from 8.53 to 14.87 points. Thus the improvement persists when datasets receive equal weight and when each remaining dataset is omitted in turn.

\end{document}